\documentclass{article} % For LaTeX2e
\usepackage{iclr2027_conference,times}

\usepackage{amsmath,amsfonts,bm}

\def\eqref#1{equation~\ref{#1}}
\def\1{\bm{1}}

\def\ry{{\textnormal{y}}}

\def\rva{{\mathbf{a}}}

\def\rvx{{\mathbf{x}}}

\def\erva{{\textnormal{a}}}

\def\ervx{{\textnormal{x}}}

\def\vmu{{\bm{\mu}}}

\def\vp{{\bm{p}}}

\def\vr{{\bm{r}}}

\def\vw{{\bm{w}}}
\def\vx{{\bm{x}}}
\def\vy{{\bm{y}}}

\def\mX{{\bm{X}}}

\DeclareMathAlphabet{\mathsfit}{\encodingdefault}{\sfdefault}{m}{sl}
\SetMathAlphabet{\mathsfit}{bold}{\encodingdefault}{\sfdefault}{bx}{n}

\def\sX{{\mathbb{X}}}

\newcommand{\E}{\mathbb{E}}

\newcommand{\R}{\mathbb{R}}

\newcommand{\Var}{\mathrm{Var}}

\newcommand{\Cov}{\mathrm{Cov}}
\usepackage{hyperref}
\usepackage{url}

\usepackage{graphicx}
\usepackage{enumitem}
\usepackage{amsthm}
\usepackage{amssymb}
\usepackage{multirow}
\usepackage{subcaption}
\usepackage{xspace}
\usepackage{booktabs}
\usepackage{lipsum} % For dummy text
\usepackage{algorithm}
\usepackage{algpseudocode}
\usepackage{threeparttable}
\usepackage{mathtools}
\newtheorem{theorem}{Theorem}[section]
\newtheorem{proposition}[theorem]{Proposition}

\newtheorem{definition}[theorem]{Definition}

\theoremstyle{remark}

\usepackage{wrapfig}

\usepackage{bbm}
\algrenewcommand\algorithmicindent{1.0em}
\usepackage{pifont}
\usepackage[percent]{overpic}

\newcommand{\red}[1]{\textcolor{red}{#1}}

\newcommand{\method}[0]{\textsc{TaskBridge}\xspace}

\usepackage{bm}
\usepackage{longtable,pdflscape}

\title{\method: Bridging Unsupervised Tabular Anomaly Detection and In-Context Learning via Virtual Tasks}

\author{
Doyun Choi$^{1}$ \quad
Dooho Lee$^{2,3}$ \quad
Jaemin Yoo$^{1,3}$\thanks{Corresponding author.} \\
$^{1}$Seoul National University \quad
$^{2}$KAIST \quad
$^{3}$Nums AI Inc. \\
\texttt{\{doyun.choi,jaeminyoo\}@snu.ac.kr} \quad
\texttt{dooho.lee@kaist.ac.kr}
}

\iclrfinalcopy 
\begin{document}

\maketitle
\fancyhead{}
\renewcommand{\headrulewidth}{0pt}

\begin{abstract}
Unsupervised tabular anomaly detection (TAD) aims to identify anomalous rows in tabular data using normal training samples.
% and plays an important role in many real-world applications.
While conventional methods rely on dataset-specific training and configuration search, 
%recent prior-data fitted networks (PFNs)-based tabular foundation model (TFMs)
recent tabular foundation models (TFMs) enable zero-shot anomaly detection on unseen datasets via in-context learning.
Most TFM-based approaches, however, require anomaly-specific pretraining from scratch, making detection inherently dependent on synthetic TAD-specific priors and costly to update. % and extend.
Some approaches instead repurpose pretrained general-purpose TFMs for TAD to avoid this burden, but rely on computationally expensive formulations with restrictive anomaly inductive biases.
% In this work, we introduce {\method}, a new framework for effectively repurposing pretrained general-purpose TFMs for unsupervised TAD.
In this work, we introduce {\method}, a new framework that efficiently repurposes pretrained general-purpose TFMs for unsupervised TAD by constructing virtual supervised tasks that directly recast anomaly detection as supervised in-context inference of TFMs.
% \dy{In this work, we introduce {\method}, a new framework that effectively repurposes pretrained general-purpose TFMs for unsupervised TAD. {\method} bridges unsupervised TAD and supervised ICL of TFMs by constructing virtual supervised tasks.}
% {\method} bridges unsupervised TAD and supervised ICL by constructing normality-anchored virtual supervised tasks from an unlabeled nominal context. 
% These tasks induce predictive structures under which nominal query--target pairs remain compatible on PFN-based TFMs, while anomalies tend to violate the induced task structure and receive low predictive support from them.
The resulting virtual tasks induce predictive structures under which normal queries and their target pairs receive high support, whereas
anomalies tend to violate the induced structures and receive lower support, 
providing direct anomaly evidence.
Across 790 real-world datasets, {\method} consistently outperforms 30 baselines, including state-of-the-art TFM-based approaches, without anomaly-specific TFM pretraining or dataset-specific model optimization.
\end{abstract}
\section{Introduction}
%%%%% 원본 %%%%%
% Unsupervised tabular anomaly detection (TAD) plays a crucial role in a wide range of real-world applications, including finance~\citep{ALHASHEDI2021100402}, healthcare~\citep{fernando2021deeplearningmedicalanomaly}, and security~\citep{https://doi.org/10.1002/ett.4150}. 
% Various detection methods have been proposed for this problem, ranging from statistical and conventional machine learning approaches to deep learning-based methods~\citep{4781136, 10.1145/342009.335388, ICLR2024_6dfd16ff}.

% \dy{In TAD, anomalous observations are often rare and/or difficult to label in advance, making it common to design detection methods that utilize only normal training data to score unseen anomalies ~\citep{graham2023denoisingdiffusionmodelsoutofdistribution, Maziarka_2021, bergman2020classificationbasedanomalydetectiongeneral}.
% In this \emph{unsupervised} or \emph{one-class} setting, selecting or tuning a suitable detector is highly challenging since no guiding signals are provided; 
% conventional hyperparameter searches or dataset-specific tuning can no longer be reliably applied without validation data~\citep{10.1145/3606274.3606277, ding2022hyperparametersensitivitydeepoutlier}.}

Tabular anomaly detection (TAD) plays a crucial role in various real-world applications, including finance~\citep{ALHASHEDI2021100402}, healthcare~\citep{10.1145/3464423}, and security~\citep{https://doi.org/10.1002/ett.4150}.
Tabular samples lack explicit structure, unlike other domains such as images and text, making anomalous patterns more challenging to identify.
Various methods have been proposed for TAD, ranging from statistical and conventional machine learning approaches to deep learning-based methods~\citep{4781136, 10.1145/342009.335388, ICLR2024_6dfd16ff}.

In TAD, anomalous observations are often rare and difficult to label in advance, making it common to use only normal data for training~\citep{Graham_2023_CVPR, Maziarka_2021}. 
In this \emph{unsupervised} or \emph{one-class} setting, selecting a suitable detector or tuning its hyperparameters (HPs) is highly challenging since no guiding signals are provided; conventional HP searches or dataset-specific tuning cannot be reliably guided without validation data~\citep{10.1145/3606274.3606277, NEURIPS2022_3e9113e2}.
% To address this limitation, 
% % motivated by the success of foundation models (FMs) in other domains~\citep{10.1145/3637528.3671451},
% recent work has begun exploring foundation models (FMs) for unsupervised TAD as a way to reduce reliance on dataset-specific model fitting and tuning.

% while repeated hyperparameter search and dataset-specific training incur substantial computational overhead.

% In particular, recent advances in tabular foundation models (TFMs) have been driven by prior-data fitted networks (PFNs), transformer-based models trained to approximate the bayesian posterior predictive distribution (PPD) induced by a prior over supervised learning tasks~\citep{muller2024transformersbayesianinference}. During pretraining, PFNs repeatedly observe synthetic datasets sampled from diverse task priors and learn to infer the underlying task from a labeled context, allowing them to predict the task-specific label distribution of unseen query samples. This context-conditioned predictive capability enables \textit{in-context learning} without gradient-based adaptation, providing the foundation for zero-shot inference on previously unseen tabular tasks. Building on this paradigm, PFN-based TFMs have recently emerged as a promising framework for general-purpose tabular modeling~\citep{qu2026tabiclv2betterfasterscalable, grinsztajn2026tabpfn25advancingstateart, zhang2025limixunleashingstructureddatamodeling, ma2026tabdptscalingtabularfoundation, ye2025closerlooktabpfnv2}.

To address this limitation, recent work has studied tabular foundation models (TFMs) for unsupervised TAD as a way to reduce reliance on dataset-specific training and HP tuning.
A prominent line of work is to use \textit{prior-data fitted networks} (PFNs)~\citep{ shen2025fomo0dfoundationmodelzeroshot, ding2026zeroheroadvancingzeroshot, marszalek2026tacticnavigatingunknowntabular}.
After pretraining over synthetic tasks, PFNs perform zero-shot inference through \textit{in-context learning} (ICL) by inferring the task structure from a given context and amortizing Bayesian posterior-predictive inference for unseen queries~\citep{muller2024transformersbayesianinference}.
% Throughout this paper, we use \emph{tabular foundation models} (TFMs) to refer specifically to PFN-based FMs for tabular data.

There are two notable approaches to using TFMs for unsupervised TAD. The first develops TFMs pretrained exclusively for TAD~\citep{shen2025fomo0dfoundationmodelzeroshot, ding2026zeroheroadvancingzeroshot}. During pretraining, unlabeled normal samples form the context, while synthetic anomalies are introduced as queries, allowing the model to infer whether a query is anomalous conditioned on the normal context. By amortizing TAD across diverse episodic tasks with varying synthetic anomaly mechanisms, these models enable in-context anomaly detection on unseen datasets. We refer to them as \textit{TAD-specialized TFMs}.
\begin{figure*}[t]
    \centering

    \begin{subfigure}[c]{0.53\linewidth}
        \centering
        \includegraphics[width=\linewidth]{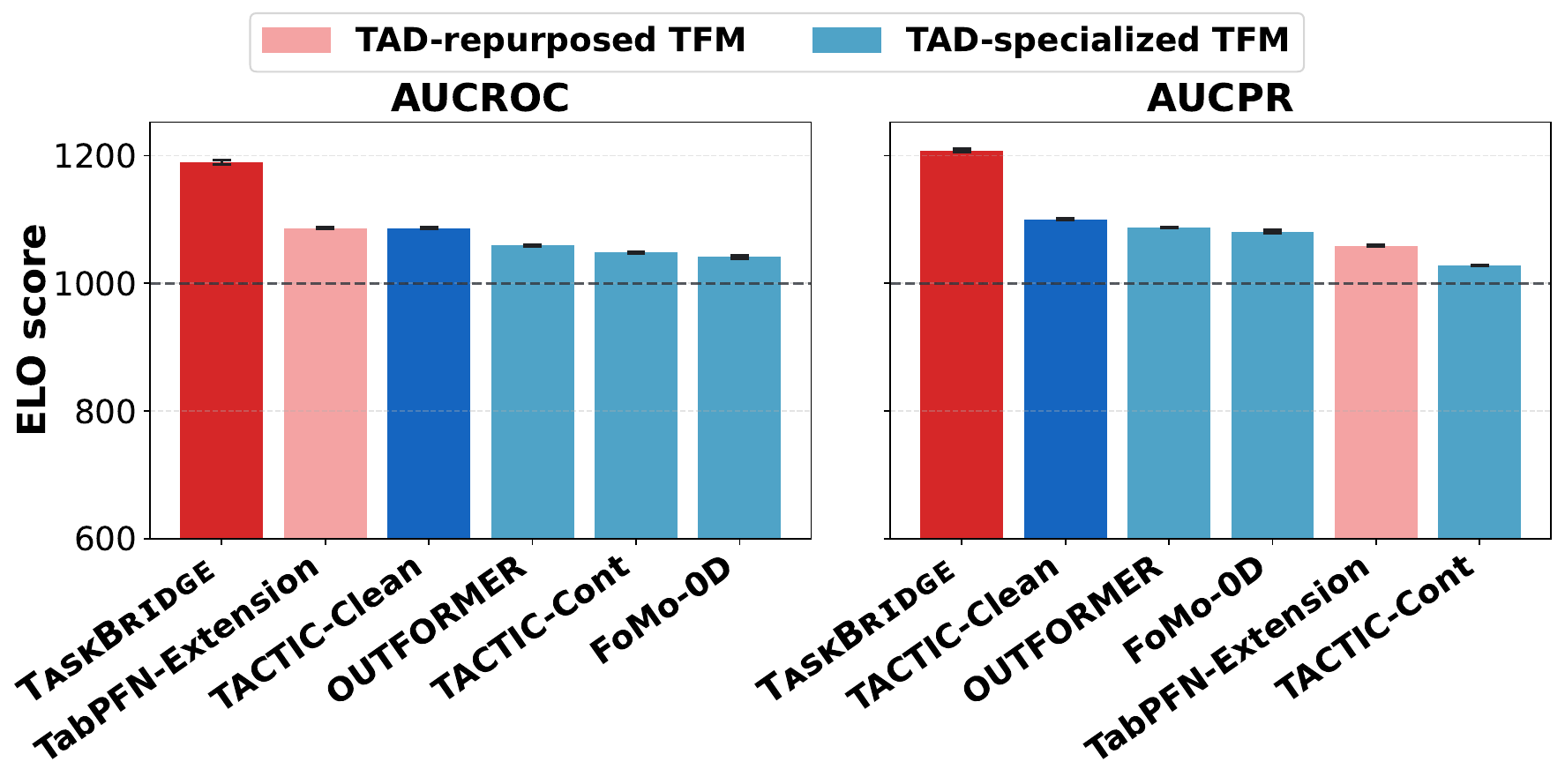}
        \label{fig:elo_performance}
    \end{subfigure}
    % \hfill
    % \hfill
    % \begin{subfigure}[c]{0.26\linewidth}
    %     \centering
    %     \includegraphics[
    %         width=\linewidth,
    %         % trim=0 0 0 0.2cm,
    %         % clip
    %     ]{Figures/oddbench_efficiency_runtime_comparison.pdf}
    %     \label{fig:efficiency_runtime}
    % \end{subfigure}
    % \hfill
    \begin{subfigure}[c]{0.41\linewidth}
        \centering
        \includegraphics[
            width=\linewidth,
            % trim=0 0 0 0.2cm,
            % clip
        ]{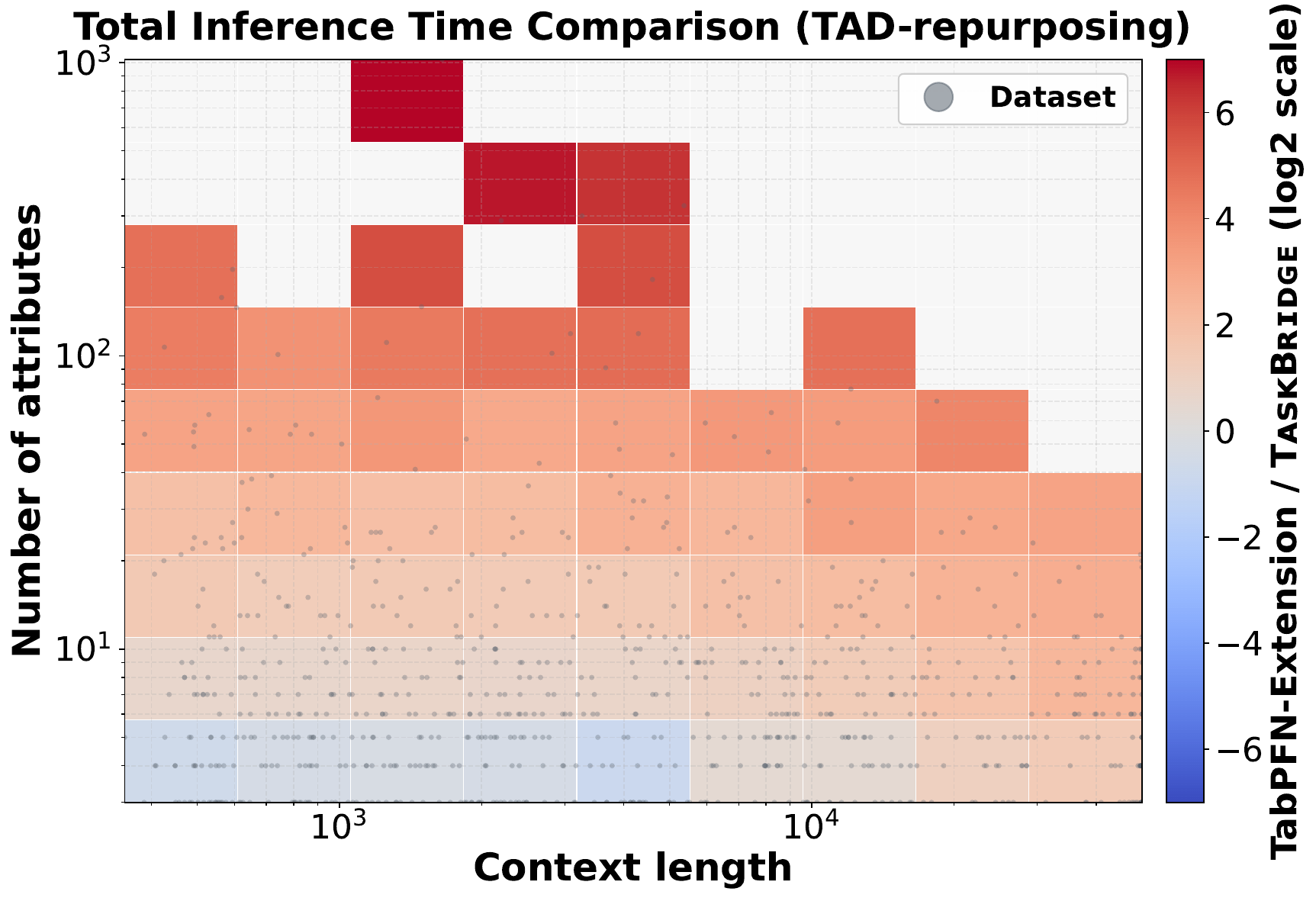}
        \label{fig:context_attributes}
    \end{subfigure}

\vspace{-15pt}
    % \caption{(Left)Elo scores for PFN-based tabular anomaly detection performance on OODBench. We compare \method against five PFN-based baselines using AUCROC and AUCPR, where \method achieves the highest Elo score on both metrics. (Right)Inference-time comparision between PFN-harnessing baselines including ours. Our method becomes super efficient compared to TabPFN-extension with the growing scale(attributes + Context length)}
%     \caption{(Left) Elo scores of PFN-based TAD methods on OODBench. (Deeper color:
% best model in each category.) \method achieves the highest Elo scores in both AUCROC and AUCPR among all PFN-based baselines. A full Elo comparison across all baselines is provided in
% Appendix~\ref{appendix:OverallPerformance}. (Right) Inference-time comparison of PFN-harnessing methods. (Color: Median inference time of TabPFN-extension/\method in that grid, Red: \method is faster) Compared with the TabPFN extension, \method scales substantially more efficiently as the number of attributes and context length increase.}
\caption{
% (Left) Elo scores of TFM-based TAD methods on \dy{790 ODDBench datasets}; darker colors indicate better performance. 
% \method achieves the highest Elo scores for both AUCROC and AUCPR, while the full comparison is in Appendix~\ref{app:ExtendedOverallPerformance}.
(Left) Elo scores of TFM-based TAD methods on ODDBench; darker colors indicate
better performance. \method achieves the highest Elo scores for both AUCROC and
AUCPR. Full baseline comparisons are provided in
Appendix~\ref{app:ExtendedOverallPerformance}.
(Right) Inference-time comparison between 
% of TAD-repurposed TFM methods on ODDBench.
% Each grid cell
% shows the median inference-time difference between 
the TabPFN-Extension and \method, with red indicating that \method is faster.
\method scales more efficiently as both context length and dimensionality increase.}
\label{fig:overall_comparison_intro}
\vspace{-20pt}
\end{figure*}
Complementing TAD-specialized TFMs, TabPFN~\citep{
grinsztajn2026tabpfn3technicalreport}, which is a representative TFM pretrained for general-purpose tabular modeling, supports a TAD extension that harnesses its pretrained predictive capabilities for anomaly detection through
autoregressive joint-density estimation. Such repurposing avoids TAD-specific pretraining and can directly benefit from advances in general-purpose TFMs. We refer to approaches of this kind as \textit{TAD-repurposed TFMs}.
%%%%%% (Original Version) %%%%%
% \dy{Complementing these TAD-specialized PFN construction, TabPFN~\citep{
% grinsztajn2026tabpfn3technicalreport}, a representative PFN-based tabular
% foundation model (TFM) pretrained for general-purpose tabular modeling,
% supports an unsupervised TAD extension that directly repurposes its pretrained
% predictive capabilities for autoregressive joint-density estimation via
% supervised in-context prediction. Such TFM harnessing avoids TAD-specific
% pretraining and can readily benefit from advances in the underlying
% TFM. In this paper, we refer to approaches of this kind as \textit{TFM-harnessed TAD}.}
% \textit{TFM-harnessed TAD} throughout this paper.}

Both directions have notable limitations. 
TAD-specialized TFMs, relying on AD-centric pretraining with designed synthetic priors, require redesigning the priors and pretraining strategy followed by model retraining, both to expand anomaly coverage and to incorporate advances in increasingly capable TFM backbones.
The existing TAD-repurposed approach avoids such burdens, but its autoregressive formulation requires repeated attribute-wise conditional predictions, resulting in substantial inference overhead as
dimensionality increases, while reducing the notion of anomalousness to low estimated joint likelihood. We discuss these limitations in greater detail in Section~\ref{subsec:pfn_based_ad}.

To address these limitations, we propose {\method}, a new TAD framework that efficiently repurposes pretrained general-purpose TFMs by constructing \emph{virtual supervised tasks} from unlabeled, normality-aware training data.
% by leveraging their in-context predictive capabilities.
% It directly bridges unsupervised TAD and the supervised in-context inference of TFMs by constructing \textit{virtual supervised tasks} from unlabeled, normality-aware training data.
% Under the TFMs' in-context inference, 
% The virtual tasks induce the predictive \dy{structure} from the virtually labeled context under in-context inference,
These tasks induce a predictive structure through the in-context inference over the virtually labeled context, under which normal query--target pairs receive high predictive support, whereas anomalous pairs tend to violate the induced structure and receive lower support, directly providing evidence of anomalies. 
To generalize this mechanism across hundreds of datasets, we create task templates that probe different aspects of normal structure and apply data-adaptive task selection to retain tasks that best satisfy the AD-oriented properties for each dataset.

\method offers several advantages: 
(i) it enables zero-shot AD inference on unseen datasets, without dataset-specific training or tuning, while requiring neither TAD-specialized pretraining nor a dedicated backbone;
(ii) as shown in the left panel of Figure~\ref{fig:overall_comparison_intro}, it achieves stronger overall performance than existing TAD baselines, including diverse TFM-based approaches, across 790 real-world datasets; %attaining the highest Elo scores in terms of both AUC-ROC and AUC-PR; 
and (iii) as shown in the right panel of Figure~\ref{fig:overall_comparison_intro}, \method remains substantially more efficient than the TabPFN-Extension as both dataset size and dimensionality increase, demonstrating a more scalable and practical way to repurpose general-purpose TFMs for unsupervised TAD.
\section{Preliminaries and Related Works}
% \dy{In this section, we provide preliminaries on tabular foundation models utilizing PFNs to enable in-context Bayesian inference, and review existing PFN-based approaches for unsupervised TAD.}
% \dy{In this section, we provide preliminaries on the pretraining and inference
% paradigm of general-purpose TFMs, and review existing TFM-based approaches for unsupervised TAD.}
% In this section, We review the details of how PFNs are adapted to the tabular foundation models and exsting PFN-based approaches for unsupervised TAD.

% \subsection{\dy{Tabular Foundation Models via Prior-Data Fitted Networks}}
\subsection{General-Purpose Tabular Foundation Models}
\label{ssec:PFNs}
 
% Let $\mathbf A=(A_1,\ldots,A_d)$
% denote the vector of attribute random variables, and let $Y$
% % $Y\in[K]:=\{1,\ldots,K\}$ 
% denote the target random variable.
% A realized sample is written as $(\mathbf x_i,y_i)$, where
% $\mathbf x_i=(x_{i1},\ldots,x_{id})^\top\in\mathcal X\subseteq\mathbb R^d$.
% Let $\rva=(\erva_1,\ldots,\erva_d)$ denote the vector of attribute random variables, and let $\ry$ denote the target random variable. A realized sample is written as $(\vx_i,y_i)$, where $\vx_i=(x_{i1},\ldots,x_{id})^\top\in\sX\subseteq\R^d$.

\paragraph{Pretraining.}
Let $\rva=(\erva_1,\ldots,\erva_d)$ denote the vector of attribute random variables, and let $\ry$ denote the target random variable. A realized sample is written as $(\vx_i,y_i)$, where $\vx_i=(x_{i1},\ldots,x_{id})^\top\in\sX\subseteq\R^d$.
Let $\tau\in \mathcal T$ denote a \emph{task} and let $p(\tau)$ denote the \emph{prior} used to sample pretraining tasks.
For each pretraining episode, a task $\tau$ is first sampled from $p(\tau)$, after
which a labeled context and a query sample are drawn from the corresponding
task-specific distribution.
Under the task prior, the Bayesian posterior predictive distribution (PPD) is
defined as
% \begin{equation}
% \begin{gathered}
%     \tau\sim p(\tau),
%     \qquad
%     \mathcal D_{\mathrm{ctx}}
%     =
%     \{(\mathbf x_i,y_i)\}_{i=1}^{n},
%     \qquad
%     \mathcal D_{\mathrm{ctx}}
%     \cup
%     \{(\mathbf x_q,y_q)\}
%     \overset{\mathrm{i.i.d.}}{\sim}
%     P_\tau,
%     \\
%     p\!\left(
%         y_q
%         \mid
%         \mathbf x_q,
%         \mathcal D_{\mathrm{ctx}}
%     \right)
%     =
%     \int_{\mathcal T}
%         p_\tau\!\left(
%             y_q\mid\mathbf x_q
%         \right)
%         p\!\left(
%             \tau\mid\mathcal D_{\mathrm{ctx}}
%         \right)
%         \,\mathrm d\tau.
% \end{gathered}
% \label{eq:tabpfn_ppd}
% \end{equation}
\begin{equation}
\begin{gathered}
    \tau \sim p(\tau),
    \qquad
    \mathcal D_{\mathrm{ctx}}
    =
    \{(\vx_i,y_i)\}_{i=1}^{n},
    \qquad
    \mathcal D_{\mathrm{ctx}}
    \cup
    \{(\vx_q,y_q)\}
    \overset{\mathrm{i.i.d.}}{\sim}
    P_\tau,
    \\
    p\!\left(
        y_q
        \mid
        \vx_q,
        \mathcal D_{\mathrm{ctx}}
    \right)
    =
    \int_{\mathcal T}
        p_\tau\!\left(
            y_q\mid\vx_q
        \right)
        p\!\left(
            \tau\mid\mathcal D_{\mathrm{ctx}}
        \right)
        \,d\tau.
\end{gathered}
\label{eq:tabpfn_ppd}
\end{equation}

A general-purpose TFM $q_\theta$, typically instantiated as a transformer~\citep{NIPS2017_3f5ee243}, is trained to approximate the PPD by minimizing the prior-data negative log-likelihood:
\begin{equation}
    \mathcal L_{\mathrm{PFN}}(\theta)
    :=
    \E_{\mathcal D_{\mathrm{ctx}},(\vx_q,y_q)}
    \left[
        -\log
        q_\theta
        \!\left(
            y_q
            \mid
            \vx_q,
            \mathcal D_{\mathrm{ctx}}
        \right)
    \right].
    \label{eq:tabpfn_prior_nll}
\end{equation}
Assuming realizability and global optimization, the optimal predictor $q_{\theta^\star}$ recovers the PPD almost surely, i.e.,
$q_{\theta^\star}(\cdot \mid \vx_q, \mathcal D_{\mathrm{ctx}})
=
p(\cdot \mid \vx_q, \mathcal D_{\mathrm{ctx}})$.
%A detailed derivation is provided in
For a detailed derivation, see \citep{muller2024transformersbayesianinference}.

\paragraph{Inference.}
At inference time, the pretrained TFM $q_{\theta^\star}$ receives a labeled context that represents a new task, i.e., 
$\mathcal D_{\mathrm{ctx}}
:=
\{(\vx_i,y_i)\}_{i=1}^{n}
\equiv
(\mX_{\mathrm{ctx}},\vy_{\mathrm{ctx}}),
$
where
$\mX_{\mathrm{ctx}}\in\R^{n\times d}$ and
$\vy_{\mathrm{ctx}}\in\mathcal Y^n$.
Given an unlabeled query row $\vx_q$, it produces the predictive distribution over the target space,
% \begin{equation}
%     q_{\theta^\star}
%     \!\left(
%         Y_q=k
%         \mid
%         \mathbf x_q,
%         \mathcal D_{\mathrm{ctx}}
%     \right),
%     \qquad k\in[K].
%     \label{eq:tfm_predictive_interface}
% \end{equation}
\begin{equation}
    q_{\theta^\star}
    \!\left(
        Y_q
        \mid
        % \mathbf x_q,
        \vx_q,
        \mathcal D_{\mathrm{ctx}}
    \right),
    \qquad
    Y_q\in\mathcal Y.
    \label{eq:tfm_predictive_interface}
\end{equation}
By pretraining over diverse tasks, $q_{\theta^\star}$ is assumed to have the ability to infer the task-relevant structure from $\mathcal D_{\mathrm{ctx}}$ and approximate the posterior predictive distribution of $\vx_q$, thereby enabling zero-shot prediction on previously unseen tabular tasks without gradient-based parameter updates.

% \subsection{\dy{TAD-Exclusive PFNs and PFN-harnessed TAD}}
% \label{subsec:pfn_based_ad}
% \paragraph{TAD-exclusive PFNs.}
% Recent studies have developed PFNs
% specifically for tabular anomaly detection
% \cite{muller2024transformersbayesianinference}.
% Unlike PFN-based TFMs, whose in-context inference is conditioned on context consist of labeled input--target pairs, these methods receive an attribute-only
% context $\mathcal{C}
%     =
%     \{\mathbf{x}_i\}_{i=1}^{n}$
% and directly estimate the anomaly probability of a query:
% $q_{\phi^\star}
%     \left(
%         A_q = 1
%         \mid
%         \mathbf{x}_q,
%         \mathcal{C}
%     \right),$
% where \(A_q\in\{0,1\}\) denotes the nominal--anomalous status. These models are pretrained in a supervised manner on synthetic datasets containing
% labeled nominal and anomalous samples.

% \subsection{TAD-Exclusive PFNs and PFN-Harnessed TAD}
\subsection{Tabular Foundation Models for Unsupervised TAD}
\label{subsec:pfn_based_ad}

\paragraph{TAD-specialized TFMs.}

% Recent studies have developed TFMs exclusively for unsupervised TAD.
% Unlike \dy{general-purpose TFMs}, whose in-context inference is conditioned on a context consisting of labeled input--target pairs, 
Unlike general-purpose TFMs, TAD-specialized TFMs receive an unlabeled context
$\mathcal{C}
=
\{\vx_i\}_{i=1}^{n}$
of normal data and directly estimate the anomaly probability of a query,
$q_{\phi^\star}
\left(
A_q = 1
\mid
\vx_q,
\mathcal{C}
\right),$
where \(A_q\in\{0,1\}\) represents the normal-anomalous status and $q_{\phi^\star}$ denotes the pretrained TAD-specialized TFM.
To enable this inference, these models are pretrained under TAD-specific priors that generate normal contexts and normal/anomalous queries.
%To enable this inference, these models are pretrained in a supervised manner on synthetic TAD tasks containing labeled nominal and anomalous samples.

\citet{shen2025fomo0dfoundationmodelzeroshot} is an early TAD-specialized TFM pretrained under Gaussian mixture model priors with variance-inflated subspace anomalies.
\citet{ding2026zeroheroadvancingzeroshot} extends this framework through a mixture of Gaussian-mixture, structural-causal, and copula-based priors that represent several anomaly archetypes.
It further introduces a self-evolving curriculum to coordinate pretraining over heterogeneous prior families and task difficulties.
\citet{marszalek2026tacticnavigatingunknowntabular} similarly performs discriminative in-context anomaly detection,
but explicitly supports contaminated contexts.
% These approaches demonstrate that PFNs exclusive for TAD provide fast and competitive in-context detection. However, their notion of
% anomalousness is inherently tied to the anomaly-generating mechanisms encoded
% in the pretraining prior. Although an unlabeled context can characterize the
% dataset-specific nominal structure, it cannot by itself determine which form of
% deviation should be regarded as anomalous; this missing semantics must therefore
% be supplied by anomaly-centric pretraining. Expanding the detectable anomaly space consequently requires broader synthetic anomaly generators, more elaborate
% training curricula, and, renewed AD-specific pretraining. 

% Moreover, PFN architectures and pretraining strategies continue to evolve toward more accurate and scalable posterior predictive inference. However, AD-exclusive PFNs cannot readily inherit these advances and may require substantial redesign and retraining from scratch to adopt updated backbones or pretraining schemes. This limits development efficiency and makes it difficult for AD-specific PFNs to continuously benefit from the rapidly evolving ecosystem of general-purpose PFN-based TFMs.

% These approaches demonstrate that \dy{TAD-specialized TFMs} can provide fast and competitive anomaly detection. However, 
A notable limitation of these methods is that their notion of anomalousness is tied to the anomaly-generating mechanisms encoded in the pretraining prior; while a context can characterize dataset-specific normal structures, it cannot determine which deviations should be regarded as anomalous. 
% The required anomaly semantics must therefore be supplied through pretraining.
% Consequently, expanding anomaly coverage requires broader TAD priors together with appropriate training strategies, which introduces massive computational burden~\citep{ding2026zeroheroadvancingzeroshot}.
Consequently, expanding anomaly coverage requires broader TAD priors together with appropriate training strategies~\citep{ding2026zeroheroadvancingzeroshot}, necessitating retraining from scratch.
Moreover, as general-purpose TFMs continue to advance in architecture and pretraining, these specialized approaches cannot readily inherit such improvements without substantial redesign and retraining, limiting their ability to continuously benefit from the evolving TFM ecosystem.

\paragraph{TAD-repurposed TFMs.}
% An alternative direction is to harness pretrained PFN-based TFMs for TAD. \citet{grinsztajn2026tabpfn3technicalreport} pursue this direction by estimating sample likelihood through the predictive distributions of a pretrained PFN\footnote{\url{https://docs.priorlabs.ai/capabilities/anomaly-detection}}.
Recent works have successfully repurposed pretrained general-purpose TFMs for new downstream problems~\citep{xu2026needtrainrdbfoundation,
hoo2026tablestimeextendingtabpfnv2}.
% with similar directions
% recently emerging for general-purpose TFMs~\citep{
% xu2026needtrainrdbfoundation,
% hoo2026tablestimeextendingtabpfnv2}.
% \dy{Following this paradigm, TabPFN\citep{grinsztajn2026tabpfn3technicalreport} repurposes its predictive capabilities to unsupervised TAD by estimating sample likelihoods from the predictive distributions of it\footnote{\url{https://docs.priorlabs.ai/capabilities/anomaly-detection}}.}
Following this paradigm, TabPFN~\citep{grinsztajn2026tabpfn3technicalreport}
repurposes its predictive capabilities for unsupervised TAD by estimating sample
likelihoods from its predictive distributions.\footnote{\url{https://docs.priorlabs.ai/capabilities/anomaly-detection}}
It factorizes the joint feature likelihood according to the chain rule.
For a feature ordering
\(
    \pi=(\pi_1,\ldots,\pi_d)
\),
it approximates
% \begin{equation}
%     \log\widehat{p}_{\pi}
%     \left(
%         \mathbf{x}_q
%         \mid
%         \mathcal{C}
%     \right)
%     =
%     \sum_{j=1}^{d}\log
%     q_{\phi^\star}
%     \left(
%         x_{q,\pi_j}
%         \mid
%         \mathbf{x}_{q,\pi_{<j}},
%         \mathcal{D}_\mathcal{C}
%     \right), \quad\mathcal{D}_\mathcal{C}=\{(\mathbf{x}_{i,\pi_{<j}},x_{i,\pi_j})\}_{i=1}^n,
%     \label{eq:tabpfn_joint_likelihood}
% \end{equation}
% \dy{\begin{equation}
%     \log \widehat{p}_{\pi}
%     \left(
%         \mathbf{x}_q
%         \mid
%         \mathcal{C}
%     \right)
%     =
%     \sum_{j=1}^{d}
%     \log
%     q_{\phi^\star}
%     \left(
%         x_{q,\pi_j}
%         \mid
%         \mathbf{x}_{q,\pi_{<j}},
%         \mathcal{D}_{\mathcal{C},\pi,j}
%     \right),
%     \qquad
%     \mathcal{D}_{\mathcal{C},\pi,j}
%     =
%     \left\{
%         \left(
%             \mathbf{x}_{i,\pi_{<j}},
%             x_{i,\pi_j}
%         \right)
%     \right\}_{i=1}^{n},
%     \label{eq:tabpfn_joint_likelihood}
% \end{equation}}
\begin{equation}
    \log \widehat{p}_{\pi}
    \left(
        \vx_q
        \mid
        \mathcal{C}
    \right)
    =
    \sum_{j=1}^{d}
    \log
    q_{\theta^\star}
    \left(
        x_{q,\pi_j}
        \mid
        \vx_{q,\pi_{<j}},
        \mathcal{D}_{\mathcal{C},\pi,j}
    \right),
    \qquad
    \mathcal{D}_{\mathcal{C},\pi,j}
    =
    \left\{
        \left(
            \vx_{i,\pi_{<j}},
            x_{i,\pi_j}
        \right)
    \right\}_{i=1}^{n},
    \label{eq:tabpfn_joint_likelihood}
\end{equation}
where each feature is temporarily treated as the prediction target.
% continuous attributes are evaluated using predictive densities, whereas
% categorical attributes are evaluated using predictive probability
% masses.
Assuming that anomalous samples receive lower likelihood under the estimated joint distribution, the resulting negative log-likelihood can serve as the anomaly score for the query sample $\vx_q$.
%predictions are
% averaged over multiple random feature permutations to reduce dependence
% on a particular autoregressive ordering.
% The resulting negative log-likelihood serves as the anomaly score.

This extension avoids a key limitation of TAD-specialized TFMs, namely the need for substantial redesign and retraining to accommodate evolving TAD priors or improved TFM backbones, but it equates anomalousness with low joint likelihood.
While this is a principled and widely used anomaly criterion, global
rarity need not coincide with the application-relevant violation of
normality: rare but valid observations may receive high scores.
In addition, evaluating every attribute as a target across multiple feature permutations
and predictor ensembles incurs a computational cost that grows with the feature dimension~\citep{marszalek2026tacticnavigatingunknowntabular}, as shown also in Figure~\ref{fig:overall_comparison_intro}.

% \subsection{Preliminaries}
%The quality of the resulting likelihood can also be affected by the mismatch between the supervised pretraining objective and autoregressive density estimation.
%by context contamination and by the mismatch between the supervised
% pretraining objective and autoregressive density estimation
% \cite{marszalek2026tacticnavigatingunknowntabular}.

% \paragraph{Positioning of our work.}
% In this work, we develop a zero-shot anomaly detection framework that directly leverages the canonical supervised inference interface of pretrained PFN-based TFMs, without requiring anomaly-specific pretraining. General tabular predictive competence is inherited from an off-the-shelf TFM, while the anomaly criterion is specified explicitly at inference time through the observed context rather than being fixed by an anomaly-centric pretraining prior or reduced to global joint rarity.
% %This separation allows our framework to benefit from advances in general-purpose PFN-based TFMs without rebuilding a dedicated AD backbone, while enabling context-relative anomaly detection.
% This separation enables context-relative anomaly detection while allowing the framework to directly benefit from advances in general-purpose PFN-based TFMs.
%\input{Sections/methodology}
%\input{Sections/methodology_ver2}
\section{\method: Anomaly Detection via Virtual Supervised Tasks }
\label{sec:methodology}
\begin{wrapfigure}[14]{r}{0.4\textwidth}
\vspace{-12pt}
    \centering
    \includegraphics[
        width=\linewidth,
    ]{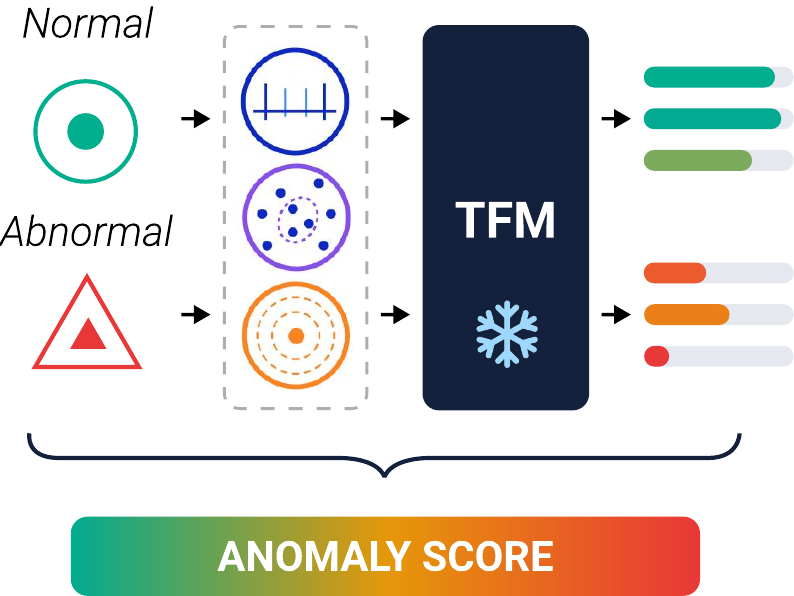}
    \vspace{-15pt}
    \caption{
    Overview of \method.
    Given a query, it assigns its virtual targets under the designed supervised tasks and distinguishes normal and anomalous queries according to the predictive supports from the pretrained TFM.
    }
    \label{fig:overview}
\end{wrapfigure}
We introduce \method, a new framework that enables general-purpose TFMs to solve anomaly detection on arbitrary datasets without dataset-specific model tuning.
% We introduce \method, a new framework for unsupervised TAD that enables
% zero-shot in-context inference by effectively harnessing pretrained PFN-based
% TFMs. 
The core idea is to construct \emph{virtual supervised tasks} turning unsupervised TAD into a supervised prediction problem that these models are designed to solve. 
% from an
% unlabeled, normality-aware context so that they can be directly processed
% through the canonical labeled-context interface of PFN-based TFMs.

As illustrated in Figure~\ref{fig:overview}, \method constructs virtual supervised tasks that allow a pretrained TFM to perform posterior-predictive inference for each query from the virtually labeled context.
These tasks are carefully designed so that the resulting anomaly score reflects the predictive support of the pretrained TFM: normal query--target pairs receive high support, while anomalous pairs tend to receive lower support.
This support gap provides the basis for anomaly detection.
% \tbf{We first motivate \method from the perspective of PFNs (\S~\ref{ssec:motivation}). 
% We then describe the construction of virtual tasks tailored to this objective (\S~\ref{ssec:VirtualNormalityAnchoredTask}--~\ref{ssec:TaskFamilesandSelection}), followed by the final detection (\S~\ref{ssec:ad}).}

% \subsection{}

\subsection{Motivation}
\label{ssec:motivation}

In the supervised setting, a task can be viewed as a latent data-generating mechanism $\tau$ that induces a joint distribution $p_{\tau}(\rvx,Y)$.
% , 
% and thus a task-specific predictive structure governing which input--target pairs are likely under that distribution. 
%and thus a task-specific predictive structure between attributes and targets.
During pretraining, general-purpose TFMs repeatedly encounter supervised datasets
sampled from diverse latent tasks and are optimized to approximate the
corresponding PPDs for new queries by inferring task-relevant predictive structure from the labeled context, as described in Equation~\ref{eq:tabpfn_ppd}.
This \emph{in-context learning} (ICL) enables TFMs to assign higher
posterior-predictive support to targets for queries that are more compatible
with the inferred task.
%This \emph{in-context learning} (ICL) makes TFMs assign higher posterior-predictive support to queries that are more compatible with the inferred task.
% This process corresponds to \emph{in-context learning} (ICL); a labeled context serves as evidence for the latent mechanism $\tau$, allowing the TFM to assign higher posterior-predictive support to queries that are more compatible with the inferred task.
% allowing the TFM to assign higher posterior-predictive support to targets for the queries more compatible with the inferred task.

We directly use this ICL capability of TFMs for unsupervised TAD.
Specifically, we construct virtual supervised tasks that generate virtually labeled contexts from the unlabeled normal samples.
% If we provide such a labeled context to query an unseen observation, the TFM's prediction indicates how normal that observation is with respect to the virtual task.
Under each virtual task, the labeled normal context induces an input--target
predictive structure through ICL, allowing a query to be evaluated by how well
its virtual target conforms to this structure.
Assuming that modern TFMs can reliably infer the predictive
structures of the virtual tasks considered in this work, we design the tasks
so that normal query--target pairs remain compatible with the induced structure,
whereas anomalous pairs are more likely to violate it.
To quantify the structural compatibility, we use the TFM's amortized
posterior-predictive support for each query--target pair.
Formally, let
\begin{equation}
    \widehat{\kappa}_m(\vx;\mathcal{C})
    :=
    q_{\theta^{\star}}
    \!\left(
        \widehat{y}_m(\vx)
        \mid
        \vx,
        \mathcal{D}^{m}_{\mathcal{C}}
    \right)
\end{equation}
be the predictive support assigned by the pretrained TFM $q_{\theta^{\star}}$ to the virtual target $\widehat{y}_m(\vx)$ for input $\vx$, where $\mathcal{D}^{m}_{\mathcal{C}}$ is the virtually labeled context induced by task $m$.
A higher value of $\widehat{\kappa}_m(\vx;\mathcal{C})$
% $q_{\theta^{\star}}
%     \!\left(
%         \widehat{y}_m(\vx)
%         \mid
%         \vx,
%         \mathcal{D}^{m}_{\mathcal{C}}
%     \right)  $
indicates greater compatibility of $(\vx, \widehat{y}_m(\vx))$ with the predictive structure induced by $m$.
% which is equivalently a higher chance of normality.

% However, not all tasks satisfy our expectations for virtual tasks;
Based on this formulation, we characterize task suitability for anomaly detection through predictive support as a measure of compatibility, and identify two required conditions:
\begin{enumerate}[left=0pt, label=\textbf{(C\arabic*)}]
    \item \textbf{Nominal Predictive Coherence.}
    Query--target pair $(\vx, \widehat{y}_m(\vx))$ generated from a normal sample $\vx$ should receive consistently high predictive
    support from TFMs.

    \item \textbf{Selective Task Coverage.}
    High predictive support should remain concentrated on pairs from the normal input distribution rather than extending broadly to deviant inputs.
    That is, $\widehat{\kappa}_m(\vx;\mathcal{C})$ should be low enough if $\vx$ is not sampled from the normal distribution.
    % Thus, the task should provide selective compatibility coverage around normally structured inputs.
\end{enumerate}

\subsection{Normality-anchored Virtual Supervised Task}
\label{ssec:VirtualNormalityAnchoredTask}
% Following the above motivation, we construct virtual supervised tasks by distilling normality-aware structure from the unlabeled context, which is assumed to consist of nominal samples. The key idea is to extract regularities shared by nominal samples and explicitly encode them into the virtual-target semantics, thereby inducing a task-specific predictive structure anchored to the nominal distribution. We refer to such tasks as \emph{normality-anchored virtual tasks}.
% From the resulting virtually labeled context, a pretrained TFM can infer this structure through ICL, under which nominal query--virtual-target pairs generated by the same task rule are expected to receive high posterior-predictive support.
Following the above motivation, we construct virtual supervised tasks by extracting normality-aware structure from the unlabeled training data used as
context, which is assumed to consist of normal samples. The key idea is to encode context-derived normal regularities into the virtual-target semantics, thereby inducing a task-specific predictive structure anchored to the normal distribution. We refer to such tasks as \emph{normality-anchored virtual tasks}.
From the resulting virtually labeled context, a pretrained TFM can infer this
structure through ICL, under which normal query--target pairs generated
by the same task rule are expected to receive high posterior-predictive support.

\begin{definition}[Normality-anchored virtual task]
\label{def:ptf_normality_anchored}
Let $P_0$ denote the normal distribution on $\mathcal{X}$ and
$\mathcal{C}=(\rvx_1,\ldots,\rvx_n)\sim P_0^n$ a clean normality-aware context.
For task $m$, a virtual supervised task constructor
$\mathfrak{T}_m=(A_m,g_m)$ induces
\begin{equation}
    \widehat{\psi}_{m,\mathcal{C}}
    =
    A_m(\mathcal{C}),
    \qquad
    \widehat{y}_m(\vx;\mathcal{C})
    =
    g_m\!\left(
        \vx;
        \widehat{\psi}_{m,\mathcal{C}}
    \right),
    \label{eq:ptf_task_constructor}
\end{equation}
where $A_m$ extracts context-dependent parameters that define the supervised task, and $g_m$ generates the virtual target for input $\vx$ based on these parameters.

We call $\mathfrak{T}_m$ \emph{normality-anchored} if there exists a
distribution-dependent functional $a_m$ of $P_0$, with
$\psi_m^0=a_m(P_0)$, such that $\widehat{\psi}_{m,\mathcal{C}}$ consistently
estimates $\psi_m^0$ and nontrivially determines the virtual-target semantics.
Specifically, there exist sequences
$\epsilon_{m,n}\to0$ and $\delta_{m,n}\to0$ such that
\begin{equation}
    \Pr_{\mathcal{C}\sim P_0^n}
    \!\left[
        d_m\!\left(
            \widehat{\psi}_{m,\mathcal{C}},
            \psi_m^0
        \right)
        >
        \epsilon_{m,n}
    \right]
    \leq
    \delta_{m,n},
    \label{eq:ptf_anchor_consistency}
\end{equation}
where $d_m$ is a metric on the task-parameter space.
\end{definition}

In Appendix~\ref{app:ProofofConditionalTaskConsistency}, we further show that normality anchoring stabilizes the induced virtual-task semantics: an independent normal query follows the same context-induced task conditional with high probability. Consequently, normal queries and their virtual-target pairs are likely to remain aligned
with the predictive structure induced by the virtually labeled normality-aware context, providing the structural basis for nominal predictive coherence (C1).

Normality anchoring biases the induced predictive structure toward normal samples, encouraging TFMs to assign higher predictive support to normal query--target pairs through ICL over the virtually labeled normality-aware context. However, this alone does not always guarantee selective task coverage (C2). Depending on which normal regularities are encoded and how they
are realized through the task construction, anomalous queries may still remain compatible with the induced structure, and the effectiveness of a given task can vary across datasets. We therefore construct diverse normality-anchored task candidates that capture different aspects of nominal structure under different
task configurations, and perform data-adaptive task selection to retain those that best satisfy both conditions for each dataset.

\subsection{Normality-Anchored Task Families}
\label{ssec:TaskFamilesandSelection}
%% 필요시 사용.
% For the task $m$, let
% $\widehat{\kappa}_m(\vx;\mathcal{C})
%     :=
%     q_{\theta^{\star}}
%     \!\left(
%         \widehat{y}_m(\vx;\mathcal{C})
%         \mid
%         \vx,
%         \mathcal{D}^{m}_{\mathcal{C}}
%     \right)$
% denote the amortized posterior-predictive support assigned by the pretrained PFN-based TFM $q_{\theta^{\star}}$ to the input--virtual-target pair, where
% $\mathcal{D}^{m}_{\mathcal{C}}$ is the virtually labeled context induced by task $m$. Higher values indicate greater compatibility with the predictive structure induced by the task.

% \subsubsection{Instantiation of Normality-Anchored Task Families}
% \label{sssec:NormalityAnchoredTaskFamilies}
Rather than assuming that a single normality-anchored virtual supervised task can satisfy both conditions above across diverse tabular datasets, we construct a family of complementary virtual task templates from the given normality-aware unlabeled context $\mathcal{C}$. Each instantiated task captures a different aspect of the nominal structure, thereby inducing a distinct context-dependent predictive structure and broadening the range of structural violations that can be exposed from diverse datasets.

We first infer the dataset's attribute profile from $\mathcal{C}$. Specifically, we classify each attribute as numerical, categorical, or constant using a hybrid attribute-type inference procedure inspired by \citet{grinsztajn2026tabpfn25advancingstateart}, and assign the dataset to a specific profile based on the relative composition of its non-constant attributes.
% We first infer the feature profile of the dataset from $\mathcal{C}$. Specifically, we classify each attribute as numerical, categorical, or constant using a hybrid feature-type inference procedure inspired by~\citep{grinsztajn2026tabpfn25advancingstateart} and assign the specific profile based on the relative composition of attributes.
% Based on the relative composition of non-constant numerical and categorical attributes, we assign each dataset to one of five profiles: \emph{numerical-only}, \emph{categorical-only}, \emph{mixed-numerical-dominant}, \emph{mixed-categorical-dominant}, or \emph{mixed}. 
This profile determines how the templates of task families are instantiated for the dataset. Detailed algorithms and profile-assignment criteria are provided in Appendix~\ref{appendix:feature_profile_inference}.

% \paragraph{Normality-anchored task families.}
% Rather than expecting that a single normality-anchored virtual task can provide suitable predictive coverage which  across diverse tabular datasets, we construct a family of complementary virtual tasks. 
% Each task is designed to capture a different aspect of nominal structure, thereby inducing a distinct context-dependent predictive relation and broadening the range of anomalous deviations that can be exposed across diverse datasets. 

%In our framework, 
Given the inferred attribute profile of the dataset, we instantiate a common set of five
complementary virtual-task templates using profile-specific constructions:
(i) \textit{single-attribute}, capturing attribute-wise dependencies between a
selected attribute and the remaining attributes;
(ii) \textit{localized subspace direction}, capturing dependencies within
selected attribute subsets;
(iii) \textit{global direction}, capturing joint dependencies across the full
attribute space;
(iv) \textit{prototype-based organization}, capturing multimodal structure
among normal samples; and
(v) \textit{distributional extremity}, capturing the relative position of a
sample within the normal distribution.
All families are constructed from the normality-aware context and therefore target
complementary aspects of its structure. Detailed profile-specific instantiations
and their correspondence to Definition~\ref{def:ptf_normality_anchored} are
provided in Appendix~\ref{appendix:virtual_task_instantiations}.

\subsection{Data-Adaptive Task Selection}
\label{sec:DataAdaptiveTaskFilter}
Among the task families instantiated from these templates, their effectiveness for TAD depends on the characteristics of the given dataset. As illustrated in Figure~\ref{fig:TaskSelection}, we therefore perform hierarchical data-adaptive task selection to retain the tasks that best satisfy the two conditions defined above. %for each dataset.

First, for each instantiated task $m$ corresponding to a specific template, we consider a predefined set of candidate task-construction configurations that determine the context-induced parameter $\widehat{\psi}_{m,\mathcal{C}}$.
Although all candidates are normality-anchored, the configuration that best satisfies the two conditions above can vary substantially across datasets. We therefore instantiate 4--8 candidates for each task $m$ and select the most suitable configuration in a data-adaptive manner.

%we partition the context into a nominal held-out subset and its complementary train-side context.
\begin{wrapfigure}[19]{r}{0.4\textwidth}
\vspace{-15pt}
    \centering
    \includegraphics[
        width=\linewidth, trim=25mm 30mm 50mm 30mm,
        clip
    ]{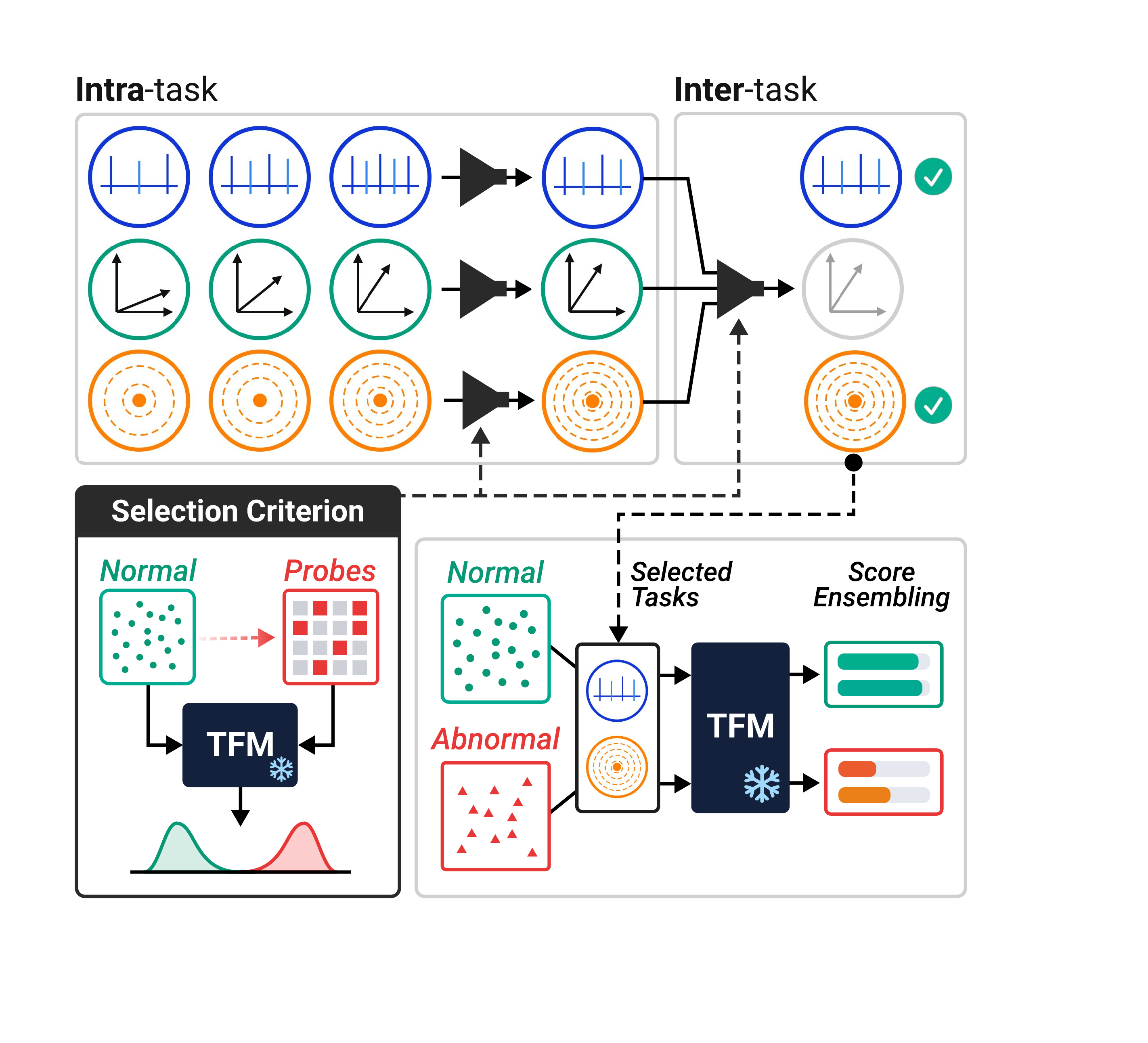}
    \vspace{-20pt}
    \caption{Overview of data-adaptive task selection. Candidates are evaluated with a pretrained TFM using nominal held-outs and structural-violation probes, and the most suitable tasks with proper configurations are retained.}
    \label{fig:TaskSelection}
\end{wrapfigure}
To select the most suitable candidate using only the available context $\mathcal{C}$,
we partition the context into a nominal held-out subset and its complementary train-side context.
% we partition the context into a nominal held-out subset and its complementary train-side context.
For split $s\in \mathcal{S}$, let $\mathcal{H}_s\subset\mathcal{C}$ denote the held-out subset and $\mathcal{C}_{-s}=\mathcal{C}\setminus\mathcal{H}_s$ the train-side context. 
For task $m$ under candidate configuration $c\in\mathcal{C}_m$, let
$\widehat{\psi}_{m,\mathcal{C}}^{\,c}$ denote the corresponding context-induced task parameter and define
% $\widehat{y}_{m}^{\,c}(\vx;\mathcal{C})$ its virtual-labeling rule.
$\widehat{y}_{m}^{\,c}(\vx;\mathcal{C})
=
g_m(\vx;\widehat{\psi}_{m,\mathcal{C}}^{\,c})$.
The virtually labeled train-side context is then
% \begin{wrapfigure}[19]{r}{0.4\textwidth}
% \vspace{-10pt}
%     \centering
%     \includegraphics[
%         width=\linewidth,
%     ]{Figures/Method_Virtual_Task_Detail.pdf}
%     \vspace{-17pt}
%     \caption{Overview of data-adaptive task selection. Candidates are evaluated with a pretrained TFM using nominal held-outs and structural-violation probes, and the most suitable task configurations and families are retained.}
%     \label{fig:TaskSelection}
% \end{wrapfigure}
\begin{equation}
    \mathcal{D}^{m,c}_{\mathcal{C}_{-s}}
    :=
    \left\{
        \left(
            \vx,
            \widehat{y}_{m}^{\,c}(\vx;\mathcal{C})
        \right)
        :
        \vx\in\mathcal{C}_{-s}
    \right\}.
    \label{eq:selection_train_context}
\end{equation}

Beyond the nominal held-outs, we additionally construct perturbed held-out samples calibrated from $\mathcal{C}_{-s}$,
%that are designed to disrupt nominal structure of inputs,
\begin{equation}
    \widetilde{\mathcal{H}}^{(r)}_s
    =
    T_r
    \!\left(
        \mathcal{H}_s;
        \mathcal{C}_{-s}
    \right),
    \qquad
    r\in\mathcal{R}_{\mathcal{C}_{-s}},
    \label{eq:structural_violation_probes}
\end{equation}
where $\mathcal{R}_{\mathcal{C}_{-s}}$ denotes the set of
structure-disrupting operators available under $\mathcal{C}_{-s}$, and $T_r$ denotes the transformation associated with operator $r$.

These samples are deliberately constructed to depart from the normal structure characterized by $\mathcal{C}_{-s}$, thereby inducing non-trivial shifts from the normal distribution. 
They serve as surrogate samples for evaluating whether each candidate exhibits sufficiently selective task coverage (C2), rather than to approximate the unknown test-time anomaly distribution. Details of their construction and distributional effects are provided in Appendix~\ref{appendix:pseudo_ood_generation}.
We evaluate each candidate on both the nominal held-out samples and the structural-violation probes using the pretrained general-purpose TFM
$q_{\theta^\star}$, conditioned on the corresponding labeled context $\mathcal{D}^{m,c}_{\mathcal{C}_{-s}}$.
For any evaluation sample $\vx$ associated with split $s$, its
posterior-predictive support is
\begin{equation}
    \widehat{\kappa}_{m,c,s}(\vx)
    :=
    q_{\theta^\star}
    \!\left(
        %\widehat{y}_{m,c}
        \widehat{y}_{m}^{\,c}(\vx;\mathcal{C})
        \mid
        \vx,
        \mathcal{D}^{m,c}_{\mathcal{C}_{-s}}
    \right).
    \label{eq:heldout_predictive_compatibility}
\end{equation}

We aggregate these supports across held-out splits into the following sets:
%into the following sets:
% \begin{equation}
% \begin{gathered}
%     \mathcal{K}^{\mathrm{nom}}_{m,c}
%     :=
%     \left\{
%         \widehat{\kappa}_{m,c,s}(\vx)
%         :
%         \vx\in\mathcal{H}_s,\;
%         s\in \mathcal{S}
%     \right\},
%     \\
%     \mathcal{K}^{\mathrm{vio}}_{m,c}
%     :=
%     \left\{
%         \widehat{\kappa}_{m,c,s}(\widetilde{\vx})
%         :
%         \widetilde{\vx}\in\widetilde{\mathcal{H}}^{(r)}_s,\;
%        r\in\mathcal{R}_{\mathcal{C}_{-s}},\;
%         s\in \mathcal{S}
%     \right\}.
% \end{gathered}
% \label{eq:heldout_compatibility_sets}
% \end{equation}
\begin{equation}
\begin{aligned}
\mathcal{K}^{\mathrm{nom}}_{m,c}
&:=
\left\{
    \widehat{\kappa}_{m,c,s}(\vx)
    :
    \vx\in\mathcal{H}_s,\;
    s\in\mathcal{S}
\right\},
\\
\mathcal{K}^{\mathrm{vio}}_{m,c}
&:=
\left\{
    \widehat{\kappa}_{m,c,s}(\widetilde{\vx})
    :
    \widetilde{\vx}\in\widetilde{\mathcal{H}}^{(r)}_s,\;
    r\in\mathcal{R}_{\mathcal{C}_{-s}},\;
    s\in\mathcal{S}
\right\}.
\end{aligned}
\label{eq:heldout_compatibility_sets}
\end{equation}
% \begin{equation}
% \begin{aligned}
%     \mathcal{K}^{\mathrm{nom}}_{m,c}
%     &:=
%     \left\{
%         \widehat{\kappa}_{m,c,s}(\mathbf{x})
%         :
%         \mathbf{x}\in\mathcal{H}_s,\;
%         s\in\mathcal{S}
%     \right\}, \\
%     \mathcal{K}^{\mathrm{vio}}_{m,c}
%     &:=
%     \left\{
%         \widehat{\kappa}_{m,c,s}(\widetilde{\mathbf{x}})
%         :
%         \widetilde{\mathbf{x}}\in\widetilde{\mathcal{H}}^{(r)}_s,\;
%         r\in\mathcal{R},\;
%         s\in\mathcal{S}
%     \right\}.
% \end{aligned}
% \label{eq:heldout_compatibility_sets}
% \end{equation}
% where $\mathcal{S}$ denotes the set of held-out splits.
For each candidate $c$, we summarize its nominal coherence (C1) and selective task coverage (C2) by
\begin{equation}
\begin{aligned}
    \vr_{m,c}
    &:=
    \Big(
        \operatorname{Med}(\mathcal{K}^{\mathrm{nom}}_{m,c}),
        % \operatorname{Med}(\mathcal{K}^{\mathrm{vio}}_{m,c}),
        \Var(\mathcal{K}^{\mathrm{nom}}_{m,c}),
        % \operatorname{Var}(\mathcal{K}^{\mathrm{vio}}_{m,c}),
        \rho_{m,c},
        \Delta_{m,c}
    \Big), \\
    \rho_{m,c}
    &:=
    \Pr_{u\sim\mathcal{K}^{\mathrm{nom}}_{m,c},\,
         v\sim\mathcal{K}^{\mathrm{vio}}_{m,c}}
    [u>v],\,
    \Delta_{m,c}
    :=
    Q_{0.25}(\mathcal{K}^{\mathrm{nom}}_{m,c})
    -
    Q_{0.75}(\mathcal{K}^{\mathrm{vio}}_{m,c}),
\end{aligned}
\label{eq:intra_selection_statistics}
\end{equation}
% \begin{wrapfigure}[20]{r}{0.4\textwidth}
% \vspace{-20pt}
%     \centering
%     \includegraphics[
%         width=\linewidth,
%     ]{Figures/Virtual_Task_Detail.pdf}
%     \vspace{-15pt}
%     \caption{Overview of data-adaptive task selection. Candidates are evaluated with a pretrained TFM using nominal held-outs and structural-violation probes, and the most suitable task configurations and families are retained.}
%     \label{fig:TaskSelection}
% \end{wrapfigure}
where $\operatorname{Med}$ and $\Var$ denote the median and variance, respectively, $Q_{\alpha}$ denotes the $\alpha$-quantile.
%, $\rho_{m,c}$ measures the pairwise ordering probability between nominal and structural-violation compatibilities.
% jointly captures the typical compatibility and dispersion of both
% sample types, their pairwise ordering, and their conservative
% inter-distribution separation. 
% Based on this evidence, We perform \emph{intra-task selection} by choosing one representative
% configuration for each task,
Based on these, we perform \emph{intra-task selection} to choose one representative configuration for each task,
\begin{equation}
    c_m^\star
    =
    \operatorname{Select}_{\mathrm{intra}}
    \left(
        \left\{
            (c,\vr_{m,c})
            :
            c\in\mathcal{C}_m
        \right\}
    \right).
    \label{eq:intra_family_selection}
\end{equation}
We then perform \emph{inter-task selection} over the resulting task
representatives to retain only the tasks most suitable for the current dataset,
%over the resulting task representatives,
\begin{equation}
    \mathcal{M}^{\star}
    =
    \operatorname{Select}_{\mathrm{inter}}
    \left(
        \left\{
            \left(
                c_m^\star,
                \vr_{m,c_m^\star}
            \right)
            :
            m\in\mathcal{M}
        \right\}
    \right),
    \label{eq:inter_family_filtering}
\end{equation}
where $\mathcal{M}^{\star}$ denotes the subset of tasks
retained for the current dataset. This second stage data-adaptively retains the tasks most suitable for TAD according to the two proposed conditions. The candidate grids and details of
$\operatorname{Select}_{\mathrm{intra}}$ and
$\operatorname{Select}_{\mathrm{inter}}$ are provided in
Appendix~\ref{appendix:task_filtering}.

% removes tasks whose
% predictive coverage is insufficiently reliable relative to the other
% candidate tasks.

% Consequently, rather than applying the entire predefined task bank uniformly,
% \method retains a dataset-adaptive subset of virtual tasks whose induced
% predictive relations exhibit both nominal coherence and sensitivity to
% controlled structural violations. 

\subsection{Anomaly Detection}
\label{ssec:ad}
% After task selection, we perform anomaly detection using only the retained virtual tasks $m\in\mathcal{M}^{\star}$ with $c_m^\star$.
We perform anomaly detection using only the retained virtual tasks $m\in\mathcal{M}^{\star}$ with their selected configurations $c_m^\star$.
For a query $\vx_q$, the pretrained TFM evaluates the predictive support
on each retained task as
% Let $m_f^{\star}$ denote the representative virtual task selected for family $f\in\mathcal{F}^{\star}$.
% For a query $\vx_q$, each retained task evaluates its predictive support
% using the pretrained PFN-based TFM as
% \begin{equation}
% $    \widehat{\kappa}_{m^{\star}}
%     (\mathcal{C},\vx_q)
$\widehat{\kappa}_{m}(\vx_q;\mathcal{C})
    =
    q_{\theta^\star}
    \left(
        \widehat{y}_{m}^{\,c_m^\star}
        (\vx_q;\mathcal{C})
        %\widehat y_m^{\,c}(\vx_q;\mathcal C)
        \mid
        \vx_q,
        % \mathcal D_{\mathcal C}^{m,c_m^\star}
        \mathcal{D}^{m,c^{\star}_m}_{\mathcal{C}}
    \right).
    $
%     \label{eq:final_task_compatibility}
% \end{equation}
This is then converted into a task-specific anomaly score,
\begin{equation}
    s_{m}
    (\vx_q)
    =
    \varphi
    \left(
        % \widehat{\kappa}_{m^{\star}}
        % (\mathcal{C},\vx_q)
        \widehat{\kappa}_{m}(\vx_q;\mathcal{C})
    \right),
    \label{eq:final_task_anomaly_score}
\end{equation}
% where $\phi$ is strictly decreasing decreasing scoring function such that lower predictive compatibility corresponds to stronger anomaly evidence.
% where $\varphi$ monotonically maps lower predictive support on task $m$ to stronger anomaly evidence.
where $\varphi$ maps lower calibrated predictive support to stronger anomaly evidence.
% where $\varphi$ is a non-increasing scoring function such that lower predictive support on task $m$ corresponds to stronger anomaly evidence.
% Since different retained tasks probe complementary nominal regularities, 
We aggregate task-specific anomaly scores from the retained tasks to obtain the final anomaly
score:
\begin{equation}
    S(\vx_q)
    =
    \operatorname{Ensemble}
    \left(
        \left\{
            s_{m}
            (\vx_q)
        \right\}_{m\in\mathcal{M}^{\star}}
    \right).
    \label{eq:final_anomaly_score}
\end{equation}
The detailed construction of $\varphi$ and the ensemble strategies are provided in Appendix~\ref{appendix:anomaly_scoring}.

\section{Experiment}
\label{sec:experiment}
\paragraph{\method setup.}
% \input{Tables/table_oddbench_top5_aucroc_formatted}
% \input{Figures/AUCROC_elo_avg_rank_dist}
% We instantiate our method with two pretrained PFN-based TFM backbones,
% TabICLv2~\citep{qu2026tabiclv2betterfasterscalable} and
% TabPFNv3~\citep{grinsztajn2026tabpfn3technicalreport}, both of which achieve strong predictive performance on TabArena~\citep{erickson2025tabarenalivingbenchmarkmachine}.
% Each backbone is used with its default configuration.
% All empirical results for \method use TabICLv2 backbone unless otherwise stated.
% We instantiate our method with the pretrained PFN-based TFM backbone,
% TabICLv2~\citep{qu2026tabiclv2betterfasterscalable}, which achieves strong predictive performance on TabArena~\citep{erickson2025tabarenalivingbenchmarkmachine}.
% The backbone is used with its default configuration.
% For generalization, we further provide the AD performance with other backbone TabPFN V2.6 and V3 in Appendix~\ref{app:MethodwDifferentTFM}.
We instantiate \method with the pretrained general-purpose TFM backbone
TabICLv2~\citep{qu2026tabiclv2betterfasterscalable}, which achieves strong
predictive performance on TabArena~\citep{NEURIPS2025_1697e3fb}.
The backbone is used with its default configuration.
%All empirical results for \method use TabICLv2 backbone unless otherwise stated.

To account for the predictive characteristics of each backbone, we perform a one-time backbone-specific calibration of the framework-level hyperparameters governing $\operatorname{Select}_{\mathrm{intra}}$, $\operatorname{Select}_{\mathrm{inter}}$, and $\operatorname{Ensemble}$ using a set of real-world datasets from ADBench~\citep{NEURIPS2022_cf93972b}, which are disjoint from the main benchmark. The hyperparameters yielding the best performance on them are then fixed for each backbone before evaluation and applied unchanged to all unseen datasets. No dataset-specific fitting or further adaptation is performed.
Details are provided in Appendix~\ref{Appendix:Configuration Search}.
% \dy{To account for the predictive characteristics of each TFM backbone, we perform a one-time framework-level function hyperparameter fitting search for $\operatorname{Select}_{\mathrm{intra}}$, $\operatorname{Select}_{\mathrm{inter}}$, and $\operatorname{Ensemble}$ using a set of real-world datasets from~\citet{han2022adbenchanomalydetectionbenchmark}, which are disjoint from ODDBench. The hyperparameters with the best average performance over AUCROC and AUCPR are then fixed for each TFM backbone and applied to all unseen evaluation datasets without further tuning or adaptation. The backbone parameters remain frozen throughout this process. Details are provided in Appendix~\ref{Appendix:Configuration Search}.}
% \subsection{Experimental Settings}
% \label{ssec:experimentalsetting}
% \paragraph{Datasets.}
% We evaluate all baselines, including \method, on ODDBench~\citep{Ding_2026}, which contains 790 real-world tabular anomaly detection datasets spanning diverse application domains and semantic anomaly types, such as fraud and system faults.
\paragraph{Experimental Settings.}
All baselines, including ours, are evaluated on ODDBench~\citep{Ding_2026}, which contains 790 real-world TAD datasets. 
%spanning diverse application domains and semantic anomaly types, such as fraud and system faults.
For each dataset, we use the same normal-only split as the training set for conventional baselines and the context set for PFN-based methods, following~\citet{marszalek2026tacticnavigatingunknowntabular}. Details of the dataset setup are provided in
Appendix~\ref{Appendix:Datasets}.

% \paragraph{Baselines.}
We compare \method with 30 baselines, including 25 conventional machine learning and deep learning methods adopted from the benchmark suite of~\citet{ICLR2024_6dfd16ff} and five recent TFM-based TAD baselines. Details of the baselines and their settings are provided in Appendix~\ref{Appendix:Baselines}.

% We evaluate detection performance using both AUCROC and AUCPR on each dataset.
Detection performance is evaluated using both AUCROC and AUCPR on each dataset.
Given the large number and heterogeneity of datasets, we complement raw performance averages with aggregate comparison metrics, such as Elo scores, to better capture relative performance across baselines. All experiments are repeated over five random seeds. Further details are provided in
Appendix~\ref{Appendix:EvaluationMetrics}.

\begin{table*}[t]
\centering
\caption{
Overall performance on ODDBench, showing the five methods with the highest average AUCROC among 31 baselines including \method. Results are averaged over five random seeds.
Parentheses indicate ranks among all 31 methods, and Total Rank is the mean rank across eight evaluation metrics; six representative metrics are shown. Best, second-best, and third-best results are highlighted in \textcolor{red}{\textbf{red bold}},
\textcolor{blue}{\underline{blue underline}}, and
\textcolor{green!60!black}{green}, respectively.
Our method achieves the strongest overall performance across the reported
metrics. Full results are provided in Appendix~\ref{app:ExtendedOverallPerformance}.
}
\vspace{-5pt}
\label{tab:oddbench-top5-aucroc}
% \large
\resizebox{\textwidth}{!}{%
\begin{tabular}{lrrrrrr}
\toprule
\textbf{Methods}
& \begin{tabular}[c]{@{}c@{}}\textbf{Avg. Rank}\\\textbf{(AUCROC) $\downarrow$}\end{tabular}
& \begin{tabular}[c]{@{}c@{}}\textbf{Avg. Rank}\\\textbf{(AUCPR) $\downarrow$}\end{tabular}
& \begin{tabular}[c]{@{}c@{}}\textbf{ELO}\\\textbf{(AUCROC) $\uparrow$}\end{tabular}
& \begin{tabular}[c]{@{}c@{}}\textbf{ELO}\\\textbf{(AUCPR) $\uparrow$}\end{tabular}
& \begin{tabular}[c]{@{}c@{}}\textbf{Top3 Ratio(\%)}\\\textbf{(AUCROC) $\uparrow$}\end{tabular}
& \begin{tabular}[c]{@{}c@{}}\textbf{Top3 Ratio(\%)}\\\textbf{(AUCPR) $\uparrow$}\end{tabular} \\
% & \begin{tabular}[c]{@{}c@{}}\textbf{Total}\\\textbf{Rank $\downarrow$}\end{tabular} \\
\midrule

DTE-NP
& \textcolor{blue}{\underline{$9.09_{\pm 0.03}\;(2)$}}
& \textcolor{blue}{\underline{$9.37_{\pm 0.02}\;(2)$}}
& \textcolor{blue}{\underline{$1167.7_{\pm 1.0}\;(2)$}}
& \textcolor{blue}{\underline{$1160.1_{\pm 0.7}\;(2)$}}
& \textcolor{blue}{\underline{$25.3_{\pm 0.6}\;(2)$}}
& $20.3_{\pm 0.5}\;(4)$ \\
% & \textcolor{blue}{\underline{$2.25\;(2)$}} \\

KNN
& \textcolor{green!60!black}{$9.83_{\pm 0.03}\;(3)$}
& \textcolor{green!60!black}{$10.38_{\pm 0.02}\;(3)$}
& \textcolor{green!60!black}{$1144.7_{\pm 0.9}\;(3)$}
& \textcolor{green!60!black}{$1130.5_{\pm 0.7}\;(3)$}
& $16.8_{\pm 0.4}\;(8)$
& $13.7_{\pm 0.6}\;(10)$ \\
% & \textcolor{green!60!black}{$4.75\;(3)$} \\

% FeatureBagging
% & $12.14_{\pm 0.20}\;(4)$
% & $13.19_{\pm 0.21}\;(10)$
% & $1086.9_{\pm 4.8}\;(4)$
% & $1061.4_{\pm 5.2}\;(10)$
% & $19.0_{\pm 1.1}\;(6)$
% & $18.0_{\pm 1.1}\;(8)$
% & $7.00\;(6)$ \\

TACTIC-Clean
& $12.22_{\pm 0.05}\;(5)$
& $11.66_{\pm 0.05}\;(4)$
& $1086.1_{\pm 1.4}\;(6)$
& $1100.2_{\pm 1.3}\;(4)$
& $17.9_{\pm 0.4}\;(7)$
& $19.0_{\pm 0.5}\;(6)$\\
%& $5.50\;(4)$ \\

OUTFORMER & $13.32_{\pm 0.05}\;(10)$ & $12.19_{\pm 0.02}\;(6)$ & $1059.1_{\pm 1.4}\;(10)$ & $1087.3_{\pm 0.5}\;(5)$ & $19.9_{\pm 0.4}\;(4)$ & \textcolor{green!60!black}{$21.5_{\pm 0.3}\;(3)$} \\ % & $6.00\;(5)$ \\

\midrule

% \textbf{\method ($\mathcal{S}=1$)}
% & \textcolor{red}{$\bm{8.39_{\pm 0.10}\;(1)}$}
% & \textcolor{red}{$\bm{7.82_{\pm 0.06}\;(1)}$}
% & \textcolor{red}{$\bm{1189.5_{\pm 3.2}\;(1)}$}
% & \textcolor{red}{$\bm{1208.1_{\pm 2.2}\;(1)}$}
% & \textcolor{red}{$\bm{44.9_{\pm 1.3}\;(1)}$}
% & \textcolor{red}{$\bm{46.3_{\pm 1.6}\;(1)}$}
% & \textcolor{red}{$\bm{1.00\;(1)}$} \\
\textbf{\method}
& \textcolor{red}{$\bm{8.39_{\pm 0.10}\;(1)}$}
& \textcolor{red}{$\bm{7.82_{\pm 0.06}\;(1)}$}
& \textcolor{red}{$\bm{1189.5_{\pm 3.2}\;(1)}$}
& \textcolor{red}{$\bm{1208.1_{\pm 2.2}\;(1)}$}
& \textcolor{red}{$\bm{44.9_{\pm 1.3}\;(1)}$}
& \textcolor{red}{$\bm{46.3_{\pm 1.6}\;(1)}$}\\
% & \textcolor{red}{$\bm{1.00\;(1)}$} \\
% \textbf{\method (TabPFNV3)}
% & \textcolor{red}{$\bm{8.39_{\pm 0.10}\;(1)}$}
% & \textcolor{red}{$\bm{7.82_{\pm 0.06}\;(1)}$}
% & \textcolor{red}{$\bm{1189.5_{\pm 3.2}\;(1)}$}
% & \textcolor{red}{$\bm{1208.1_{\pm 2.2}\;(1)}$}
% & \textcolor{red}{$\bm{44.9_{\pm 1.3}\;(1)}$}
% & \textcolor{red}{$\bm{46.3_{\pm 1.6}\;(1)}$}
% & \textcolor{red}{$\bm{1.00\;(1)}$} \\
% \textbf{\method (Heldout 1)}
% & \textcolor{red}{$\bm{8.39_{\pm 0.10}\;(1)}$}
% & \textcolor{red}{$\bm{7.82_{\pm 0.06}\;(1)}$}
% & \textcolor{red}{$\bm{1189.5_{\pm 3.2}\;(1)}$}
% & \textcolor{red}{$\bm{1208.1_{\pm 2.2}\;(1)}$}
% & \textcolor{red}{$\bm{44.9_{\pm 1.3}\;(1)}$}
% & \textcolor{red}{$\bm{46.3_{\pm 1.6}\;(1)}$}
% & \textcolor{red}{$\bm{1.00\;(1)}$} \\

\bottomrule
\end{tabular}%
}
\vspace{-5pt}
\end{table*}

% \parbox{\textwidth}{\scriptsize\textit{Notes.} Values are mean $\pm$ standard deviation across five seeds. AUC and Top-3 Ratio are reported in \%. Parentheses denote relative ranks among all 31 models. Total Rank is the mean of the eight parenthetical ranks. Best, second-best, and third-best results are shown in \textcolor{red}{\textbf{red bold}}, \textcolor{blue}{\underline{blue underline}}, and \textcolor{green!60!black}{green}, respectively.}
\begin{figure*}[t]
    \centering
    \includegraphics[width=0.95\linewidth]{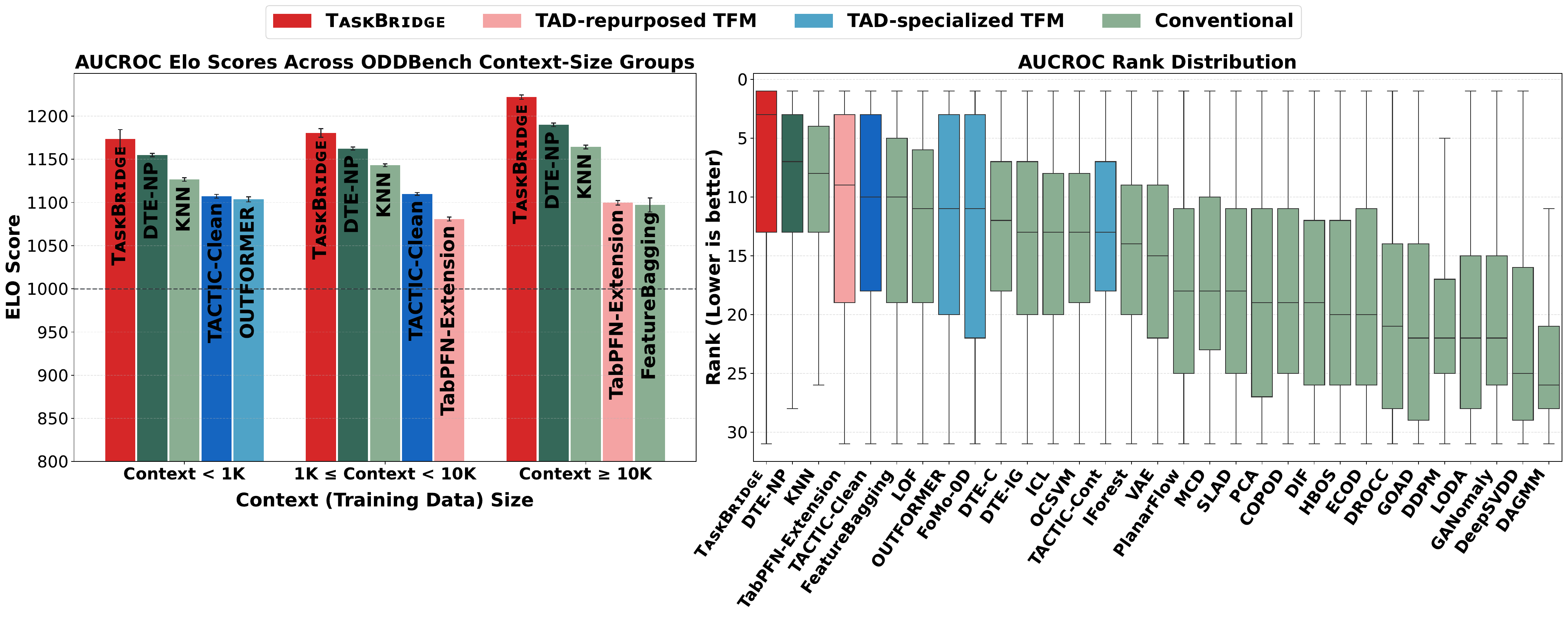}
    \vspace{-10pt}
    % \caption{Elo score comparision between total 31 baselines. (Extended version of right-side figure of Figure~\ref{fig:overall_comparison_intro}.) \method achieves the highest Elo scores in both AUCROC and AUCPR among all baselines, demonstrating our method steadily get robust ad performance across diverse table datasets.}
    \caption{(Left) Elo scores across three OODBench dataset groups with different context (training data) size 
($<1$K, $1$K--$10$K, and $\geq10$K), showing the five methods with the highest
average Elo scores in each group.
(Right) Per-dataset AUCROC rank distributions across all baselines; colors follow
Figure~\ref{fig:overall_comparison_intro}.
\method achieves the highest Elo score across all dataset groups and the
strongest overall rank distribution. Corresponding AUCPR results are provided
in Appendix~\ref{app:ExtendedOverallPerformance}.}
\label{fig:aucroc_elo_rank_dist}
    \vspace{-15pt}
\end{figure*}
\subsection{Overall Performance.}
\label{ssec:overallperformance}
As shown in Table~\ref{tab:oddbench-top5-aucroc} and Figure~\ref{fig:aucroc_elo_rank_dist}, \method achieves the strongest overall anomaly-detection performance among 31 methods across 790 datasets, ranking highest across all comparison metrics.
In particular, the left panel of Figure~\ref{fig:aucroc_elo_rank_dist} shows that \method remains consistently strong across dataset groups with diverse training/context resource sizes.
These results demonstrate that \method effectively harnesses a pretrained general-purpose TFM for TAD by bridging supervised ICL and anomaly detection through data-adaptive selection of suitable normality-anchored virtual tasks, without dataset-specific tuning or TFM retraining.
Full baseline/metric results for Table~\ref{tab:oddbench-top5-aucroc} and the corresponding AUCPR results for Figure~\ref{fig:aucroc_elo_rank_dist} are provided in Appendix~\ref{app:ExtendedOverallPerformance}.

% We provide the complete performance table, Elo-score comparisons, and per-dataset rank distributions in Appendix~\ref{appendix:OverallPerformance}, which further demonstrate the consistency of \method across heterogeneous
% datasets.
%while remaining substantially less costly than methods that require additional model training.
\begin{figure*}[t]
    \centering

    \begin{subfigure}[c]{0.43\linewidth}
        \centering
        \includegraphics[width=\linewidth]{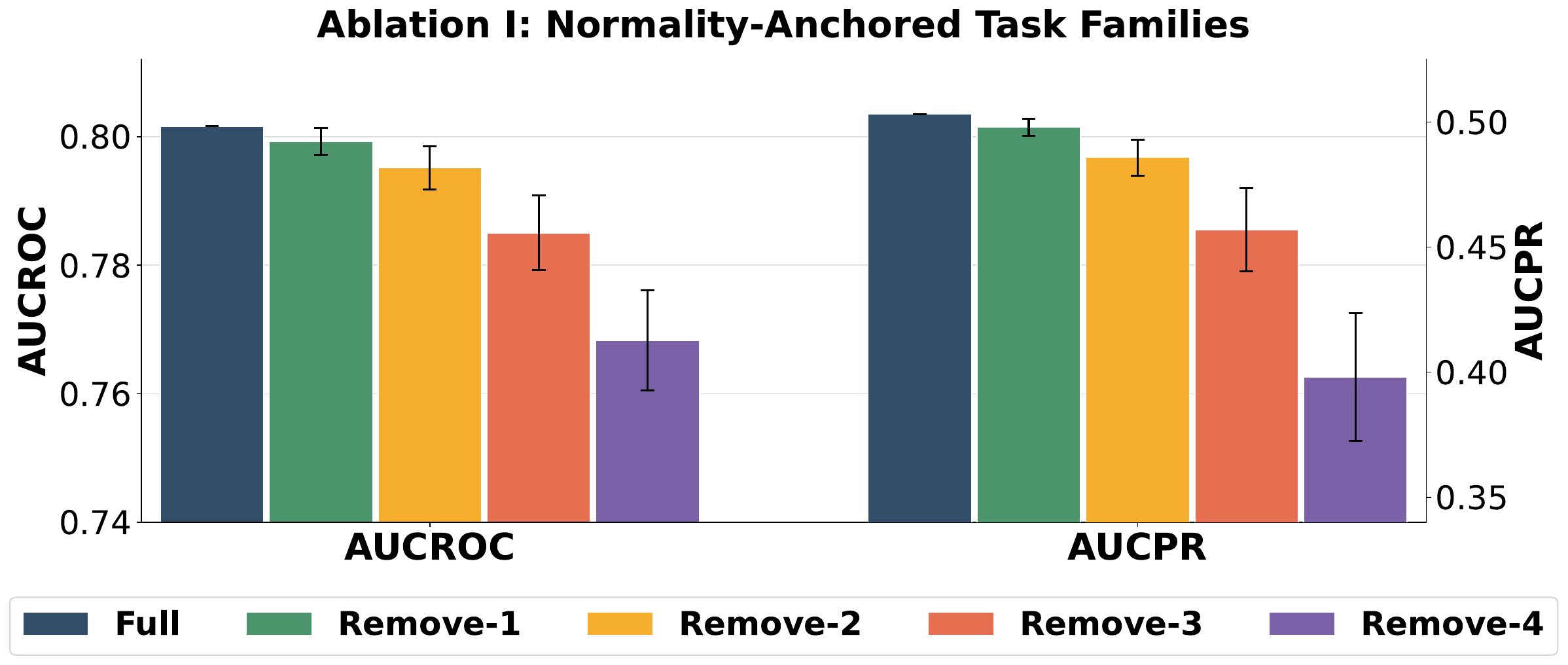}
        \label{fig:ablation1}
    \end{subfigure}
    % \hfill
    \begin{subfigure}[c]{0.43\linewidth}
        \centering
        \includegraphics[
            width=\linewidth,
            % trim=0 0 0 0.2cm,
            % clip
        ]{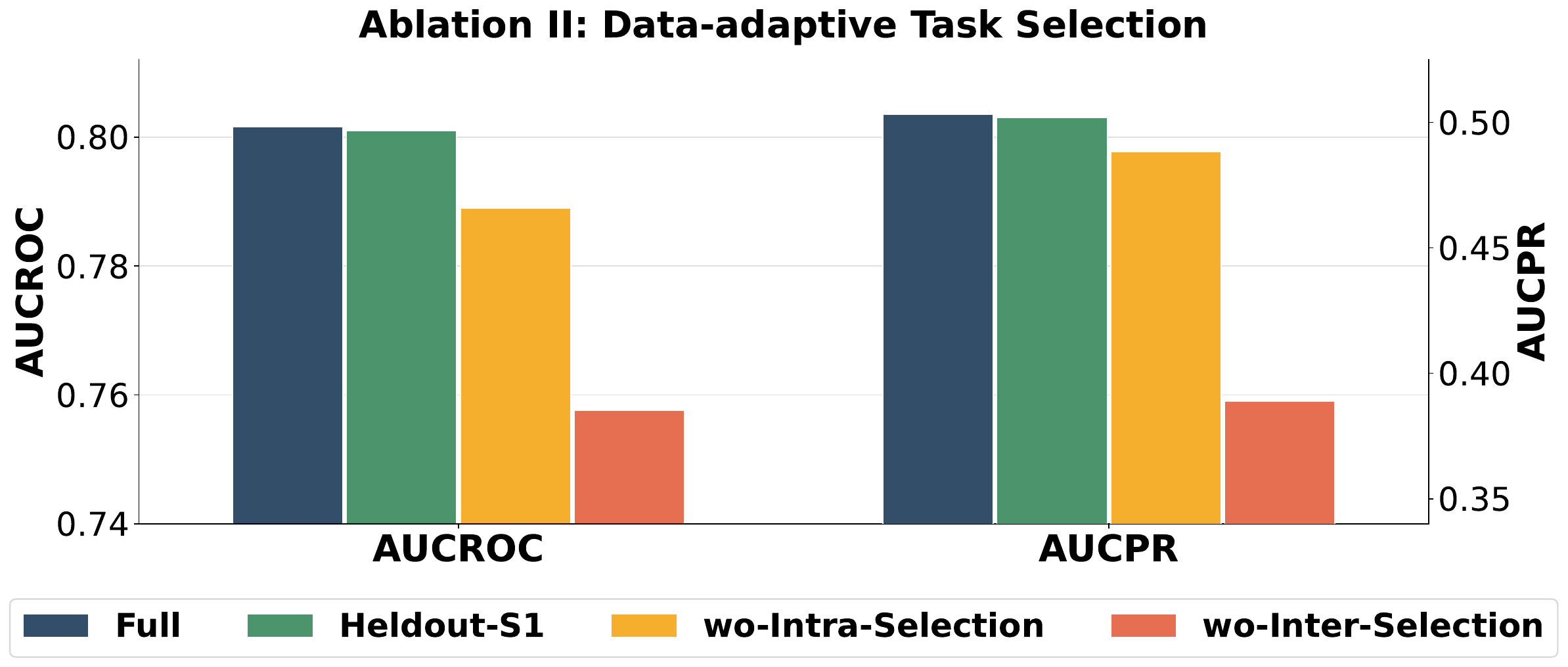}
        \label{fig:ablation2}
    \end{subfigure}

\vspace{-10pt}
\caption{
Ablation studies of \method.
(Left) Effect of the instantiated normality-anchored virtual tasks.
Remove-$i$ denotes removing $i$ tasks from the five task templates introduced
in Section~\ref{ssec:TaskFamilesandSelection}; error bars indicate the
variation across different removal combinations. The full task family achieves the best average performance.
(Right) Effect of data-adaptive task selection. We compare the full procedure
with a single held-out split, without intra-task selection, and without inter-task selection. The full selection procedure achieves the best overall
performance.
}
\label{fig:ablation-study}
\vspace{-10pt}
\end{figure*} 
\subsection{Ablation Studies}
\label{ssec:ablationstudy}
% We conduct ablation studies on two key components of \method: the normality-anchored virtual-task templates and data-adaptive task selection. 
% As shown in Figure~\ref{fig:ablation-study}, using the full set of task templates yields the best overall performance, while variants with one or more templates removed remain competitive, indicating that complementary task families provide broad anomaly coverage, while the task-selection procedure can still identify informative tasks from the remaining families. 
We ablate two key components of \method: the normality-anchored virtual-task templates and data-adaptive task selection. As shown in Figure~\ref{fig:ablation-study}, using the full set of task templates yields the best overall performance, while variants with one or more templates removed remain competitive. 
This supports the complementary anomaly coverage of the task families and the effectiveness of the task-selection procedure, which can still identify informative tasks from the remaining candidates.
% For the task selection, removing each selection stage degrades performance, with a larger drop observed when inter-task selection is removed. This highlights the importance of retaining well suited instantiated tasks  to each dataset. Interestingly, reducing the number of held-out splits utilized during task selection from the default setting $S=3$ to $S=1$ preserves the detection performance, which can lead substantially reducing task-selection cost while effective AD performance.
For task selection, removing either the intra- or inter-task stage degrades performance, with a larger drop when inter-task selection is removed, highlighting the importance of retaining dataset-suitable tasks. Reducing the number of held-out splits from the default $S=3$ to $S=1$ largely preserves detection performance while substantially lowering task-selection cost.
Further ablation details, including results for each template-removal combination, are provided in Appendix~\ref{app:AblationDetails}.
% Detailed results for different removal combinations are provided in Appendix~\ref{app:AblationDetails}.
% Further details of task-selection ablation are provided in Appendix~\ref{app:AblationDetails}, and the cost reducing effect for single held-out reported in Appendix~\ref{app:InferenceEfficiency}.
\subsection{Further Experiments}
We further analyze the inference efficiency of \method beyond the relative comparison in Figure~\ref{fig:overall_comparison_intro}. As detailed in Appendix~\ref{app:InferenceEfficiency}, task selection with the default three held-out splits takes less than one minute for most ODDBench datasets, even at context sizes near $100$K, while using a single split reduces this to below 20 seconds for most datasets while largely preserving performance, as shown in Section~\ref{ssec:ablationstudy}. 
After task selection, inference requires only milliseconds per sample.
%After task selection, detection with the retained tasks requires only millisecond-level inference per sample.

We also evaluate robustness to context contamination and limited context budgets in Appendices~\ref{app:ContaminationRobustness} and
\ref{app:ContextSizeRobustness}. 
As shown in Figures~\ref{fig:ContextContamination} and
\ref{fig:ContextSizeRobustness}, \method remains competitive under contaminated contexts compared with other TFM-based approaches, and retains strong performance even when the available normality-aware context is substantially reduced, remaining competitive with baselines that use their original context/training resources.
Definite performance degradation under heavier contamination also motivates improving task construction under imperfect contexts.

Finally, we assess backbone generalization by replacing the default TabICLv2 backbone with alternative pretrained TFM backbones, TabPFN~v2.6 and TabPFN~v3, which also show strong predictive performance on TabArena~\citep{NEURIPS2025_1697e3fb}, as reported in Appendix~\ref{app:MethodwDifferentTFM}. As shown in Tables~\ref{tab:oddbench-31-baselines-tabpfnv2p6} and \ref{tab:oddbench-31-baselines-tabpfnv3}, \method with both backbones maintains strong overall performance across ODDBench among 31 TAD baselines, demonstrating that its task construction,
selection, and anomaly-scoring framework generalizes beyond a single pretrained TFM backbone.
\section{Conclusion}
In this work, we introduced \method, a new framework for efficiently repurposing pretrained general-purpose TFMs for unsupervised tabular anomaly detection.
% \method bridges unsupervised TAD and supervised in-context prediction of TFMs by constructing normality-anchored virtual supervised tasks from an unlabeled nominal context following the conditions suitable for anomaly detection under which nominal query--target pairs remain predictively compatible while anomalous pairs tend to receive lower support from the pretrained TFM.
It bridges unsupervised TAD and supervised in-context prediction of TFMs by constructing normality-anchored virtual supervised tasks from an unlabeled normality-aware context, designed to satisfy AD-oriented conditions under which normal query--target pairs remain predictively compatible, whereas anomalous pairs tend to receive lower support from the pretrained TFM.
To adapt this mechanism to each unseen dataset, \method further employs data-adaptive task selection to retain the tasks with the most suitable predictive coverage.
% To further adapt detection to each unseen dataset,  \method employs data-adaptive task selection to retain tasks that assign high predictive support to nominal samples while separating them from anomalies.
Across 790 real-world datasets, \method achieves the strongest overall performance among 31 methods while preserving efficient inference with repurposed TFMs. 

A limitation of our work is that \method is designed to repurpose general-purpose PFN-based TFMs for TAD under the assumption that the context consists of clean normal samples. As shown in our appendix experiments, performance degrades when the context is heavily contaminated with anomalies. An important direction for future work is therefore to extend \method to settings where the available context is not restricted to clean normal samples, including fully unsupervised and semi-supervised TAD. Another promising direction is to extend the framework to more structured tabular domains, such as relational databases.

\bibliography{iclr2027_conference}

@misc{muller2024transformersbayesianinference,
      title={Transformers Can Do Bayesian Inference}, 
      author={Samuel Müller and Noah Hollmann and Sebastian Pineda Arango and Josif Grabocka and Frank Hutter},
      year={2024},
      eprint={2112.10510},
      archivePrefix={arXiv},
      primaryClass={cs.LG},
      url={https://arxiv.org/abs/2112.10510}, 
}

@misc{qu2026tabiclv2betterfasterscalable,
      title={TabICLv2: A better, faster, scalable, and open tabular foundation model}, 
      author={Jingang Qu and David Holzmüller and Gaël Varoquaux and Marine Le Morvan},
      year={2026},
      eprint={2602.11139},
      archivePrefix={arXiv},
      primaryClass={cs.LG},
      url={https://arxiv.org/abs/2602.11139}, 
}

@misc{grinsztajn2026tabpfn25advancingstateart,
      title={TabPFN-2.5: Advancing the State of the Art in Tabular Foundation Models}, 
      author={Léo Grinsztajn and Klemens Flöge and Oscar Key and Felix Birkel and Philipp Jund and Brendan Roof and Benjamin Jäger and Dominik Safaric and Simone Alessi and Adrian Hayler and Mihir Manium and Rosen Yu and Felix Jablonski and Shi Bin Hoo and Anurag Garg and Jake Robertson and Magnus Bühler and Vladyslav Moroshan and Lennart Purucker and Clara Cornu and Lilly Charlotte Wehrhahn and Alessandro Bonetto and Bernhard Schölkopf and Sauraj Gambhir and Noah Hollmann and Frank Hutter},
      year={2026},
      eprint={2511.08667},
      archivePrefix={arXiv},
      primaryClass={cs.LG},
      url={https://arxiv.org/abs/2511.08667}, 
}

@misc{ding2026zeroheroadvancingzeroshot,
      title={From Zero to Hero: Advancing Zero-Shot Foundation Models for Tabular Outlier Detection}, 
      author={Xueying Ding and Haomin Wen and Simon Klüttermann and Leman Akoglu},
      year={2026},
      eprint={2602.03018},
      archivePrefix={arXiv},
      primaryClass={cs.LG},
      url={https://arxiv.org/abs/2602.03018}, 
}

@misc{shen2025fomo0dfoundationmodelzeroshot,
      title={FoMo-0D: A Foundation Model for Zero-shot Tabular Outlier Detection}, 
      author={Yuchen Shen and Haomin Wen and Leman Akoglu},
      year={2025},
      eprint={2409.05672},
      archivePrefix={arXiv},
      primaryClass={cs.LG},
      url={https://arxiv.org/abs/2409.05672}, 
}

@misc{marszalek2026tacticnavigatingunknowntabular,
      title={TACTIC for Navigating the Unknown: Tabular Anomaly deteCTion via In-Context inference}, 
      author={Patryk Marszałek and Tomasz Kuśmierczyk and Marek Śmieja},
      year={2026},
      eprint={2603.14171},
      archivePrefix={arXiv},
      primaryClass={cs.LG},
      url={https://arxiv.org/abs/2603.14171}, 
}

@misc{hollmann2023tabpfntransformersolvessmall,
      title={TabPFN: A Transformer That Solves Small Tabular Classification Problems in a Second}, 
      author={Noah Hollmann and Samuel Müller and Katharina Eggensperger and Frank Hutter},
      year={2023},
      eprint={2207.01848},
      archivePrefix={arXiv},
      primaryClass={cs.LG},
      url={https://arxiv.org/abs/2207.01848}, 
}

@inproceedings{NEURIPS2022_cf93972b,
 author = {Han, Songqiao and Hu, Xiyang and Huang, Hailiang and Jiang, Minqi and Zhao, Yue},
 booktitle = {Advances in Neural Information Processing Systems},
 doi = {10.52202/068431-2329},
 editor = {S. Koyejo and S. Mohamed and A. Agarwal and D. Belgrave and K. Cho and A. Oh},
 pages = {32142--32159},
 publisher = {Curran Associates, Inc.},
 title = {ADBench: Anomaly Detection Benchmark},
 url = {https://proceedings.neurips.cc/paper_files/paper/2022/file/cf93972b116ca5268827d575f2cc226b-Paper-Datasets_and_Benchmarks.pdf},
 volume = {35},
 year = {2022}
}

@article{10.1145/3464423,
author = {Fernando, Tharindu and Gammulle, Harshala and Denman, Simon and Sridharan, Sridha and Fookes, Clinton},
title = {Deep Learning for Medical Anomaly Detection – A Survey},
year = {2021},
issue_date = {September 2022},
publisher = {Association for Computing Machinery},
address = {New York, NY, USA},
volume = {54},
number = {7},
issn = {0360-0300},
url = {https://doi.org/10.1145/3464423},
doi = {10.1145/3464423},
journal = {ACM Comput. Surv.},
month = jul,
articleno = {141},
numpages = {37}
}

@article{ALHASHEDI2021100402,
title = {Financial fraud detection applying data mining techniques: A comprehensive review from 2009 to 2019},
journal = {Computer Science Review},
volume = {40},
pages = {100402},
year = {2021},
issn = {1574-0137},
doi = {https://doi.org/10.1016/j.cosrev.2021.100402},
url = {https://www.sciencedirect.com/science/article/pii/S1574013721000423},
author = {Khaled Gubran Al-Hashedi and Pritheega Magalingam}
}

@article{https://doi.org/10.1002/ett.4150,
author = {Ahmad, Zeeshan and Shahid Khan, Adnan and Wai Shiang, Cheah and Abdullah, Johari and Ahmad, Farhan},
title = {Network intrusion detection system: A systematic study of machine learning and deep learning approaches},
journal = {Transactions on Emerging Telecommunications Technologies},
volume = {32},
number = {1},
pages = {e4150},
doi = {https://doi.org/10.1002/ett.4150},
url = {https://onlinelibrary.wiley.com/doi/abs/10.1002/ett.4150},
eprint = {https://onlinelibrary.wiley.com/doi/pdf/10.1002/ett.4150},
year = {2021}
}

@INPROCEEDINGS{4781136,
  author={Liu, Fei Tony and Ting, Kai Ming and Zhou, Zhi-Hua},
  booktitle={2008 Eighth IEEE International Conference on Data Mining}, 
  title={Isolation Forest}, 
  year={2008},
  volume={},
  number={},
  pages={413-422},
  doi={10.1109/ICDM.2008.17}}

@inproceedings{10.1145/342009.335388,
author = {Breunig, Markus M. and Kriegel, Hans-Peter and Ng, Raymond T. and Sander, J{\"o}rg},
title = {LOF: identifying density-based local outliers},
year = {2000},
isbn = {1581132174},
publisher = {Association for Computing Machinery},
address = {New York, NY, USA},
url = {https://doi.org/10.1145/342009.335388},
doi = {10.1145/342009.335388},
booktitle = {Proceedings of the 2000 ACM SIGMOD International Conference on Management of Data},
pages = {93–104},
numpages = {12},
location = {Dallas, Texas, USA},
series = {SIGMOD '00}
}

@inproceedings{ICLR2024_6dfd16ff,
 author = {Livernoche, Victor and Jain, Vineet and Hezaveh, Yashar and Ravanbakhsh, Siamak},
 booktitle = {International Conference on Learning Representations},
 editor = {B. Kim and Y. Yue and S. Chaudhuri and K. Fragkiadaki and M. Khan and Y. Sun},
 pages = {25836--25866},
 title = {On Diffusion Modeling for Anomaly Detection},
 url = {https://proceedings.iclr.cc/paper_files/paper/2024/file/6dfd16ff880a63fee9f6469fee58a496-Paper-Conference.pdf},
 volume = {2024},
 year = {2024}
}

@inproceedings{Ding_2026,
   title={MacrOData: New Benchmarks of Thousands of Datasets for Tabular Outlier Detection},
   url={http://dx.doi.org/10.1145/3770855.3817520},
   DOI={10.1145/3770855.3817520},
   booktitle={Proceedings of the 32nd ACM SIGKDD Conference on Knowledge Discovery and Data Mining V.2},
   publisher={ACM},
   author={Ding, Xueying and Klüttermann, Simon and Wen, Haomin and Chen, Yilong and Akoglu, Leman},
   year={2026},
   month=Aug, pages={8777–8788} }

@inproceedings{10.1145/3701716.3715196,
author = {Chen, Sihan and Qian, Zhuangzhuang and Siu, Wingchun and Hu, Xingcan and Li, Jiaqi and Li, Shawn and Qin, Yuehan and Yang, Tiankai and Xiao, Zhuo and Ye, Wanghao and Zhang, Yichi and Dong, Yushun and Zhao, Yue},
title = {PyOD 2: A Python Library for Outlier Detection with LLM-powered Model Selection},
year = {2025},
isbn = {9798400713316},
publisher = {Association for Computing Machinery},
address = {New York, NY, USA},
url = {https://doi.org/10.1145/3701716.3715196},
doi = {10.1145/3701716.3715196},
booktitle = {Companion Proceedings of the ACM on Web Conference 2025},
pages = {2807–2810},
numpages = {4},
location = {Sydney NSW, Australia},
series = {WWW '25}
}

@misc{wei2026icladincontextlearningunified,
      title={ICLAD: In-Context Learning for Unified Tabular Anomaly Detection Across Supervision Regimes}, 
      author={Jack Yi Wei and Narges Armanfard},
      year={2026},
      eprint={2603.19497},
      archivePrefix={arXiv},
      primaryClass={cs.LG},
      url={https://arxiv.org/abs/2603.19497}, 
}

@article{elo1967proposed,
  title={The proposed uscf rating system, its development, theory, and applications},
  author={Elo, Arpad E},
  journal={Chess life},
  volume={22},
  number={8},
  pages={242--247},
  year={1967}
}

@inproceedings{NEURIPS2025_1697e3fb,
 author = {Erickson, Nick and Purucker, Lennart and Tschalzev, Andrej and Holzm\"{u}ller, David and Desai, Prateek and Salinas, David and Hutter, Frank},
 booktitle = {Advances in Neural Information Processing Systems},
 doi = {10.52202/085713-0519},
 editor = {D. Belgrave and C. Zhang and H. Lin and R. Pascanu and P. Koniusz and M. Ghassemi and N. Chen},
 pages = {},
 publisher = {Curran Associates, Inc.},
 title = {TabArena: A Living Benchmark for Machine Learning on Tabular Data},
 url = {https://proceedings.neurips.cc/paper_files/paper/2025/file/1697e3fb412da11dc9488249f9e7bbc9-Paper-Datasets_and_Benchmarks_Track.pdf},
 volume = {38, Main Conference},
 year = {2025}
}

@misc{xu2026needtrainrdbfoundation,
      title={No Need to Train Your RDB Foundation Model}, 
      author={Linjie Xu and Yanlin Zhang and Quan Gan and Minjie Wang and David Wipf},
      year={2026},
      eprint={2602.13697},
      archivePrefix={arXiv},
      primaryClass={cs.AI},
      url={https://arxiv.org/abs/2602.13697}, 
}

@misc{hoo2026tablestimeextendingtabpfnv2,
      title={From Tables to Time: Extending TabPFN-v2 to Time Series Forecasting}, 
      author={Shi Bin Hoo and Samuel Müller and David Salinas and Frank Hutter},
      year={2026},
      eprint={2501.02945},
      archivePrefix={arXiv},
      primaryClass={cs.LG},
      url={https://arxiv.org/abs/2501.02945}, 
}

@misc{grinsztajn2026tabpfn3technicalreport,
      title={TabPFN-3: Technical Report}, 
      author={Léo Grinsztajn and Klemens Flöge and Oscar Key and Felix Birkel and Philipp Jund and Brendan Roof and Mihir Manium and Shi Bin Hoo and Magnus Bühler and Anurag Garg and Dominik Safaric and Jake Robertson and Benjamin Jäger and Simone Alessi and Adrian Hayler and Vladyslav Moroshan and Lennart Purucker and Philipp Singer and Alan Arazi and Julien Siems and Jan Hendrik Metzen and Georg Grab and Nick Erickson and Siyuan Guo and Eliott Kalfon and Simon Bing and David Salinas and Clara Cornu and Lilly Charlotte Wehrhahn and Diana Kriuchkova and Kursat Kaya and Lydia Sidhoum and Marie Salmon and Jerry Chen and Madelon Hulsebos and Yann LeCun and Samuel Müller and Bernhard Schölkopf and Sauraj Gambhir and Noah Hollmann and Frank Hutter},
      year={2026},
      eprint={2605.13986},
      archivePrefix={arXiv},
      primaryClass={cs.LG},
      url={https://arxiv.org/abs/2605.13986}, 
}

@misc{eggert2023tablibdataset627mtables,
      title={TabLib: A Dataset of 627M Tables with Context}, 
      author={Gus Eggert and Kevin Huo and Mike Biven and Justin Waugh},
      year={2023},
      eprint={2310.07875},
      archivePrefix={arXiv},
      primaryClass={cs.CL},
      url={https://arxiv.org/abs/2310.07875}, 
}

@article{10.1214/aos/1079120141,
author = {David R. Hunter},
title = {{MM algorithms for generalized Bradley-Terry models}},
volume = {32},
journal = {The Annals of Statistics},
number = {1},
publisher = {Institute of Mathematical Statistics},
pages = {384 -- 406},
year = {2004},
doi = {10.1214/aos/1079120141},
URL = {https://doi.org/10.1214/aos/1079120141}
}

@article{Maziarka_2021,
   title={OneFlow: One-class flow for anomaly detection based on a minimal volume region},
   ISSN={1939-3539},
   url={http://dx.doi.org/10.1109/TPAMI.2021.3108223},
   DOI={10.1109/tpami.2021.3108223},
   journal={IEEE Transactions on Pattern Analysis and Machine Intelligence},
   publisher={Institute of Electrical and Electronics Engineers (IEEE)},
   author={Maziarka, Lukasz and Smieja, Marek and Sendera, Marcin and Struski, Lukasz and Tabor, Jacek and Spurek, Przemyslaw},
   year={2021},
   pages={1–1} }

@InProceedings{Graham_2023_CVPR,
    author    = {Graham, Mark S. and Pinaya, Walter H.L. and Tudosiu, Petru-Daniel and Nachev, Parashkev and Ourselin, Sebastien and Cardoso, Jorge},
    title     = {Denoising Diffusion Models for Out-of-Distribution Detection},
    booktitle = {Proceedings of the IEEE/CVF Conference on Computer Vision and Pattern Recognition (CVPR) Workshops},
    month     = {June},
    year      = {2023},
    pages     = {2948-2957}
}

@article{10.1145/3606274.3606277,
author = {Ma, Martin Q. and Zhao, Yue and Zhang, Xiaorong and Akoglu, Leman},
title = {The Need for Unsupervised Outlier Model Selection: A Review and Evaluation of Internal Evaluation Strategies},
year = {2023},
issue_date = {June 2023},
publisher = {Association for Computing Machinery},
address = {New York, NY, USA},
volume = {25},
number = {1},
issn = {1931-0145},
url = {https://doi.org/10.1145/3606274.3606277},
doi = {10.1145/3606274.3606277},
journal = {SIGKDD Explor. Newsl.},
month = jul,
pages = {19–35},
numpages = {17}
}

@inproceedings{NEURIPS2022_3e9113e2,
 author = {Ding, Xueying and Zhao, Lingxiao and Akoglu, Leman},
 booktitle = {Advances in Neural Information Processing Systems},
 doi = {10.52202/068431-0698},
 editor = {S. Koyejo and S. Mohamed and A. Agarwal and D. Belgrave and K. Cho and A. Oh},
 pages = {9603--9616},
 publisher = {Curran Associates, Inc.},
 title = {Hyperparameter Sensitivity in Deep Outlier Detection: Analysis and a Scalable Hyper-Ensemble Solution},
 url = {https://proceedings.neurips.cc/paper_files/paper/2022/file/3e9113e2bc2e700baa7d765470f140e1-Paper-Conference.pdf},
 volume = {35},
 year = {2022}
}

@inproceedings{NIPS2017_3f5ee243,
 author = {Vaswani, Ashish and Shazeer, Noam and Parmar, Niki and Uszkoreit, Jakob and Jones, Llion and Gomez, Aidan N and Kaiser, \L ukasz and Polosukhin, Illia},
 booktitle = {Advances in Neural Information Processing Systems},
 editor = {I. Guyon and U. Von Luxburg and S. Bengio and H. Wallach and R. Fergus and S. Vishwanathan and R. Garnett},
 pages = {},
 publisher = {Curran Associates, Inc.},
 title = {Attention is All you Need},
 url = {https://proceedings.neurips.cc/paper_files/paper/2017/file/3f5ee243547dee91fbd053c1c4a845aa-Paper.pdf},
 volume = {30},
 year = {2017}
}
\bibliographystyle{iclr2027_conference}

\appendix
\section{Method Details}
\label{Appendix:MethodDetails}
\subsection{Consistency of Normality-Anchored Virtual Tasks}
\label{app:ProofofConditionalTaskConsistency}
% The following proposition formalizes how the normality-anchored virtual task defined in Definition~\ref{def:ptf_normality_anchored} controls the conditional task structure encountered by a nominal query. 
We formalize how normality anchoring stabilizes the virtual-task semantics for
normal queries. Specifically, the following proposition shows that, under a
local-stability condition, the task conditional induced by a finite normality-aware
context agrees with its population-level counterpart with high probability.
% The proof is provided in Appendix~\ref{app:ProofofConditionalTaskConsistency}.
\begin{proposition}
[Consistency of the context-induced task conditional for normal queries]
\label{prop:nominal_conditional_task_consistency}
Under Definition~\ref{def:ptf_normality_anchored}, recall that $\widehat{y}_m(\vx;\mathcal{C})=g_m(\vx;\widehat{\psi}_{m,\mathcal{C}})$.
Define the
deterministic task conditionals induced by the context-estimated
and population-level parameters as
\begin{equation}
    \widehat{\pi}_{m,\mathcal{C}}(y\mid\vx)
    :=
    \1
    \left\{
        y
        =
        g_m(\vx;\widehat{\psi}_{m,\mathcal{C}})
    \right\},
    \qquad
    \pi_m^0(y\mid\vx)
    :=
    \1
    \left\{
        y
        =
        g_m(\vx;\psi_m^0)
    \right\}.
    \label{eq:context_and_population_task_conditionals}
\end{equation}
Suppose that the virtual-task conditional is locally stable at
$\psi_m^0$ under $P_0$. Specifically, let
\begin{equation}
    \omega_m(t)
    :=
    \sup_{\psi:\,d_m(\psi,\psi_m^0)\le t}
    \Pr_{\rvx\sim P_0}
    \left[
        g_m(\rvx;\psi)
        \neq
        g_m(\rvx;\psi_m^0)
    \right],
    \label{eq:conditional_task_stability}
\end{equation}
and assume that $\omega_m(t)\to0$ as $t\to0$.

Then, for an independent normal query
$\rvx_q\sim P_0$,
\begin{equation}
    \Pr_{\substack{
        \mathcal{C}\sim P_0^n\\
        \rvx_q\sim P_0
    }}
    \left[
        \widehat{\pi}_{m,\mathcal{C}}
        (\cdot\mid\rvx_q)
        =
        \pi_m^0(\cdot\mid\rvx_q)
    \right]
    =
    % &\qquad=
    \Pr_{\substack{
        \mathcal{C}\sim P_0^n\\
        \rvx_q\sim P_0
    }}
    \left[
        \widehat{y}_m(\rvx_q;\mathcal{C})
        =
        g_m(\rvx_q;\psi_m^0)
    \right]
    % &\qquad
    \ge
    1
    -
    \omega_m(\epsilon_{m,n})
    -
    \delta_{m,n}.
    \label{eq:nominal_conditional_task_consistency}
\end{equation}
\end{proposition}

\begin{proof}
Define the successful anchoring event
\begin{equation}
    \mathcal{E}_{m,n}
    :=
    \left\{
        d_m
        \left(
            \widehat{\psi}_{m,\mathcal{C}},
            \psi_m^0
        \right)
        \le
        \epsilon_{m,n}
    \right\}.
    \label{eq:successful_anchoring_event}
\end{equation}
By Definition~\ref{def:ptf_normality_anchored},
\begin{equation}
    \Pr_{\mathcal{C}\sim P_0^n}
    \left[
        \mathcal{E}_{m,n}^{c}
    \right]
    \le
    \delta_{m,n}.
    \label{eq:anchoring_event_failure}
\end{equation}
Consider any realized context $\mathcal{C}$ for which
$\mathcal{E}_{m,n}$ holds. Then
$d_m
    \left(
        \widehat{\psi}_{m,\mathcal{C}},
        \psi_m^0
    \right)
    \le
    \epsilon_{m,n}$.
By the definition of $\omega_m$ in
Equation~\ref{eq:conditional_task_stability},
\begin{equation}
    \Pr_{\rvx_q\sim P_0}
    \left[
        \widehat{y}_m(\rvx_q;\mathcal{C})
        \neq
        g_m(\rvx_q;\psi_m^0)
    \right]
    \le
    \omega_m(\epsilon_{m,n}).
    \label{eq:conditional_disagreement_under_anchoring}
\end{equation}
Let
\begin{equation}
    \mathcal{D}_{m,n}
    :=
    \left\{
        \widehat{y}_m(\rvx_q;\mathcal{C})
        \neq
        g_m(\rvx_q;\psi_m^0)
    \right\}
    \label{eq:nominal_target_disagreement_event}
\end{equation}
denote the event that the context-induced and population-level
task conditionals assign different virtual targets to the
normal query.

Equation~\ref{eq:conditional_disagreement_under_anchoring}
holds for every realized context satisfying
$\mathcal{E}_{m,n}$. Therefore, averaging the corresponding
query-wise disagreement probabilities over all such context
realizations gives
\begin{equation}
    \Pr
    \left(
        \mathcal{D}_{m,n}
        \mid
        \mathcal{E}_{m,n}
    \right)
    =
    \E_{\mathcal{C}\mid\mathcal{E}_{m,n}}
    \left[
        \Pr_{\rvx_q\sim P_0}
        \left[
            \widehat{y}_m(\rvx_q;\mathcal{C})
            \neq
            g_m(\rvx_q;\psi_m^0)
        \right]
    \right]
    \le
    \omega_m(\epsilon_{m,n}).
    \label{eq:disagreement_given_successful_anchoring}
\end{equation}
Here, the expectation is taken over context realizations conditioned on successful anchoring. 
Since each query-wise disagreement probability inside the expectation is bounded by $\omega_m(\epsilon_{m,n})$, their conditional average satisfies the same bound.

When $\mathcal{E}_{m,n}^{c}$ occurs, the estimated task
parameter is not guaranteed to lie within
$\epsilon_{m,n}$ of $\psi_m^0$. We therefore use the
worst-case bound
\begin{equation}
    \Pr
    \left(
        \mathcal{D}_{m,n}
        \mid
        \mathcal{E}_{m,n}^{c}
    \right)
    \le
    1.
    \label{eq:disagreement_given_failed_anchoring}
\end{equation}

Applying the law of total probability over the successful and
failed anchoring events yields
\begin{align}
    \Pr(\mathcal{D}_{m,n})
    &=
    \Pr
    \left(
        \mathcal{D}_{m,n}
        \mid
        \mathcal{E}_{m,n}
    \right)
    \Pr(\mathcal{E}_{m,n})
    \nonumber\\
    &\quad+
    \Pr
    \left(
        \mathcal{D}_{m,n}
        \mid
        \mathcal{E}_{m,n}^{c}
    \right)
    \Pr(\mathcal{E}_{m,n}^{c})
    \nonumber\\
    &\le
    \omega_m(\epsilon_{m,n})
    \Pr(\mathcal{E}_{m,n})
    +
    \Pr(\mathcal{E}_{m,n}^{c})
    \nonumber\\
    &\le
    \omega_m(\epsilon_{m,n})
    +
    \delta_{m,n},
    \label{eq:overall_conditional_structure_disagreement}
\end{align}
where the final inequality follows from
$\Pr(\mathcal{E}_{m,n})\le1$ and
Equation~\ref{eq:anchoring_event_failure}.

By the deterministic definitions in
Equation~\ref{eq:context_and_population_task_conditionals},
the target-disagreement event $\mathcal{D}_{m,n}$ is equivalent
to
$\left\{
        \widehat{\pi}_{m,\mathcal{C}}
        (\cdot\mid\rvx_q)
        \neq
        \pi_m^0(\cdot\mid\rvx_q)
    \right\}$.
Taking complements in
Equation~\ref{eq:overall_conditional_structure_disagreement}
therefore gives
\begin{equation}
    \Pr_{\substack{
        \mathcal{C}\sim P_0^n\\
        \rvx_q\sim P_0
    }}
    \left[
        \widehat{\pi}_{m,\mathcal{C}}
        (\cdot\mid\rvx_q)
        =
        \pi_m^0(\cdot\mid\rvx_q)
    \right]
    \ge
    1
    -
    \omega_m(\epsilon_{m,n})
    -
    \delta_{m,n},
\end{equation}
which proves
Equation~\ref{eq:nominal_conditional_task_consistency}.
% This proposition shows that normality anchoring stabilizes the virtual-task semantics: as the context-estimated parameter approaches its population counterpart, a fresh nominal query follows the same task conditional with high probability. Consequently, nominal query--target pairs remain aligned with the predictive structure induced by the virtually labeled nominal context.
% Taking complements proves Equation~\eqref{eq:nominal_conditional_task_consistency}.
% Moreover, Definition~\ref{def:ptf_normality_anchored} gives
% $\epsilon_{m,n}\to0$ and $\delta_{m,n}\to0$, while local
% stability gives
% $\omega_m(\epsilon_{m,n})\to0$. Thus, the lower bound in
% Equation~\eqref{eq:nominal_conditional_task_consistency}
% approaches one.
\end{proof}
Thus, normality anchoring stabilizes the virtual-task semantics for normal
queries: as the context-estimated task parameter approaches its population
counterpart, a fresh normal query follows the same task conditional with high
probability. Consequently, normal query--target pairs are likely to
remain aligned with the predictive structure induced by the virtually labeled
normal context, providing the structural basis for nominal predictive
coherence.

\subsection{Feature-Type and Data-Profile Inference}
\label{appendix:feature_profile_inference}

Before constructing profile-specific virtual tasks, we infer the type of each
observed attribute directly from the normality-aware context. Attribute-type inference is performed before preprocessing so that the original
cardinality structure of each attribute is preserved.

Let
\(
    \mathcal{C}=\{\vx_i\}_{i=1}^{n}
\)
denote the context with $d$ observed attributes. For attribute $j$, let
$n_j$ denote the number of non-missing observations and $u_j$ the number of unique values among them. 
% We additionally define the integer-likeness
% indicator
% \begin{equation}
%     I_j^{\mathrm{int}}
%     =
%     \mathbbm{1}
%     \left[
%         x_{ij}\approx\operatorname{round}(x_{ij}),
%         \quad
%         \forall i \text{ such that } x_{ij} \text{ is observed}
%     \right].
% \end{equation}
An attribute is regarded as constant if it contains no valid observations or
if $u_j\leq1$. Otherwise, we define an adaptive cardinality threshold
\begin{equation}
    t_j
    =
    \min
    \left(
        20,\,
        \max\left(2,\left\lceil\sqrt{n_j}\right\rceil\right)
    \right).
    \label{eq:feature_type_dynamic_threshold}
\end{equation}
Attribute $j$ is classified as categorical if
\begin{equation}
    \left(u_j < 10 \;\lor\; u_j \leq t_j\right)
    \;\land\;
    \left(I_j^{\mathrm{int}} = 1 \;\lor\; u_j \leq t_j\right)
    \;\land\;
    \frac{u_j}{n_j} \leq 0.5,
    \label{eq:categorical_feature_inference}
\end{equation}
where $I_j^{\mathrm{int}}$ is an indicator of whether the observed values of
attribute $j$ are integer-like.
All remaining non-constant attributes are classified as numerical.
This hybrid criterion combines cardinality, integer-likeness, and the
proportion of unique values to distinguish low-cardinality categorical
attributes from numerical ones.

\paragraph{Attribute-profile assignment.}
Let
$d_{\mathrm{num}}$, $d_{\mathrm{cat}}$, and $d_{\mathrm{const}}$
denote the numbers of attributes inferred as numerical, categorical, and
constant, respectively. Since constant attributes do not provide a meaningful
type preference, we exclude them when computing the relative composition of
numerical and categorical attributes. Specifically, letting
% \begin{equation}
$d_{\mathrm{use}}=d_{\mathrm{num}}+d_{\mathrm{cat}},$
% \end{equation}
we define
\begin{equation}
    r_{\mathrm{num}}
    =
    \frac{d_{\mathrm{num}}}{d_{\mathrm{use}}},
    \qquad
    r_{\mathrm{cat}}
    =
    \frac{d_{\mathrm{cat}}}{d_{\mathrm{use}}},
    \label{eq:feature_profile_ratios}
\end{equation}
whenever $d_{\mathrm{use}}>0$.

The feature profile of each dataset is then determined as
\begin{equation}
\operatorname{Profile}(\mathcal{C})
=
\begin{cases}
\texttt{numerical-only},
    & d_{\mathrm{cat}}=0,\; d_{\mathrm{num}}>0,
    \\[1mm]
\texttt{categorical-only},
    & d_{\mathrm{num}}=0,\; d_{\mathrm{cat}}>0,
    \\[1mm]
\texttt{mixed-numerical-dominant},
    & d_{\mathrm{num}},d_{\mathrm{cat}}>0
      \ \text{and}\ 
      r_{\mathrm{num}}\geq\tau_{\mathrm{dom}},
    \\[1mm]
\texttt{mixed-categorical-dominant},
    & d_{\mathrm{num}},d_{\mathrm{cat}}>0
      \ \text{and}\ 
      r_{\mathrm{cat}}\geq\tau_{\mathrm{dom}},
    \\[1mm]
\texttt{mixed},
    & \text{otherwise}.
\end{cases}
\label{eq:data_profile_assignment}
\end{equation}
We use the dominance threshold
\(
    \tau_{\mathrm{dom}}=0.80
\)
in all experiments. The inferred profile is used only to determine the appropriate instantiations of virtual-task templates for the observed attribute composition.
% Specifically, we use automatic profile-based routing, while the subsequent task construction and selection are performed within the corresponding active task families. Constant attributes do not contribute to the numerical--categorical dominance ratio.

\subsection{Profile-Specific Normality-Anchored Virtual-Task Instantiations}
\label{appendix:virtual_task_instantiations}

We describe the concrete instantiations of the five normality-anchored virtual-task templates
introduced in Section~\ref{ssec:TaskFamilesandSelection}. Let
$\mathcal{J}_{\mathrm{num}}$ and $\mathcal{J}_{\mathrm{cat}}$ denote the sets
of numerical and categorical attributes inferred from the context $\mathcal{C}$, respectively. Each task construction produces a discrete virtual
target so that the resulting labeled context can be directly processed by a pretrained TFM backbone. While the underlying templates are shared
across attribute profiles, the concrete operator is adapted to the profile of the observed attributes. For notation, we use $\operatorname{QBin}_{K}(z;\mathcal{C})$ to denote the assignment of a scalar statistic $z$ to one of $K$ classes according to empirical quantile boundaries estimated from $\mathcal{C}$.

\paragraph{\textcircled{1} Single attribute: Reg2Class.}
This template probes attribute-wise predictive dependencies in the normal distribution by treating an individual attribute as the prediction target and the remaining attributes as predictors. We instantiate this template using \emph{Reg2Class}, a regression-to-classification construction motivated by the synthetic classification-task generation paradigm commonly used in general-purpose TFMs~\citep{qu2026tabiclv2betterfasterscalable, hollmann2023tabpfntransformersolvessmall}. Rather than directly predicting a continuous attribute, Reg2Class discretizes its values using context-derived boundaries to form a classification target.

For a numerical target attribute $j$, we construct
\begin{equation}
    \widehat{y}_{m}(\vx;\mathcal{C})
    =
    \operatorname{QBin}_{K}
    \left(
        x_j;\mathcal{C}
    \right),
    \label{eq:reg2class_num}
\end{equation}
where the class boundaries are empirical quantiles of attribute $j$ in the context. Candidate numerical attributes are prioritized according to their variation.
% , and attributes that are constant or fail to induce at least two
% classes are discarded.

For a categorical target attribute, the observed categories themselves define
the virtual classes:
\begin{equation}
    \widehat{y}_{m}(\vx;\mathcal{C})
    =
    \operatorname{CatIndex}_{\mathcal{C}}(x_j),
    \label{eq:reg2class_cat}
\end{equation}
where $\operatorname{CatIndex}_{\mathcal{C}}$ maps categories observed in the
context to discrete class indices. Candidate categorical attributes are
screened according to their cardinality and empirical category structure so
that degenerate or excessively sparse classification tasks are avoided.

Accordingly, Reg2Class is instantiated using numerical attributes for the
\emph{numerical-only} profile and categorical attributes for the
\emph{categorical-only} profile. For the \emph{mixed} profile, both numerical
and categorical candidates are considered, with the corresponding
numerical-binning or categorical-mapping rule applied according to the
selected target attribute. For mixed profiles dominated by one attribute type, the dominant attribute group is preferentially used to construct the
axis-aligned task.

\paragraph{\textcircled{2} Localized subspace direction: Masked Projection.}
The second task template probes localized predictive dependencies of the normal distribution that are expressed within a subset of attributes.
The underlying motivation is that an anomalous observation may violate the predictive structure of a particular subspace while remaining compatible with the normal regularities expressed by the remaining attributes. We therefore construct a scalar projection from a randomly sampled attribute subset and discretize the resulting statistic into virtual classes.

For numerical attributes, Masked Random Subspace Projection (MRSP) samples
a subset
$
    \mathcal{S}_m\subseteq\mathcal{J}_{\mathrm{num}}
$
and constructs
\begin{equation}
    z_m(\vx)
    =
    \vw_m^\top
    T_{\mathcal{C}}
    \left(
        \vx_{\mathcal{S}_m}
    \right),
    \qquad
    \widehat{y}_{m}(\vx;\mathcal{C})
    =
    \operatorname{QBin}_{K}
    \left(
        z_m(\vx);\mathcal{C}
    \right),
    \label{eq:mrsp_rule}
\end{equation}
where $T_{\mathcal{C}}$ denotes the context-fitted numerical transformation
and $\vw_m$ is a randomly sampled projection vector that is fixed for
the instantiated task. 
% Thus, different sampled subspaces provide localized views of the numerical dependency structure.

For categorical attributes, we instead use Masked Categorical Random
Projection (MCRP). Given a sampled categorical subset, each observed
attribute--category token is assigned a random weight modulated by its empirical rarity in the context. 
% \dy{For category value $v$ of attribute $j$, we use a smoothed rarity statistic of the form}
% \begin{equation}
%     r_j(v)
%     =
%     \log
%     \frac{
%         n+\alpha(K_j+1)
%     }{
%         N_j(v)+\alpha
%     },
% \end{equation}
% and form
% \begin{equation}
%     z_m(\mathbf{x})
%     =
%     \frac{1}{\sqrt{|\mathcal{S}_m|}}
%     \sum_{j\in\mathcal{S}_m}
%     r_j(x_j)\,\xi_{j,x_j},
%     \label{eq:masked_cat_projection}
% \end{equation}
% where $\xi_{j,v}$ is a task-specific random weight. 
For category value $v$ of attribute $j$, we define the smoothed rarity statistic
\begin{equation}
    r_j(v)
    =
    \log
    \frac{
        n+\alpha(K_j+1)
    }{
        N_j(v)+\alpha
    },
\end{equation}
where $n$ is the number of context samples, $K_j$ is the number of categories observed for attribute $j$, $N_j(v)$ denotes the number of context samples with category value $v$, and $\alpha>0$ is a smoothing constant.
We then construct
\begin{equation}
    z_m(\vx)
    =
    \frac{1}{\sqrt{|\mathcal{S}_m|}}
    \sum_{j\in\mathcal{S}_m}
    r_j(x_j)\,\xi_{j,x_j},
    \label{eq:masked_cat_projection}
\end{equation}
where $\mathcal{S}_m$ is the sampled categorical attribute subset and
$\xi_{j,v}$ is a task-specific random weight.
%associated with category $v$ of attribute $j$.
The resulting scalar score is again partitioned using context-derived quantiles $\operatorname{QBin}_{K}$.

For mixed attributes, Masked Mixed Random Projection (MMRP) first converts numerical attributes into context-dependent quantile tokens while retaining categorical values as categorical tokens. The resulting heterogeneous token view is then processed using the same rarity-weighted random projection mechanism for categorical attributes.

Therefore, the concrete masked operator is selected according to the inferred profile.
Numerical-only data use MRSP, categorical-only data use MCRP, and balanced mixed data use MMRP. For
mixed-numerical-dominant data, we retain both the numerical MRSP and the
MMRP; analogously, mixed-categorical-dominant data use both the
MCRP and MMRP. 
% In the dominant mixed profiles, the sampled subset range is enlarged to cover from approximately $30\%$ to $100\%$ of the corresponding dominant feature group.

\paragraph{\textcircled{3} Global direction: Global Projection.}
This template probes global predictive dependencies of the normal distribution that are expressed across a broader set of attributes. The underlying motivation is that some anomalous observations may violate distributed multivariate relationships that are not confined to a localized subspace. Accordingly, this family captures global directional predictive structure by constructing projections over the full feature view. % rather than a sampled attribute subset.

For numerical-only data, Global Random Projection uses all available
numerical attributes:
\begin{equation}
    z_m(\vx)
    =
    \vw_m^\top
    T_{\mathcal{C}}(\vx_{\mathcal{J}_{\mathrm{num}}}),
    \qquad
    \widehat{y}_{m}(\vx;\mathcal{C})
    =
    \operatorname{QBin}_{K}
    \left(
        z_m(\vx);\mathcal{C}
    \right).
\end{equation}
For categorical-only data, Global Categorical Random Projection applies the
rarity-weighted categorical projection to the complete categorical feature
set rather than a sampled subset.

For mixed data, Mixed Random Projection uses the complete mixed token view:
numerical attributes are converted to context-derived quantile tokens and
categorical attributes retain their categorical identities. A
rarity-weighted random projection over all resulting tokens then defines the
scalar statistic used for virtual-class construction.

In the current implementation, numerical-only and categorical-only profiles
use their corresponding type-specific global projections, whereas the mixed
and both mixed-dominant profiles use the Mixed Random Projection. 
% In
% particular, the dominant mixed profiles do not additionally include a
% type-specific non-masked global projection; their dominant-type structure is
% instead represented through the corresponding masked family together with
% the global mixed view.

\paragraph{\textcircled{4} Prototype-based multimodal organization: Clustering.}
This template probes multimodal predictive structure of the normal distribution through context-fitted prototypes. The underlying motivation is that normal data may consist of multiple modes with distinct attribute relationships, such that an anomalous sample can become incompatible with the prototype structure induced by the context. Unlike projection-based tasks, clustering represents this structure through multiple nominal prototypes rather than a single projected ordering.

For numerical-only data, we apply $K$-means to the context-transformed
numerical attributes and use the resulting cluster index as the virtual
target:
\begin{equation}
    \widehat{y}_{m}(\vx;\mathcal{C})
    =
    \arg\min_{k\in[K]}
    \left\|
        T_{\mathcal{C}}(\vx)
        -
        \widehat{\vmu}_{k,\mathcal{C}}
    \right\|_2^2.
    \label{eq:kmeans_virtual_rule}
\end{equation}

% For categorical-only data, we replace Euclidean clustering with categorical
% $K$-modes. Each prototype consists of the feature-wise modal categories of
% its assigned context samples, and assignment is based on categorical mismatch:
% \begin{equation}
%     d_{\mathrm{cat}}
%     (\mathbf{x},\mathbf{p}_k)
%     =
%     \sum_{j\in\mathcal{J}_{\mathrm{cat}}}
%     \mathbbm{1}[x_j\neq p_{k,j}].
% \end{equation}

For categorical-only data, we replace Euclidean clustering with categorical
$K$-modes. Let $\widehat{\vp}_{k,\mathcal{C}}$ denote the
attribute-wise modal-category prototype estimated from the context samples
assigned to cluster $k$. The categorical discrepancy is the Hamming distance
\begin{equation}
    d_{\mathrm{cat}}
    (\vx,\vp_k)
    =
    \sum_{j\in\mathcal{J}_{\mathrm{cat}}}
    \1[x_j\neq p_{k,j}].
    \label{eq:kmodes_categorical_distance}
\end{equation}
The categorical virtual target is then the index of the nearest modal
prototype:
\begin{equation}
    \widehat{y}_{m}(\vx;\mathcal{C})
    =
    \arg\min_{k\in[K]}
    d_{\mathrm{cat}}
    \left(
        \vx_{\mathcal{J}_{\mathrm{cat}}},
        \widehat{\vp}_{k,\mathcal{C}}
    \right).
    \label{eq:kmodes_virtual_rule}
\end{equation}
Thus, two samples receive the same virtual class when they are assigned to the
same context-fitted categorical prototype.
%, rather than when they share any single category value. 
The prototypes are refined by alternating nearest-prototype assignments and attribute-wise modal updates. 
% If small-cluster merging
% is enabled, a cluster whose context support falls below the prescribed minimum
% size is mapped to its nearest sufficiently supported prototype; this mapping
% is then applied consistently to both context and query assignments.

For mixed profiles, we use a $K$-prototypes-style construction on a
context-fitted mixed representation.
First, numerical attributes are converted to context-derived quantile tokens, as in mixed random projection.
% Specifically, numerical attributes are first transformed using statistics estimated from $\mathcal{C}$ and then discretized into context quantile bins. 
Let $\widetilde{\vx}_{\mathrm{num},\mathcal{C}}$
denote the resulting numerical bin-code vector, while
$\vx_{\mathcal{J}_{\mathrm{cat}}}$ retains the categorical values. Each mixed
prototype contains the coordinate-wise median numerical bin code and the
attribute-wise modal categorical value of its assigned context samples. We
define
\begin{equation}
    d_{\mathrm{num}}
    (\vx,\vp_k;\mathcal{C})
    =
    \frac{1}{|\mathcal{J}_{\mathrm{num}}|}
    \sum_{j\in\mathcal{J}_{\mathrm{num}}}
    \left(
        \widetilde{x}_{j,\mathcal{C}}-p_{k,j}
    \right)^2
\end{equation}
and the normalized categorical mismatch
\begin{equation}
    \overline{d}_{\mathrm{cat}}
    (\vx,\vp_k)
    =
    \frac{1}{|\mathcal{J}_{\mathrm{cat}}|}
    \sum_{j\in\mathcal{J}_{\mathrm{cat}}}
    \1[x_j\neq p_{k,j}].
\end{equation}
The mixed prototype distance and corresponding virtual target are
\begin{equation}
    d_{\mathrm{mix}}
    (\vx,\vp_k;\mathcal{C})
    =
    d_{\mathrm{num}}
    (\vx,\vp_k;\mathcal{C})
    +
    % \gamma\,
    \overline{d}_{\mathrm{cat}}
    (\vx,\vp_k),
    \label{eq:kprototypes_mixed_distance}
\end{equation}
\begin{equation}
    \widehat{y}_{m}(\vx;\mathcal{C})
    =
    \arg\min_{k\in[K]}
    d_{\mathrm{mix}}
    \left(
        \vx,
        \widehat{\vp}_{k,\mathcal{C}};\mathcal{C}
    \right).
    \label{eq:kprototypes_virtual_rule}
\end{equation}
% where $\gamma$ controls the relative contribution of categorical mismatch
% (the implementation uses $\gamma=1$ unless configured otherwise). 
As with categorical $K$-modes, the prototypes are obtained by alternating assignment and prototype-update steps. 
%and optional small-cluster merging is applied after fitting.

Accordingly, the numerical-only, categorical-only, and mixed profiles use
$K$-means, categorical $K$-modes, and the context-discretized mixed
$K$-prototypes construction, respectively. Both mixed-dominant profiles also
use the mixed $K$-prototypes construction.

% \dy{For mixed profiles, we use a $K$-prototypes-style construction that combines the numerical and categorical distances. Each prototype contains numerical location statistics and categorical modal values, and the mixed distance is
% of the form
% \dy{\begin{equation}
%     d_{\mathrm{mix}}(\mathbf{x},\mathbf{p}_k)
%     =
%     d_{\mathrm{num}}
%     (\mathbf{x}_{\mathrm{num}},\mathbf{p}_{k,\mathrm{num}})
%     +\,
%     d_{\mathrm{cat}}
%     (\mathbf{x}_{\mathrm{cat}},\mathbf{p}_{k,\mathrm{cat}}),
% \end{equation}}
% with the nearest prototype defining the virtual class.}

% Accordingly, the numerical-only, categorical-only, and mixed profiles use
% $K$-means, categorical $K$-modes, and mixed $K$-prototypes, respectively.
% Both mixed-dominant profiles also use the mixed $K$-prototypes construction.
% rather than a dominant-type-specific clustering rule, allowing the task to
% retain information from both feature types.

\paragraph{Distributional extremity: Radial and Rarity-based Constructions.}
The last fifth template probes predictive structures associated with
context-relative distributional extremity. The underlying motivation is that the normality-aware context induces not only relational structure among attributes, but also a characteristic organization of samples according to their relative position within the nominal support. An anomalous sample exhibiting a scale shift or other support-relative deviation may therefore be assigned an extremity level that is incompatible with the predictive structure induced by the context.

For numerical-only data, Radial Bin computes the radial magnitude of the
context-transformed numerical vector,
\begin{equation}
    \rho(\vx)
    =
    \left\|
        T_{\mathcal{C}}(\vx)
    \right\|_2,
\end{equation}
and assigns the corresponding virtual target using context-derived quantile
bins:
$\widehat{y}_{m}(\vx;\mathcal{C})
    =
    \operatorname{QBin}_{K}
    \left(
        \rho(\vx);\mathcal{C}
    \right)$.
%This construction partitions the context according to relative radial position rather than a particular coordinate or projection direction.

For categorical-only data, Categorical Rarity Bin estimates the smoothed
empirical occurrence probability of each category from the normality-aware
context. Specifically, for category value $v$ of attribute $j$, we define
\begin{equation}
    \widehat{p}_{j,\mathcal{C}}(v)
    =
    \frac{
        N_j(v)+\alpha
    }{
        n+\alpha(K_j+1)
    },
\end{equation}
where $n$ is the number of context samples, $K_j$ is the number of categories
observed for attribute $j$, $N_j(v)$ denotes the number of context samples
taking value $v$ at attribute $j$, and $\alpha>0$ is a smoothing constant.
A sample-level rarity statistic is then constructed as
\begin{equation}
    R(\vx)
    =
    -
    \frac{1}{|\mathcal{J}_{\mathrm{cat}}|}
    \sum_{j\in\mathcal{J}_{\mathrm{cat}}}
    \log
    \widehat{p}_{j,\mathcal{C}}(x_j),
    \label{eq:categorical_rarity_score}
\end{equation}
where $x_j$ is the category value of sample $\vx$ at attribute $j$.
Thus, categories that occur less frequently in the normality-aware context contribute
more strongly to $R(\vx)$. The resulting rarity statistic is
subsequently discretized using context-derived empirical quantiles to form
the virtual classification target.

For mixed profiles, Mixed Categorical Rarity Bin first converts numerical
attributes into context-derived quantile tokens and combines them with the
original categorical tokens. The rarity score is then evaluated over this
joint token representation and discretized into virtual classes. This
construction is used for the balanced mixed profile as well as both
mixed-dominant profiles.

\paragraph{Normality-anchoring characteristics of virtual tasks.}
Despite their different targeting predictive structures, all of the above constructions share
a common source of task anchoring. Their virtual-label semantics depend on
quantities estimated from the normality-aware context, including empirical
quantile boundaries, category supports and frequencies, numerical
transformations, and fitted prototypes. Random components such as sampled
attribute subsets or projection weights determine which view of the attributes
is probed, but are sampled once for the instantiation and kept fixed across
context and query labeling. 

Consequently, each instantiated task can be expressed in the form introduced in Equation~\ref{eq:ptf_task_constructor}. Under standard consistency of the corresponding empirical statistics, these context-dependent parameters converge to population quantities determined by $P_0$. Together with their nontrivial influence on the induced virtual-label semantics, the instantiated task families therefore satisfy the normality-anchoring principle of Definition~\ref{def:ptf_normality_anchored}.

% The context-dependent quantities subsequently determine how that view is partitioned into virtual classes.
% Consequently, each instantiated task can be expressed in the form
% as introduced in Equation.~\ref{eq:ptf_task_constructor}. Under the standard
% consistency of the corresponding empirical statistics, these
% context-dependent parameters converge to population quantities determined by
% $P_0$. 
% %Moreover, changes in these population quantities can alter the induced virtual-class assignments. 
% The instantiated task families therefore satisfy
% the normality-anchoring principle of
% Definition~\ref{def:ptf_normality_anchored}.

\subsection{Held-out Evaluation Sample Generation}
\label{appendix:pseudo_ood_generation}

For each dataset and random seed, we construct a common registry of held-out splits that is shared across all candidate virtual tasks. 
This ensures that candidate tasks are evaluated against the same normal observations.
Let $\mathcal{S}:=[S]=\{1,\ldots,S\}$ denote the set of held-out splits.
% rather than candidate-specific validation subsets. 
For split $s\in\mathcal{S}$, let $\mathcal{H}_s$ denote the nominal held-out subset,
$\widetilde{\mathcal{H}}^{(r)}_s$ the set of structural-violation samples
generated by applying the structure-disrupting operator $r$, and
$\mathcal{C}_{-s}$ the complementary train-side context. We set $S=3$ as the default throughout our framework.
\paragraph{Nominal held-out subsets.}
These sets are constructed from disjoint subsets of the original normality-aware context $\mathcal{C}$. Given a hold-out ratio $\rho_H$, the number of held-out samples per
split is
% \begin{equation}
%     n_{\mathrm{H}}
%     =
%     \min
%     \left\{
%         n_{\max},
%         \left\lceil \rho n \right\rceil,
%         n-1
%     \right\},
%     \label{eq:heldout_split_size}
% \end{equation}
\begin{equation}
    n_{\mathrm{H}}
    =
    \min
    \left\{
        n_{\max},
        \left\lceil \rho_H n \right\rceil,
        \left\lfloor \frac{n}{S} \right\rfloor,
        n-1
    \right\},
    \label{eq:heldout_split_size}
\end{equation}
where $n$ denotes the context size, $S$ the number of held-out splits, and $n_{\max}$ the maximum number of held-out samples per split. In our experiments, we set $n_{\max}=2048$, while $\rho_H$ is selected through the backbone-specific configuration calibration described in Appendix~\ref{Appendix:Configuration Search}.
% We set $n_{\max}=2048$, and consider
% $\rho\in\{0.10,0.15,0.20\}$ in our search configuration.

% \paragraph{Context-calibrated structural-violation probes.}
% For each held-out split $s$, we generate probe set $\{\widetilde{\mathcal{H}}^{(r)}_s|r\in \mathcal{R}_{\mathcal{C}_{-s}} \}$ whose total size matches the size of $\mathcal{H}_s$. The size budget is distributed as evenly as possible
% across the active perturbation operators, 
% % Let $\mathcal{R}_s$ denote the set of active operators for split $s$. 
% approximately $|\mathcal{H}_s|/|\mathcal{R}_{\mathcal{C}_{-s}}|$ for each
% $r\in\mathcal{R}_{\mathcal{C}_{-s}}$ perturbed samples are generated, with any
% remainder assigned deterministically across the operators.
\paragraph{Context-calibrated structural-violation held-out samples.}
For each held-out split $s$, we generate operator-specific probe sets
$\{\widetilde{\mathcal{H}}^{(r)}_s
: r\in\mathcal{R}_{\mathcal{C}_{-s}}\}$ whose combined sample budget equals
$|\mathcal{H}_s|$. The budget is distributed as evenly as possible across the
active perturbation operators, such that each operator generates approximately
$|\mathcal{H}_s|/|\mathcal{R}_{\mathcal{C}_{-s}}|$ samples.
% with any remainder assigned deterministically across the operators.

The active operators in $\mathcal{R}_{\mathcal{C}_{-s}}$ are column shuffling, subset replacement, and scaled jitter. Each operator is constructed using the complementary train-side context $\mathcal{C}_{-s}$. For every perturbed sample $\widetilde{\vx}$, the virtual target is computed using the same task instantiated from the full context $\mathcal{C}$. The perturbation transformation $T_r$ associated with each operator $r$ is
defined as follows.
% \begin{equation}
%     \widetilde{y}_{m}
%     =
%     \widehat{y}_m
%     \left(
%         \widetilde{\mathbf{x}};\mathcal{C}
%     \right)
%     =
%     g_m
%     \left(
%         \widetilde{\mathbf{x}};
%         \widehat{\psi}_{m,\mathcal{C}}
%     \right).
%     \label{eq:pseudo_ood_virtual_target}
% \end{equation}
% Thus, the pseudo-OOD target is neither inherited from an original held-out
% sample nor copied from any donor sample used in the perturbation. Its
% predictive compatibility is evaluated under the same held-out virtual task as
% \begin{equation}
%     \widehat{\kappa}_{m,s}
%     \left(
%         \widetilde{\mathbf{x}}
%     \right)
%     =
%     q_{\theta^\star}
%     \left(
%         \widetilde{y}_{m}
%         \mid
%         \widetilde{\mathbf{x}},
%         \mathcal{D}^{m\mid\mathcal{C}}_{\mathcal{C}_{-s}}
%     \right).
%     \label{eq:pseudo_ood_compatibility}
% \end{equation}
% Only perturbed samples receiving a valid virtual class under the instantiated
% task are retained for reliability evaluation.

% \paragraph{Column shuffling.}
Column shuffling constructs synthetic rows directly from $\mathcal{C}_{-s}$. For each generated row $i$ and attribute $j$, a donor row index
$\pi_{ij}$ is independently sampled from the rows of $\mathcal{C}_{-s}$ and
\begin{equation}
    \widetilde{x}_{ij}^{\mathrm{shuffle}}
    =
    x_{\pi_{ij},j}.
    \label{eq:pseudo_ood_column_shuffle}
\end{equation}
Because the donor row is independently sampled for each attribute, the
transformation approximately preserves the empirical marginal distribution of
each attribute while disrupting cross-attribute dependencies.
% \paragraph{Subset replacement.}
% Subset replacement begins with a sample drawn from the held-out subset and replaces randomly a subset of its attributes using values independently drawn from $\mathcal{C}_{-s}$. For the sampled sample $\mathbf{x}_h$, and radomly selected attribute subset $\mathcal{A}$
% \begin{equation}
%     \widetilde{x}_{h}^{\mathrm{replace}}
%     =
%     \begin{cases}
%         x_{\nu_j,j},
%         & j\in\mathcal{A},\\
%         x_{h,j},
%         & j\notin\mathcal{A},
%     \end{cases}
%     \label{eq:pseudo_ood_subset_replacement}
% \end{equation}
% \dy{where $\nu_j$ denotes an independently sampled donor row from $\mathcal{C}_{-s}$.} for each $\mathbf{x}_h$, the ratio for selecting as $\mathcal{A}$ is $0.3$ of total attributes and each attribute is replaced independently with probability $0.3$,

Subset replacement starts from a sample $\vx_h$ drawn from the nominal
held-out subset and randomly replaces a subset of its attributes with donor
values from $\mathcal{C}_{-s}$.
Specifically, let $\mathcal{A}$ denote the subset of attributes selected
independently with replacement probability $p_{\mathrm{rep}}=0.3$. For each
$j\in\mathcal{A}$, a donor row index $\nu_j$ is independently sampled from the rows of
$\mathcal{C}_{-s}$, and the perturbed sample is defined attribute-wise as
\begin{equation}
    \widetilde{x}_{h,j}^{\mathrm{replace}}
    =
    \begin{cases}
        x_{\nu_j,j},
        & j\in\mathcal{A},\\
        x_{h,j},
        & j\notin\mathcal{A}.
    \end{cases}
    \label{eq:pseudo_ood_subset_replacement}
\end{equation}
Thus, approximately $30\%$ of the attributes are replaced on average, with
donor values sampled independently across the selected attributes.

% Under the fixed-variation configuration, each attribute is independently
% replaced with probability $0.30$. Under the full-variation configuration, a
% sample-specific replacement rate is first drawn uniformly from
% $[0.10,1.00]$, and the corresponding number of attributes is then replaced.
% \paragraph{Scaled jitter.}
% Scaled jitter perturbs numerical attributes of a held-out sample $\mathbf{x}_h$ drawn from the nominal
% held-out subset using feature-wise noise scaled by the dispersion estimated from $\mathcal{C}_{-s}$:
% \begin{equation}
%     \widetilde{x}_{j}^{\mathrm{jitter}}
%     =
%     x_j
%     +
%     \lambda_j
%     \widehat{\sigma}_{j,\mathcal{C}_{-s}}
%     \epsilon_j,
%     \qquad
%     \epsilon_j\sim\mathcal{N}(0,1),
%     \label{eq:pseudo_ood_scaled_jitter}
% \end{equation}
% where $\widehat{\sigma}_{j,\mathcal{C}_{-s}}$ is the interquartile range of attribute $j$ estimated from $\mathcal{C}_{-s}$ and we set weight $\lambda_j=0.5$. If the interquartile range is degenerate, the empirical standard deviation is used as a fallback. Categorical attributes are left unchanged by this operator.
Scaled jitter perturbs the numerical attributes of a sample
$\vx_h$ drawn from the nominal held-out subset by adding feature-wise
Gaussian noise scaled according to the dispersion estimated from the
complementary train-side context $\mathcal{C}_{-s}$. Specifically, for each
numerical attribute $j$,
\begin{equation}
    \widetilde{x}_{h,j}^{\mathrm{jitter}}
    =
    x_{h,j}
    +
    \lambda_j
    \widehat{d}_{j,\mathcal{C}_{-s}}
    \epsilon_j,
    \qquad
    \epsilon_j\sim\mathcal{N}(0,1),
    \label{eq:pseudo_ood_scaled_jitter}
\end{equation}
where $\widehat{d}_{j,\mathcal{C}_{-s}}$ denotes the interquartile range
of attribute $j$ estimated from $\mathcal{C}_{-s}$, and we set $\lambda_j=0.5$. 
%If the interquartile range is degenerate, we use the empirical standard deviation as a fallback.
Scaled jitter is omitted when no numerical attribute is available.

\paragraph{Distributional validity of the context-calibrated structural-violation.}
The preceding perturbations are used to assess the selective task coverage defined by C2 in the main paper. Rather than attempting to approximate the unknown
test-time anomaly distribution, they construct controlled alternatives to the
normal distribution by disrupting complementary forms of statistical
regularity.

To formalize their distributional effects, consider the population
idealization in which the train-side context and held-out samples are independently drawn from $P_0$. 
% Let $\rvx=(X_1,\ldots,X_d)\sim P_0$, let $P_{0,j}$ denote the marginal
% distribution of $X_j$, 
Let
$\rvx=(\ervx_1,\ldots,\ervx_d)\sim P_0$, let $P_{0,j}$ denote the marginal
distribution of $\ervx_j$,
and let $P_{0,\mathcal{A}}$ denote the joint marginal distribution over an attribute subset $\mathcal{A}\subseteq[d]$. We define the total correlation of $\rvx_{\mathcal{A}}$ as
\begin{equation}
    \operatorname{TC}
    \left(
        \rvx_{\mathcal{A}}
    \right)
    :=
    D_{\mathrm{KL}}
    \left(
        P_{0,\mathcal{A}}
        \,\middle\|\,
        \prod_{j\in\mathcal{A}}P_{0,j}
    \right).
    \label{eq:total_correlation}
\end{equation}

For split $s$, let $\widehat{P}_{j,s}$ denote the empirical marginal
distribution of attribute $j$ in $\mathcal{C}_{-s}$. Under the i.i.d.\
assumption, $\widehat{P}_{j,s}$ converges to $P_{0,j}$ as the train-side
context grows. 
% The following proposition describes both the
% context-conditional probe distributions and their population targets.
% The following proposition characterizes their finite-context probe distributions and the corresponding distributions in the population limit.

The following proposition characterizes the distributions induced by these probe operators and their corresponding structural effects.
\begin{proposition}[Distributional effects of the context-calibrated probe operators]
\label{prop:pseudo_ood_structural_effect}
Assume that the held-out base sample is independent of
$\mathcal{C}_{-s}$ and that all donor indices are sampled independently.
Whenever the displayed KL divergences are well defined, the probe operators
have the following effects.

\begin{enumerate}
    \item \textbf{Column shuffling.}
    Conditional on $\mathcal{C}_{-s}$, column shuffling induces
    \begin{equation}
        \widehat{Q}^{\mathrm{shuffle}}_s
        =
        \prod_{j=1}^{d}
        \widehat{P}_{j,s}.
        \label{eq:shuffle_empirical_distribution}
    \end{equation}
    In the population limit, this becomes
    \begin{equation}
        Q^{\mathrm{shuffle}}
        =
        \prod_{j=1}^{d}P_{0,j},
        \label{eq:shuffle_population_distribution}
    \end{equation}
    for which
    \begin{equation}
        D_{\mathrm{KL}}
        \left(
            P_0
            \,\middle\|\,
            Q^{\mathrm{shuffle}}
        \right)
        =
        \operatorname{TC}(\rvx).
        \label{eq:shuffle_total_correlation}
    \end{equation}

    \item \textbf{Subset replacement.}
    Conditional on a realized replacement subset
    $\mathcal{A}\subseteq[d]$, subset replacement induces
    \begin{equation}
        \widehat{Q}^{\mathrm{replace}}_{\mathcal{A},s}
        =
        P_{0,\mathcal{A}^{c}}
        \prod_{j\in\mathcal{A}}
        \widehat{P}_{j,s}.
        \label{eq:replacement_empirical_distribution}
    \end{equation}
    In the population limit, this becomes
    \begin{equation}
        Q^{\mathrm{replace}}_{\mathcal{A}}
        =
        P_{0,\mathcal{A}^{c}}
        \prod_{j\in\mathcal{A}}P_{0,j},
        \label{eq:replacement_population_distribution}
    \end{equation}
    and
    \begin{equation}
    \begin{aligned}
        D_{\mathrm{KL}}
        \left(
            P_0
            \,\middle\|\,
            Q^{\mathrm{replace}}_{\mathcal{A}}
        \right)
        &=
        I_{P_0}
        \left(
            \rvx_{\mathcal{A}};
            \rvx_{\mathcal{A}^{c}}
        \right)
        \\
        &\quad+
        \operatorname{TC}
        \left(
            \rvx_{\mathcal{A}}
        \right).
    \end{aligned}
    \label{eq:replacement_dependency_divergence}
    \end{equation}
    When $\mathcal{A}=\varnothing$, the operator leaves the distribution
    unchanged and both terms in
    Equation~\ref{eq:replacement_dependency_divergence} are zero.

    \item \textbf{Scaled jitter.}
    Let $\mathcal{J}_{\mathrm{num}}$ denote the set of numerical attributes
    and $\mathcal{J}_{\mathrm{num}}^{c}$ its complement. Conditional on
    $\mathcal{C}_{-s}$ and the feature-wise jitter scales, scaled jitter
    satisfies
    \begin{equation}
        \widetilde{\rvx}^{\mathrm{jitter}}_{\mathcal{J}_{\mathrm{num}}}
        =
        \rvx_{\mathcal{J}_{\mathrm{num}}}
        +
        \boldsymbol{\eta}_s,
        \qquad
        \widetilde{\rvx}^{\mathrm{jitter}}_{\mathcal{J}_{\mathrm{num}}^{c}}
        =
        \rvx_{\mathcal{J}_{\mathrm{num}}^{c}},
        \label{eq:jitter_population_distribution}
    \end{equation}
    where
    \begin{equation}
    \begin{gathered}
        \boldsymbol{\eta}_s
        \mid
        \mathcal{C}_{-s}
        \sim
        \mathcal{N}
        \left(
            \mathbf{0},
            \widehat{\Sigma}_{\eta,s}
        \right),
        \qquad
    \boldsymbol{\eta}_s
    \perp\!\!\!\perp
    \rvx
    \mid
    \mathcal{C}_{-s},
        \\
        %\qquad
        \widehat{\Sigma}_{\eta,s}
        :=
        \operatorname{diag}
        \left(
            \lambda_j^2
            \widehat{d}_{j,\mathcal{C}_{-s}}^2
        \right)_{j\in\mathcal{J}_{\mathrm{num}}}.    
    \end{gathered}
    \label{eq:jitter_noise_distribution}
    \end{equation}
%     where
% \begin{equation}
%     \boldsymbol{\eta}_s
%     \mid
%     \mathcal{C}_{-s}
%     \sim
%     \mathcal{N}
%     \left(
%         \mathbf{0},
%         \widehat{\Sigma}_{\eta,s}
%     \right),
%     \qquad
%     \boldsymbol{\eta}_s
%     \perp\!\!\!\perp
%     \mathbf{X}
%     \mid
%     \mathcal{C}_{-s},
% \end{equation}
    Consequently, whenever the corresponding second moments exist,
    \begin{equation}
    \begin{aligned}
        \E
        \left[
            \widetilde{\rvx}^{\mathrm{jitter}}_{\mathcal{J}_{\mathrm{num}}}
            \mid
            \mathcal{C}_{-s}
        \right]
        &=
        \E
        \left[
            \rvx_{\mathcal{J}_{\mathrm{num}}}
        \right],
        \\
        \Cov
        \left(
            \widetilde{\rvx}^{\mathrm{jitter}}_{\mathcal{J}_{\mathrm{num}}}
            \mid
            \mathcal{C}_{-s}
        \right)
        &=
        \Cov
        \left(
            \rvx_{\mathcal{J}_{\mathrm{num}}}
        \right)
        +
        \widehat{\Sigma}_{\eta,s}.
    \end{aligned}
    \label{eq:jitter_covariance_shift}
    \end{equation}
    Hence, whenever $\widehat{\Sigma}_{\eta,s}\neq\mathbf{0}$, scaled
    jitter induces a nontrivial distributional shift in the numerical
    subvector.
\end{enumerate}
\end{proposition}

\begin{proof}
For column shuffling, each attribute is sampled independently from its
train-side empirical marginal, yielding
Equation~\ref{eq:shuffle_empirical_distribution}. Replacing these empirical
marginals by their population counterparts gives
Equation~\ref{eq:shuffle_population_distribution}, and the KL identity follows
directly from the definition of total correlation.

For subset replacement, the retained attributes preserve their nominal joint
distribution, whereas the replaced attributes are sampled independently from
their respective marginals. Hence,
\begin{equation}
    Q^{\mathrm{replace}}_{\mathcal A}
    =
    P_{0,\mathcal A^c}
    \prod_{j\in\mathcal A}P_{0,j}.
\end{equation}
Assuming the corresponding densities exist, its log-density ratio with
respect to $P_0$ decomposes as
\begin{align}
    \log
    \frac{
        p_0(\rvx)
    }{
        p_{0,\mathcal A^c}(\rvx_{\mathcal A^c})
        \prod_{j\in\mathcal A}p_{0,j}(\ervx_j)
    }
    &=
    \log
    \frac{
        p_0(\rvx)
    }{
        p_{0,\mathcal A}(\rvx_{\mathcal A})
        p_{0,\mathcal A^c}(\rvx_{\mathcal A^c})
    }
    \nonumber\\
    &\quad+
    \log
    \frac{
        p_{0,\mathcal A}(\rvx_{\mathcal A})
    }{
        \prod_{j\in\mathcal A}p_{0,j}(\ervx_j)
    }.
\end{align}
Taking expectations with respect to $\rvx\sim P_0$ gives
\begin{align}
    D_{\mathrm{KL}}
    \left(
        P_0
        \,\middle\|\,
        Q^{\mathrm{replace}}_{\mathcal A}
    \right)
    &=
    D_{\mathrm{KL}}
    \left(
        P_0
        \,\middle\|\,
        P_{0,\mathcal A}P_{0,\mathcal A^c}
    \right)
    \nonumber\\
    &\quad+
    D_{\mathrm{KL}}
    \left(
        P_{0,\mathcal A}
        \,\middle\|\,
        \prod_{j\in\mathcal A}P_{0,j}
    \right)
    \\
    &=
    I_{P_0}
    \left(
        \rvx_{\mathcal A};
        \rvx_{\mathcal A^c}
    \right)
    +
    \operatorname{TC}
    \left(
        \rvx_{\mathcal A}
    \right),
\end{align}

where the first term measures the dependence between the replaced and retained attributes, while the second measures the dependence among the replaced
attributes themselves.

Finally, scaled jitter adds conditionally independent zero-mean noise to the
numerical subvector. Therefore, the perturbation preserves its conditional
mean, while independence eliminates the cross-covariance terms, yielding
\begin{equation}
    % \Cov
    % \left(
    %     \rvx_{\mathcal J_{\mathrm{num}}}
    %     +
    %     \boldsymbol{\eta}_s
    %     \mid
    %     \mathcal C_{-s}
    % \right)
    % =
    % \Cov
    % \left(
    %     \rvx_{\mathcal J_{\mathrm{num}}}
    % \right)
    % +
    % \widehat{\Sigma}_{\eta,s}.
    \Cov
\left(
    \rvx_{\mathcal J_{\mathrm{num}}}
    +
    \boldsymbol{\eta}_s
    \mid
    \mathcal C_{-s}
\right)
=
\Cov
\left(
    \rvx_{\mathcal J_{\mathrm{num}}}
    \mid
    \mathcal C_{-s}
\right)
+
\widehat{\Sigma}_{\eta,s}
=
\Cov
\left(
    \rvx_{\mathcal J_{\mathrm{num}}}
\right)
+
\widehat{\Sigma}_{\eta,s}.
\end{equation}
Thus, jitter preserves the nominal mean while increasing dispersion along the
perturbed numerical attributes.
\end{proof}

Proposition~\ref{prop:pseudo_ood_structural_effect} establishes the
operator-level distributional role of the context-calibrated structural violations. 
% Column shuffling removes global cross-attribute dependence, subset replacement disrupts both within-subset dependence and dependence between replaced and retained attributes, and scaled jitter perturbs the scale and local support of the numerical distribution.
Column shuffling removes global cross-attribute dependence, subset replacement disrupts both within-subset dependence and dependence between replaced and retained attributes, and scaled jitter increases dispersion along the perturbed numerical attributes.
The probes therefore form complementary structural stress tests derived from the normality-aware context rather than arbitrary synthetic samples.

Nominal held-out samples assess whether the induced task maintains high
predictive compatibility on normal data (C1: nominal predictive coherence),
whereas structural-violation samples drawn from controlled shifted
distributions serve as surrogates for assessing whether a candidate task
assigns lower predictive compatibility to structural departures than to
nominal held-out samples (C2: selective task coverage). Together, these provide
a principled held-out criterion for selecting candidate tasks with suitable
predictive coverage for tabular anomaly detection.

\subsection{Details of Data-Adaptive Selection}
\label{appendix:task_filtering}
% \subsubsection{Intra-task Selection}
% For each active virtual-task family $m$, we construct a family-specific candidate set $\mathcal{C}_m$ by varying the task-construction parameters of the corresponding labeling rule. These parameters control how the context-supported nominal structure is converted into discrete virtual targets.
% %, and therefore determine the particular normality-anchored predictive relation induced by each candidate. 
% The applicable configuration axes are adapted to the corresponding profile-specific task instantiation described in Appendix~\ref{appendix:feature_profile_inference} and Appendix~\ref{appendix:virtual_task_instantiations}.
% Table~\ref{tab:task_candidate_grid} summarizes the principal configuration
% axes used to construct the candidate sets.
% All candidates associated with the same dataset and random seed are evaluated using the common held-out registry described in
% Appendix~\ref{appendix:pseudo_ood_generation}. 
\begin{table*}[t]
    \centering
    \caption{
        Principal configuration axes used to construct candidate configurations
        for intra-task selection.
        Each axis is applied when the corresponding virtual-task instantiation
        is available for the inferred feature profile.
        These configurations affect only virtual-task construction; the original
        observed attributes are preserved for general-purpose TFM's inference.
    }
    \label{tab:task_candidate_grid}
    \small
    \begin{tabular}{p{6.5cm} p{4.5cm} p{2.0cm}} % p{6.0cm}}
        \toprule
        \textbf{Configuration Axis}
        & \textbf{Applicable Feature Profiles}
        & \textbf{Candidates} \\
        % & \textbf{Effect on the induced virtual task} \\
        \midrule
        % &
        % Controls the granularity with which a context-derived scalar
        % statistic is partitioned into virtual classes.
        Number of clusters 
        % (target generation)
        &
        All
        %K$-means, $K$-modes, and $K$-prototypes
        &
        $\{2,3\}$
        % &
        % Controls the granularity of the prototype-based decomposition of the
        % nominal context and hence the induced virtual class structure.
        \\
        Numerical feature transformation
        &
        Numerical-only, all mixed variants
        &
        $\{\texttt{raw},\texttt{robust}\}$
        % &
        % Determines whether numerical attributes are used on their original
        % scale or after context-derived robust normalization when constructing
        % the virtual labeling rule.
        \\
        Number of quantile bins 
        &
        All
        &
        $\{2,3\}$
        \\

        Number of quantile bins for numerical tokenization
        &
        All mixed variants
        &
        $\{2,3\}$
        % &
        % Controls the discretization granularity used to convert numerical
        % attributes into the mixed categorical representation employed by the
        % corresponding labeling rule.
        \\
        \bottomrule
    \end{tabular}
\end{table*}
% \subsubsection{Intra-task Selection}
% For each instantiated task $m$ corresponding to a specific template, we construct a task-specific candidate configuration set $\mathcal{C}_m$ that determine context-induced parameters for the virtual labeling rule.
% %by varying the task construction parameters. 
% These parameters control how the context-induced
% nominal predictive structure is converted into discrete virtual targets. The applicable
% configuration axes are adapted to the corresponding profile-specific task
% instantiation described in
% Appendix~\ref{appendix:feature_profile_inference} and
% Appendix~\ref{appendix:virtual_task_instantiations}.
% Table~\ref{tab:task_candidate_grid} summarizes the principal configuration
% axes used to construct the candidate sets. All candidates associated with the
% same dataset and random seed are evaluated using the common nominal held-out registry
% and structural-violation probes described in
% Appendix~\ref{appendix:pseudo_ood_generation}.
\subsubsection{Intra-task Selection}
For each instantiated task $m$ corresponding to a specific template, we construct a task-specific candidate configuration set $\mathcal{C}_m$. Each candidate $c\in\mathcal{C}_m$ determines a context-induced task parameter $\widehat{\psi}_{m,\mathcal{C}}^{\,c}$ and the corresponding virtual-labeling rule.
These configurations control how the context-induced nominal predictive structure is converted into discrete virtual targets. The applicable configuration axes are adapted to the corresponding profile-specific task instantiation described in 
%Appendix~\ref{appendix:feature_profile_inference} and 
Appendix~\ref{appendix:virtual_task_instantiations}.
Table~\ref{tab:task_candidate_grid} summarizes the principal configuration axes used to construct the candidate sets. All candidates associated with the same dataset and random seed are evaluated using the common nominal held-out registry and structural-violation probes described in Appendix~\ref{appendix:pseudo_ood_generation}.

As defined in Section~\ref{sec:DataAdaptiveTaskFilter}, each candidate
$c\in\mathcal{C}_m$ is characterized by
\begin{equation}
    \vr_{m,c}
    =
    \left(
        \operatorname{Med}(\mathcal{K}^{\mathrm{nom}}_{m,c}),
        \Var(\mathcal{K}^{\mathrm{nom}}_{m,c}),
        \rho_{m,c},
        \Delta_{m,c}
    \right).
\end{equation}
Intra-task selection implements
$\operatorname{Select}_{\mathrm{intra}}(\cdot)$ hierarchically, so that
nominal predictive coherence is first ensured before comparing the selective task coverage of candidate configurations.

\paragraph{Nominal-coherence filtering.}
We first retain candidates whose median held-out nominal compatibility meets
the nominal-coherence threshold $\tau_{\mathrm{N}}$:
\begin{equation}
    \mathcal{C}^{\mathrm{N}}_m
    =
    \left\{
        c\in\mathcal{C}_m
        :
        \operatorname{Med}
        (\mathcal{K}^{\mathrm{nom}}_{m,c})
        \geq
        \tau_{\mathrm{N}}
    \right\}.
    \label{eq:intra_nominal_filter}
\end{equation}
This step prevents a configuration with poor nominal predictive coherence from
being selected solely because it strongly separates the structural-violation
probes. 
If $\mathcal{C}^{\mathrm{N}}_m$ is empty, we directly select the candidate
with the largest separation AUC as the representative configuration
$c_m^\star$ and skip the remaining intra-task selection steps.
% If $\mathcal{C}^{\mathrm{N}}_m$ is empty, we retain as a fallback the
% candidate with the largest separation AUC.
% ,
% \begin{equation}
%     \overline{\mathcal{C}}^{\mathrm{N}}_m
%     =
%     \begin{cases}
%         \mathcal{C}^{\mathrm{N}}_m,
%         &
%         \mathcal{C}^{\mathrm{N}}_m\neq\varnothing,
%         \\[1mm]
%         \left\{
%             \displaystyle
%             \arg\max_{c\in\mathcal{C}_m}
%             \rho_{m,c}
%         \right\},
%         &
%         \mathcal{C}^{\mathrm{N}}_m=\varnothing.
%     \end{cases}
%     \label{eq:intra_nominal_rescue}
% \end{equation}

\paragraph{Task coverage-based shortlisting.}
Among the candidates in
$\mathcal{C}^{\mathrm{N}}_m$, we next retain those with the
largest separation AUC $\rho_{m,c}$. Let
\begin{equation}
    L_m
    =
    \min
    \left\{
        \left|
            \mathcal{C}^{\mathrm{N}}_m
        \right|,
        \;
        \max
        \left\{
            L_{\min},
            \left\lceil
                \gamma_{\mathrm{AUC}}
                \left|
                    \mathcal{C}^{\mathrm{N}}_m
                \right|
            \right\rceil
        \right\}
    \right\},
    \label{eq:intra_shortlist_size}
\end{equation}
where $\gamma_{\mathrm{AUC}}$ denotes the AUC shortlist fraction and
$L_{\min}$ is the minimum shortlist size whenever sufficiently many
candidates are available. We then define
$\mathcal{C}^{\mathrm{short}}_m$ as the $L_m$ candidates with the largest
$\rho_{m,c}$.

\paragraph{Representative configuration.}
The final representative is selected by jointly considering nominal
compatibility, its stability, and conservative separation from the
structural-violation probes. For candidates in
$\mathcal{C}^{\mathrm{short}}_m$, let
$R_{\downarrow}(\cdot)$ and $R_{\uparrow}(\cdot)$ denote normalized badness
ranks for quantities for which larger and smaller values are preferred,
respectively. Ties are assigned their average rank. We define
\begin{equation}
\begin{aligned}
    B_{\mathrm{intra}}(m,c)
    &=
    R_{\downarrow}
    \left(
        \operatorname{Med}
        (\mathcal{K}^{\mathrm{nom}}_{m,c})
    \right)
    +
    R_{\uparrow}
    \left(
        \Var
        (\mathcal{K}^{\mathrm{nom}}_{m,c})
    \right)
    +
    R_{\downarrow}
    \left(
        \Delta_{m,c}
    \right),
    \\
    c_m^\star
    &=
    \arg\min_{
        c\in\mathcal{C}^{\mathrm{short}}_m
    }
    B_{\mathrm{intra}}(m,c).
\end{aligned}
\label{eq:intra_selection_detail}
\end{equation}
Importantly, the ranks in
Equation~\ref{eq:intra_selection_detail} are recomputed within the shortlisted
candidate set. Hence, $\rho_{m,c}$ acts as a separation-based qualification
criterion, while the final choice balances high nominal compatibility, low
nominal compatibility variance, and a large conservative quantile gap.

For the selection hyperparameters, we set
$\tau_{\mathrm{N}}=0.7$ and $L_{\min}=2$, while
$\gamma_{\mathrm{AUC}}$ is selected through the backbone-specific configuration calibration described in Appendix~\ref{Appendix:Configuration Search}.
These choices are fixed thereafter and applied unchanged to all unseen
evaluation datasets.

\subsubsection{Inter-task Selection}
After intra-task selection, each active virtual task $m\in\mathcal{M}$ is
represented by its selected configuration $c_m^\star$ and the corresponding
reliability statistics $\vr_{m,c_m^\star}$. We then perform
inter-task selection to remove representatives whose predictive coverage is
insufficiently reliable relative to the other retained tasks. The following procedure implements $\operatorname{Select}_{\mathrm{inter}}(\cdot)$ defined in Section~\ref{sec:DataAdaptiveTaskFilter}.

\paragraph{Task coverage-based prefiltering.}
We first require each task representative to exhibit at least chance-level
separation between held-out normal samples and structural-violation probes.
Specifically, we define the set of AUC-eligible tasks as
\begin{equation}
    \mathcal{M}_{\mathrm{AUC}}
    =
    \left\{
        m\in\mathcal{M}
        :
        \rho_{m,c_m^\star}
        \geq
        \tau_{\mathrm{A}}
    \right\},
    \label{eq:inter_auc_prefilter}
\end{equation}
where we set $\tau_{\mathrm{A}}=0.50$.
% If $\mathcal{M}_{\mathrm{AUC}}$ is empty, we retain as a fallback the task
% representative with the largest $\rho_{m,c_m^\star}$, ensuring that at least
% one task remains available for anomaly scoring.
% If $\mathcal{M}_{\mathrm{AUC}}$ is empty, we replace it with the singleton set containing the task representative with the largest $\rho_{m,c_m^\star}$, ensuring that at least one task remains available for scoring.
If $\mathcal{M}_{\mathrm{AUC}}$ is empty, we replace it with the singleton set containing the task with the largest $\rho_{m,c_m^\star}$, ensuring that at least one task remains available for scoring.

\paragraph{Inter-task ranking.}
For the AUC-eligible tasks, we jointly compare nominal predictive coherence,
its stability, and selective task coverage from the structural-violation probes.
Let $R_{\downarrow}(\cdot)$ and $R_{\uparrow}(\cdot)$ denote normalized
badness ranks for quantities for which larger and smaller values are
preferred, respectively. The ranks are recomputed over
$\mathcal{M}_{\mathrm{AUC}}$, with ties assigned their average rank. We define
\begin{equation}
\begin{aligned}
    B_{\mathrm{inter}}(m)
    &=
    w_A
    R_{\downarrow}
    \left(
        \rho_{m,c_m^\star}
    \right)
    +
    R_{\downarrow}
    \left(
        \operatorname{Med}
        (\mathcal{K}^{\mathrm{nom}}_{m,c_m^\star})
    \right)
    \\
    &\quad+
    R_{\uparrow}
    \left(
        \Var
        (\mathcal{K}^{\mathrm{nom}}_{m,c_m^\star})
    \right)
    +
    R_{\downarrow}
    \left(
        \Delta_{m,c_m^\star}
    \right),
    \qquad
    m\in\mathcal{M}_{\mathrm{AUC}},
\end{aligned}
\label{eq:inter_task_badness}
\end{equation}
% where $w_A$ controls the relative contribution of the selective task coverage.
where $w_A$ controls the relative contribution of the AUC-based separation
criterion.
A smaller $B_{\mathrm{inter}}(m)$ indicates a more reliable task
representative.

Based on $B_{\mathrm{inter}}(m)$, the final task set is obtained by applying
an inter-task retention policy $\pi_{\mathrm{ret}}$ to the AUC-eligible task
set:
\begin{equation}
    \mathcal{M}^{\star}
    =
    {\pi_{\mathrm{ret}}}
    \left(
        \mathcal{M}_{\mathrm{AUC}};
        \left\{
            B_{\mathrm{inter}}(m)\right\}_{m\in\mathcal{M}_{\mathrm{AUC}}}
            % \rho_{m,c_m^\star}_{m\in\mathcal{M}_{\mathrm{AUC}}}
    \right),
    \label{eq:inter_task_retention}
\end{equation}
where $\pi_{\mathrm{ret}}$ determines how the relative reliability of the
AUC-eligible task representatives is translated into the final retained set.
% Depending on the policy, the operator may retain all eligible tasks, select
% only the most reliable tasks according to $B_{\mathrm{inter}}$, or impose an
% additional reliability or separation criterion.

Both the AUC-rank weight $w_A$ and the retention policy $\pi_{\mathrm{ret}}$ are selected through the backbone-specific configuration calibration using the real-world development datasets described in Appendix~\ref{Appendix:Configuration Search}. Once selected, they are fixed and applied unchanged to all unseen evaluation datasets.
\subsection{Task-wise Anomaly Scoring and Score Aggregation}
\label{appendix:anomaly_scoring}

% We describe how the predictive compatibility of each retained task is converted into anomaly evidence and subsequently aggregated across the surviving tasks. We consider several task-wise scoring functions
% $\phi(\cdot)$ that quantify different forms of predictive incompatibility,
% together with consensus- and reliability-aware ensemble operators. These
% scoring and aggregation choices are treated as framework-level
% configurations and are selected using the real-world development datasets
% described in Appendix~\ref{Appendix:Configuration Search}. 
% Once selected, they are fixed and applied unchanged to all unseen evaluation datasets.
We describe how the predictive support of each retained virtual supervised task is
converted into anomaly evidence and subsequently aggregated across the
surviving tasks. We consider the specific realization of the task-wise scoring
function $\varphi$ together with several ensemble operators. The ensemble strategy is selected through backbone-specific configuration calibration using the real-world development datasets described in
Appendix~\ref{Appendix:Configuration Search}.
% The choice of $\operatorname{Ensemble}$ are treated as framework-level configurations and selected using the real-world development datasets described in Appendix~\ref{Appendix:Configuration Search}.}

\paragraph{Task-wise anomaly scoring.}
% We consider three principal forms of anomaly evidence for $\varphi$. Let
% $q_{m^\star}(\vx)
%     =
%     \widehat{\kappa}_{m^\star}(\mathcal C,\vx)
% $
% denote the predictive support assigned by the retained virtual task
% $m^\star$.
We consider \textit{Held-out Calibrated Surprisal} as the principal form of anomaly evidence for $\varphi$. 
For each retained task $m\in\mathcal{M}^{\star}$, let $\widehat{\kappa}_{m}(\vx;\mathcal{C})$
denote the predictive support assigned by the pretrained TFM, and let
$\mathcal{K}^{\mathrm{nom}}_{m}
:=
\mathcal{K}^{\mathrm{nom}}_{m,c_m^\star}$
denote the held-out nominal support collection associated with its selected configuration $c_m^\star$.
% Let
% $q_{m^\star}(\vx)=
% $\widehat{\kappa}_m(\vx;\mathcal C)$
% denote the predictive support assigned by the pretrained PFN-based TFM for the retained virtual task $m\in\mathcal{M}^\star$ and $\mathcal{K}^{\mathrm{nom}}_{m}:=\mathcal{K}^{\mathrm{nom}}_{m,c_m^\star}$ denote the support collection for held-out nominal samples associated with the retained task $m$ with configuration $c^{\star}_{m}$. 
%The following procedure implements $\varphi$.

We calibrate the query support against the held-out nominal support distribution $\mathcal{K}^{\mathrm{nom}}_{m}$ using
\begin{equation}
    p_{m}(\vx)
    =
    \frac{
        1+
        % \sum_{q\in\mathcal Q_{m^\star}^{N}}
        \sum_{q\in\mathcal{K}^{\mathrm{nom}}_{m}}
        \1
        \left[
            q\le \widehat{\kappa}_{m}(\vx;\mathcal{C})
            %q_{m}(\vx)
        \right]
    }{
        |\mathcal{K}^{\mathrm{nom}}_{m}|+1
    },
\end{equation}
and define
\begin{equation}
    s_{m}(\vx)
    =
    -\log\!\left(p_{m}(\vx)\right).
\end{equation}

\paragraph{Task-score aggregation (ensemble).}
% Let $s_{m}(\mathbf{x})$ denote the
% task-wise anomaly evidence obtained from the selected scoring rule for each surviving task $m^\star\in\mathcal M^\star$. 
% We consider two consensus-oriented operators and two forms of reliability-aware adaptive aggregation.
% The first $\operatorname{Ensemble}$ operator is naively combining the task-wise scores $s_{m}(\vx)$ directly by using min or low2mean aggregation where low2mean is:
We first consider direct aggregation of the task-wise scores $s_m(\vx)$ using either min or low2mean aggregation. The latter is defined as
\begin{equation}
    % S_{\min}(\vx)
    % =
    % \min_{m\in\mathcal M^\star}
    % s_{m}(\vx),\,
    S_{\mathrm{low2}}(\vx)
    =
    \frac{1}{|\mathcal L_2(\vx)|}
    \sum_{m^\star\in\mathcal L_2(\vx)}
    s_{m^\star}(\vx),
\end{equation}
% where $\mathcal L_2(\mathbf{x})$ contains the two surviving tasks with the
% smallest anomaly scores. 
where $\mathcal L_2(\vx)$ contains the $\min(2,|\mathcal M^\star|)$ surviving tasks with the smallest anomaly scores.
These operators provide conservative consensus criteria, with $S_{\mathrm{low2}}$ offering a relaxed alternative to min aggregation.

For adaptive aggregation, we compute task-suitability scores using only the surviving tasks. Let $K=|\mathcal M^\star|$ and define
% For adaptive aggregation,we compute the task suitability score using only the surviving
% tasks. Let $K=|\mathcal M^\star|$ and define
\begin{equation}
    B_{\mathrm{agg}}(m)
    =
    \frac{
        w_A r_A(m)
        +
        r_M(m)
        +
        r_V(m)
        +
        r_G(m)
    }{K},
    \label{eq:config_aggregation_badness}
\end{equation}
where $r_A$, $r_M$, $r_V$, and $r_G$ are the ordinal badness ranks of separation AUC, nominal median support, nominal support variance, and quantile gap, respectively, recomputed over the surviving tasks using the
same preference directions as in inter-task selection.
% where the ranks are ordinal ranks computed over the surviving tasks, and a
% smaller $B_{\mathrm{agg}}(m^\star)$ indicates greater reliability.
% where $r_A$, $r_M$, $r_V$, and $r_G$ are the ordinal badness ranks of
% separation AUC, nominal median compatibility, nominal compatibility variance,
% and quantile gap, respectively.
%recomputed over the surviving tasks using the
% same preference directions as in inter-task selection. 
A smaller $B_{\mathrm{agg}}(m)$ indicates greater task suitability for anomaly detection.
We consider the following adaptive strategies.

\begin{itemize}[left=0pt]
    \item \textbf{Reliability-spread adaptive aggregation.}
    We compute
    \begin{equation}
        \Delta_{\mathrm{rel}}
        =
        \max_{m\in\mathcal M^\star}
        B_{\mathrm{agg}}(m)
        -
        \min_{m\in\mathcal M^\star}
        B_{\mathrm{agg}}(m).
    \end{equation}
    When $K\geq3$ and $\Delta_{\mathrm{rel}}<1.0$, the task reliabilities are
    regarded as sufficiently similar, and a consensus-oriented operator
    (min or low2mean) is used.
    Otherwise, aggregation favors the most suitable task(s) according to $B_{\mathrm{agg}}$.
    % Otherwise, aggregation switches to the most suitable tasks.
    %or the mean of the two most reliable
    %tasks according to $B_{\mathrm{agg}}$.

    \item \textbf{Reliability-weighted aggregation.}
    Let $\bar r_{m}$ denote the normalized rank derived from
    $B_{\mathrm{agg}}(m)$. We define
    \begin{equation}
        % \omega_{m^\star}
        % =
        % \frac{(\bar r_{m^\star}+0.1)^{-1}}
        % {\sum_{j^\star\in\mathcal M^\star}
        % (\bar r_{j^\star}+0.1)^{-1}},
        % \qquad
        % S_{\mathrm{RW}}(\mathbf{x})
        % =
        % \sum_{m^\star\in\mathcal M^\star}
        % \omega_{m^\star}s_{m^\star}(\mathbf{x}).
        \omega_{m}
=
\frac{(\bar r_{m}+\epsilon_w)^{-1}}
{\sum_{j\in\mathcal M^{\star}}
(\bar r_{j}+\epsilon_w)^{-1}},
\qquad
        S_{\mathrm{RW}}(\vx)
        =
        \sum_{m\in\mathcal M^\star}
        \omega_{m}s_{m}(\vx),
        \qquad
    \end{equation}
    where $\epsilon_w=0.1$ prevents the weight assigned to the highest-ranked task from becoming singular.
    % where $\epsilon_w=0.1$ prevents the weight of the most reliable task from becoming singular. In our framework, we set $\epsilon_w=0.1$.
\end{itemize}

\section{Experimental Details}
\label{Appendix:ExperimentalDetails}

\begin{table*}[t]
\centering
\caption{Statistics of the ADBench datasets.}
% The feature profile column is determined from the dataset-level feature-type profiling used in our paper.}
\label{tab:config_search_datasets}
\small
\begin{tabular}{lrrrrl}
\toprule
Dataset & \#Samples & \#Feat. & \#Anom. & \%Anom. & Category \\
\midrule
ALOI              & 49534  & 27   & 1508  & 3.04  & Image \\
annthyroid        & 7200   & 6    & 534   & 7.42  & Healthcare \\
backdoor          & 95329  & 196  & 2329  & 2.44  & Network \\
breastw           & 683    & 9    & 239   & 34.99 & Healthcare \\
campaign          & 41188  & 62   & 4640  & 11.27 & Finance \\
cardio            & 1831   & 21   & 176   & 9.61  & Healthcare \\
Cardiotocography  & 2114   & 21   & 466   & 22.04 & Healthcare \\
celeba            & 202599 & 39   & 4547  & 2.24  & Image \\
cover             & 286048 & 10   & 2747  & 0.96  & Botany \\
donors            & 619326 & 10   & 36710 & 5.93  & Sociology \\
fault             & 1941   & 27   & 673   & 34.67 & Physical \\
fraud             & 284807 & 29   & 492   & 0.17  & Finance \\
glass             & 214    & 7    & 9     & 4.21  & Forensic \\
Hepatitis         & 80     & 19   & 13    & 16.25 & Healthcare \\
http              & 567498 & 3    & 2211  & 0.39  & Web \\
InternetAds       & 1966   & 1555 & 368   & 18.72 & Image \\
Ionosphere        & 351    & 33   & 126   & 35.90 & Oryctognosy \\
landsat           & 6435   & 36   & 1333  & 20.71 & Astronautics \\
letter            & 1600   & 32   & 100   & 6.25  & Image \\
Lymphography      & 148    & 18   & 6     & 4.05  & Healthcare \\
magic.gamma       & 19020  & 10   & 6688  & 35.16 & Physical \\
mammography       & 11183  & 6    & 260   & 2.32  & Healthcare \\
mnist             & 7603   & 100  & 700   & 9.21  & Image \\
musk              & 3062   & 166  & 97    & 3.17  & Chemistry \\
optdigits         & 5216   & 64   & 150   & 2.88  & Image \\
PageBlocks        & 5393   & 10   & 510   & 9.46  & Document \\
pendigits         & 6870   & 16   & 156   & 2.27  & Image \\
Pima              & 768    & 8    & 268   & 34.90 & Healthcare \\
satellite         & 6435   & 36   & 2036  & 31.64 & Astronautics \\
satimage-2        & 5803   & 36   & 71    & 1.22  & Astronautics \\
shuttle           & 49097  & 9    & 3511  & 7.15  & Astronautics \\
skin              & 245057 & 3    & 50859 & 20.75 & Image \\
smtp              & 95156  & 3    & 30    & 0.03  & Web \\
SpamBase          & 4207   & 57   & 1679  & 39.91 & Document \\
speech            & 3686   & 400  & 61    & 1.65  & Linguistics \\
Stamps            & 340    & 9    & 31    & 9.12  & Document \\
thyroid           & 3772   & 6    & 93    & 2.47  & Healthcare \\
vertebral         & 240    & 6    & 30    & 12.50 & Biology \\
vowels            & 1456   & 12   & 50    & 3.43  & Linguistics \\
Waveform          & 3443   & 21   & 100   & 2.90  & Physics \\
WBC               & 223    & 9    & 10    & 4.48  & Healthcare \\
WDBC              & 367    & 30   & 10    & 2.72  & Healthcare \\
Wilt              & 4819   & 5    & 257   & 5.33  & Botany \\
wine              & 129    & 13   & 10    & 7.75  & Chemistry \\
WPBC              & 198    & 33   & 47    & 23.74 & Healthcare \\
yeast             & 1484   & 8    & 507   & 34.16 & Biology \\
\bottomrule
\end{tabular}
\end{table*}

\subsection{Details of \method Setup}
\label{Appendix:Configuration Search}
% \subsection{Framework-Level Configuration Search}
% \label{Appendix:Configuration Search}
\method employs hierarchical task selection through $\operatorname{Select}_{\mathrm{intra}}$ and $\operatorname{Select}_{\mathrm{inter}}$ to data-adaptively retain suitable virtual tasks and their normality-anchored configurations, followed by $\operatorname{Ensemble}$ to aggregate their task-wise anomaly scores.

These components depend on the predictive characteristics of the underlying pretrained general-purpose TFM. 
In particular, task selection relies on the predictive support assigned to held-out samples, while score aggregation accounts for the relative predictive reliability of the backbone TFM across the retained tasks. 
We therefore perform a one-time backbone-specific calibration of the corresponding framework-level hyperparameters, as described in Section~\ref{sec:experiment}. Once selected, these hyperparameters are fixed for each backbone and applied unchanged to all unseen datasets.

\paragraph{ADBench.}
For backbone-specific calibration, we use 46 real-world datasets from ADBench~\citep{NEURIPS2022_cf93972b}. Table~\ref{tab:config_search_datasets} summarizes their dataset statistics. 
% \paragraph{ADBench preprocessing.}
For each dataset, we construct the training/context split from normal samples only. Specifically, when the anomaly ratio exceeds $30\%$, we use $50\%$ of the available normal samples as the training/context set; otherwise, we use $60\%$. For datasets containing more than $100{,}000$ total samples, we apply the without-replacement subsampling procedure.

\paragraph{Searched framework-level hyperparameters.}
% All framework-level configuration searches are conducted using TabICL V2 as
% the frozen PFN backbone. 
The framework-level hyperparameters governing $\operatorname{Select}_{\mathrm{intra}}$, $\operatorname{Select}_{\mathrm{inter}}$, and $\operatorname{Ensemble}$ are calibrated to the predictive characteristics of each pretrained backbone TFM. The backbone itself is kept at the default configuration provided by its original paper and repository. The principal search dimensions are summarized below.

% For each candidate task, following the notation in the main paper, let
% $\operatorname{Med}(\mathcal{K}^{\mathrm{nom}}_{m,c})$,
% $\operatorname{Var}(\mathcal{K}^{\mathrm{nom}}_{m,c})$,
% $\rho_{m,c}$, and $\Delta_{m,c}$ respectively denote the median held-out
% nominal compatibility, its variance, the separation AUC, and the conservative
% quantile gap defined in Section~\ref{ssec:DataAdaptiveTaskFilter}.
% The configuration search varies the following components of 
For the hierarchical task-selection procedure:
\begin{itemize}[left=0cm]
    \item \textbf{Held-out fraction $\rho_H$:}
    $\{0.10,\,0.15,\,0.20\}$.

    \item \textbf{Intra-family separation AUC shortlist fraction
    $\gamma_{\mathrm{AUC}}$:}
    $\{0.50,\,0.60,\,0.70\}$.

    \item \textbf{Inter-family separation AUC-rank weight $w_A$:}
    $\{1.5,\,2.0\}$.

    \item \textbf{Inter-family retention policy $\pi_{\mathrm{ret}}$:}
    $\{\textsc{All},\,\textsc{Best},\,\textsc{Top-$k$}\}$.
    % \,\textsc{Rank-Cutoff},\,\textsc{AUC-Threshold}\}$.
\end{itemize}

For inter-family retention, 
% after selecting one representative configuration
% $c_m^\star$ for each  through the intra-family selection step, we first
% discard families whose $\rho_{m,c_m^\star}$ is below $0.50$. Let
% \begin{equation}
%     \mathcal{F}_{\mathrm{AUC}}
%     =
%     \left\{
%         m :
%         \rho_{m,c_m^\star} \geq 0.50
%     \right\}
% \end{equation}
% denote the set of AUC-eligible families. For the remaining families, we define
% the inter-family reliability criterion as
% \begin{equation}
%     B_{\mathrm{inter}}(m)
%     =
%     w_A r_A(m)
%     +
%     r_M(m)
%     +
%     r_V(m)
%     +
%     r_G(m),
%     \label{eq:config_inter_family_badness}
% \end{equation}
% where $r_M$, $r_V$, and $r_G$ denote the badness ranks of the nominal median
% compatibility, nominal compatibility variance, and quantile gap,
% respectively, and $r_A$ denotes the badness rank of
% $\rho_{m,c_m^\star}$. Lower values of $B_{\mathrm{inter}}$ indicate more
% reliable families.
% \textsc{All} retains all task families in AUC-eligible task representatives $\mathcal{M}_{\mathrm{AUC}}$, whereas
% \textsc{Best} retains only the family with the smallest
% $B_{\mathrm{inter}}$. \textsc{Top-$k$} retains the top
% $k
%     =
%     \left\lceil
%         0.5
%         \left|
%             \mathcal{M}_{\mathrm{AUC}}
%         \right|
%     \right\rceil
% $
% families according to $B_{\mathrm{inter}}$.
\textsc{All} retains all AUC-eligible task representatives in
$\mathcal{M}_{\mathrm{AUC}}$, whereas \textsc{Best} retains only the task with
the smallest $B_{\mathrm{inter}}$. \textsc{Top-$k$} retains the top
$k
    =
    \left\lceil
        0.5
        \left|
            \mathcal{M}_{\mathrm{AUC}}
        \right|
    \right\rceil
$
tasks according to $B_{\mathrm{inter}}$.
%while
% \textsc{Rank-Cutoff} retains families whose normalized composite rank is at
% most $0.5$. Finally, \textsc{AUC-Threshold} retains families satisfying
% $\rho_{m,c_m^\star}\geq0.80$.
% When none satisfies the corresponding AUC requirement, the composite-best family is retained as a fallback.

% For task ensembling, we search over
% the candidate realizations of
% $\operatorname{Ensemble}(\cdot)$ defined in
% Appendix~\ref{appendix:anomaly_scoring}. The combination of task selection and aggregation configuration that yields the best average performance on the ADBench datasets are selected as the backbone-calibrated framework-level configurations and are
% then fixed for all unseen evaluation datasets.
For task ensembling, we search over the candidate realizations of
$\operatorname{Ensemble}(\cdot)$ defined in
Appendix~\ref{appendix:anomaly_scoring}. The combination of task-selection and
aggregation configurations yielding the best average performance on the
ADBench datasets is selected for each backbone as its calibrated
framework-level configuration and then fixed for all unseen evaluation
datasets.
% \subsection{Data Preprocessing}
% \label{Appendix:Datasets}
\subsection{Details of the Evaluation Datasets}
\label{Appendix:Datasets}

\paragraph{Implementation Setting.}
All experiments were conducted on a machine equipped with NVIDIA H200 GPUs and two Intel Xeon Platinum 8462Y+ CPUs, providing 64 CPU cores (128 logical threads) in total.

\paragraph{ODDBench.}
ODDBench is a large-scale, real-world benchmark designed for tabular
outlier/anomaly detection. Carefully curated from the large-scale TabLib
repository~\citep{eggert2023tablibdataset627mtables}, the benchmark comprises
790 real-world datasets with semantically meaningful anomalies, such as fraud,
failures, defects, and attacks, across diverse application domains. These
datasets are partitioned into 690 publicly available sets and 100 private sets.

It further provides standardized train/test splits under a one-class
setting, where the training set contains only normal samples and the test set
contains the remaining normal samples together with all anomalous samples. Its substantially larger scale and broader diversity than conventional anomaly
detection benchmarks provide a more comprehensive and statistically robust
testbed for evaluating unsupervised tabular anomaly detection baselines. Further details are provided in the original ODDBench paper~\citep{Ding_2026}.

\paragraph{Preprocessing.}
Following prior tabular anomaly-detection studies~\citep{shen2025fomo0dfoundationmodelzeroshot, ICLR2024_6dfd16ff}, we apply feature-wise $z$-score normalization to the input data used by conventional baselines. For TAD-specialized TFM baselines, we follow the preprocessing pipeline required by each released model and its corresponding input architecture. Likewise, TabPFN-Extension and \method use the native input preprocessing of their underlying TFM backbones.

\paragraph{Subsampling of large datasets.}
To ensure a consistent computational budget across datasets, we limit the total number of samples in each dataset to at most $100{,}000$. Specifically, if
\begin{equation}
n_{\mathrm{train}} + n_{\mathrm{test}} > 100{,}000,
\end{equation}
where $n_{\mathrm{train}}$ and $n_{\mathrm{test}}$ denote the numbers of samples in the training/context split and test split, respectively, we perform without-replacement subsampling while approximately preserving the original train--test ratio. Let $N_{\max}=100{,}000$. The retained split sizes are determined as
\begin{equation}
    n_{\mathrm{train}}'
    =
    \operatorname{round}
    \left(
        N_{\max}
        \frac{n_{\mathrm{train}}}
        {n_{\mathrm{train}}+n_{\mathrm{test}}}
    \right),
    \qquad
    n_{\mathrm{test}}'
    =
    N_{\max}-n_{\mathrm{train}}'.
\end{equation}

The training/context subset is then sampled uniformly without replacement using a seed-specific random generator. Since the ODDBench training split contains only normal samples, no stratification is required for this split. For the test set, we perform stratified random subsampling to preserve the original proportion of normal and anomalous samples as closely as possible. This procedure preserves the overall train--test structure and test contamination ratio while maintaining a tractable computational cost.

\subsection{Baseline Settings}
\label{Appendix:Baselines}
\paragraph{Conventional baselines.}
Following~\citep{shen2025fomo0dfoundationmodelzeroshot, ding2026zeroheroadvancingzeroshot}, we compare \method against 25 widely used unsupervised anomaly and outlier detection methods adopted from DTE~\citep{ICLR2024_6dfd16ff}. The classical and shallow baselines include ECOD, iForest, kNN, LOF, OCSVM, PCA, MCD, HBOS, COPOD, FeatureBagging, and LODA. These methods are implemented using PyOD~\citep{10.1145/3701716.3715196}, a widely used Python library for outlier detection. For deep learning-based baselines, we include VAE, Deep-SVDD, SLAD, ICL, GOAD, DDPM, PlanarFlow, DAGMM, DROCC, GANomaly, and DIF. We additionally evaluate the diffusion-based DTE variants, including DTE-C, DTE-IG, and DTE-NP.

Consistent with the evaluation protocols of prior TFM-based anomaly detection studies~\citep{shen2025fomo0dfoundationmodelzeroshot, ding2026zeroheroadvancingzeroshot}, we use the default PyOD hyperparameter configurations for all classical and shallow baselines. For VAE and Deep-SVDD, we likewise use the implementations and default configurations provided by PyOD. For the remaining methods, we follow the hyperparameter settings reported in their original papers and the official DTE repository~\citep{ICLR2024_6dfd16ff}.

\paragraph{PFN-based baselines.}
We include five TFM-based baselines for tabular anomaly detection: FoMo-0D~\citep{shen2025fomo0dfoundationmodelzeroshot}, OUTFORMER~\citep{ding2026zeroheroadvancingzeroshot}, TACTIC-CLEAN and TACTIC-CONT~\citep{marszalek2026tacticnavigatingunknowntabular}, and the official TabPFN anomaly-detection extension\footnote{\url{https://docs.priorlabs.ai/capabilities/anomaly-detection}}. TACTIC-CLEAN and TACTIC-CONT correspond to independently pretrained variants designed for clean and contaminated contexts, respectively. For the TabPFN anomaly-detection extension, we use the latest TabPFN v3.0 backbone available in our evaluation setup. 

Following the official implementations and experimental settings described in the corresponding papers, we configure each baseline as follows:

For FoMo-0D, we use the officially released 100-dimensional variant. For
datasets with fewer than 100 attributes, the inputs are rescaled and
zero-padded to 100 dimensions; for higher-dimensional datasets, 100 attributes
are randomly subsampled. For OUTFORMER, which follows the same 100-dimensional
input setting as FoMo-0D, we use the officially released model without
ensembling. We apply a quantile transformation to the input features and limit the context size to at most 5,000 samples for both baselines, following the setup described in their papers and code repositories. 

For TACTIC-CLEAN and TACTIC-CONT, following the preprocessing protocol described in the official paper and repository, we set the input dimensionality to 50 attributes and apply MinMax scaling to the input features. Datasets with fewer than 50 attributes are zero-padded, whereas datasets with more than 50 attributes are evaluated using an ensemble of ten randomly sampled 50-feature subsets. We do not impose an additional context-length limit for either TACTIC
variant.

For the TabPFN anomaly-detection extension, we do not impose a context-size
limit. We use the default hyperparameters and configurations of the naive
classifier- and regressor-based anomaly detection schemes provided in the
official TabPFN-Extension implementation.

To ensure a focused and fair comparison, we restrict the tabular foundation model baselines to PFN-based approaches, as our work specifically targets anomaly detection through the predictive interface of PFN-based TFMs. We therefore exclude foundation models built on fundamentally different pretraining or inference paradigms, since their inclusion would confound the comparison of PFN-specific design choices. We further require baselines to provide official code, pretrained checkpoints, and sufficiently specified inference configurations for faithful reproduction. Accordingly, we exclude ICLAD~\citep{wei2026icladincontextlearningunified}, which releases pretrained checkpoints on Hugging Face but does not provide the official inference pipeline or complete inference settings.

\paragraph{Our setting.}
For \method, we fix all configurations of the underlying general-purpose TFM backbone,
TabICLv2. We use the pretrained checkpoint released through the official
TabICLv2 Hugging Face repository\footnote{\url{https://huggingface.co/papers/2602.11139}}
and its default classifier configuration specified in the official repository. Following the one-time real-world configuration search described in Appendix~\ref{Appendix:Configuration Search}, we fix the task-selection and score-aggregation configurations for the TabICLv2 backbone as follows:
\[
\left\{
\rho_H = 0.10,\;
\gamma_{\mathrm{AUC}} = 0.70,\;
w_A = 1.5,\;
\pi_{\mathrm{ret}} = \textsc{All},\;
\operatorname{Ensemble} = \textsc{low2mean}
\right\}.
\]
These configurations are then held fixed and applied to all unseen datasets without dataset-specific tuning. Other fixed framework configurations are specified in Appendices~\ref{appendix:pseudo_ood_generation}, \ref{appendix:task_filtering}, \ref{appendix:anomaly_scoring} and
\ref{Appendix:Configuration Search}.

\subsection{Evaluation Metrics}
\label{Appendix:EvaluationMetrics}
Based on the per-dataset AUCROC and AUCPR results, we evaluate the relative performance of each method using several aggregate metrics beyond the average score and average rank. 
This provides a more robust comparison across the large and heterogeneous set of evaluation datasets. All aggregate metrics are computed independently for
AUCROC and AUCPR.
\paragraph{Elo score.}
We use an Elo-style rating~\citep{elo1967proposed} to summarize pairwise
performance across datasets. For each dataset, every pair of methods is
compared according to its detection performance, assigning one point for a
win, one-half point for a tie, and zero points for a loss. We estimate the
relative strength $\beta_i$ of each method $i$ from these pairwise outcomes
using the Bradley--Terry model~\citep{10.1214/aos/1079120141} and map the resulting strength to the Elo scale
as
\begin{equation}
    R_i = 1000 + 400\log_{10}\beta_i.
\end{equation}
This parameterization follows the standard Elo interpretation, under which
the expected pairwise score of method $i$ against method $j$ is
\begin{equation}
    E_{ij}
    =
    \frac{1}
    {1 + 10^{(R_j-R_i)/400}}.
\end{equation}
A higher Elo score therefore indicates stronger and more consistent pairwise
performance across the evaluation datasets. 
%We compute the Elo ratings independently for each random seed and report their mean across five seeds.

\paragraph{Top-3 ratio.}
The Top-3 ratio measures how frequently a method ranks among the three
best-performing methods across datasets. Specifically, we rank all methods
within each dataset according to the corresponding evaluation metric and
compute
\begin{equation}
    \operatorname{Top3}(i)
    =
    \frac{1}{N}
    \sum_{n=1}^{N}
    \1
    \left[
        r_{i,n} \leq 3
    \right],
\end{equation}
where $r_{i,n}$ denotes the rank of method $i$ on dataset $n$, and $N$ is
the total number of datasets. A larger Top-3 ratio indicates that a method
more frequently achieves competitive performance across datasets.

\paragraph{Rank distribution.}
As shown in Figure~\ref{fig:aucroc_elo_rank_dist} and Figure~\ref{fig:aucpr_elo_rank_dist}, we further report the distribution of per-dataset ranks for each method.
For each dataset, methods are ranked according to their detection performance,
and the resulting ranks are collected across all 790 ODDBench datasets.
We visualize these distributions using box plots, which provide a
complementary view of both the typical rank and its variability across
heterogeneous datasets.

\section{Further Experiments}
\label{App:FurtherExperiments}
\subsection{Extended Overall Performance Analysis}
\label{app:ExtendedOverallPerformance}

\begin{table*}[t]
\centering
\caption{Extended results corresponding to Table~\ref{tab:oddbench-top5-aucroc}. Total Rank is computed as the mean rank among 31 methods across eight evaluation metrics: average AUCROC and AUCPR, average AUCROC and AUCPR ranks, Elo scores for AUCROC and AUCPR, and Top-3 Ratios for AUCROC and AUCPR. Our method achieves the strongest overall performance across these metrics.}
% \vspace{-10pt}
\label{tab:oddbench-31-baselines}
\setlength{\tabcolsep}{3.2pt}
\scriptsize
\resizebox{\textwidth}{!}{%
\begin{tabular}{llrrrrrrrrr}
\toprule
\multicolumn{1}{c}{\textbf{}} & \textbf{Methods}
& \begin{tabular}[c]{@{}c@{}}\textbf{Avg.}\\\textbf{AUCROC $\uparrow$}\end{tabular}
& \begin{tabular}[c]{@{}c@{}}\textbf{Avg. Rank}\\\textbf{(AUCROC) $\downarrow$}\end{tabular}
& \begin{tabular}[c]{@{}c@{}}\textbf{Avg.}\\\textbf{AUCPR $\uparrow$}\end{tabular}
& \begin{tabular}[c]{@{}c@{}}\textbf{Avg. Rank}\\\textbf{(AUCPR) $\downarrow$}\end{tabular}
& \begin{tabular}[c]{@{}c@{}}\textbf{ELO}\\\textbf{(AUCROC) $\uparrow$}\end{tabular}
& \begin{tabular}[c]{@{}c@{}}\textbf{ELO}\\\textbf{(AUCPR) $\uparrow$}\end{tabular}
& \begin{tabular}[c]{@{}c@{}}\textbf{Top-3 Ratio}\\\textbf{(AUCROC) $\uparrow$}\end{tabular}
& \begin{tabular}[c]{@{}c@{}}\textbf{Top-3 Ratio}\\\textbf{(AUCPR) $\uparrow$}\end{tabular}
& \begin{tabular}[c]{@{}c@{}}\textbf{Total}\\\textbf{Rank $\downarrow$}\end{tabular} \\
\midrule
\multirow{25}{*}{\rotatebox[origin=c]{90}{\textbf{Conventional}}} & DTE-NP & \textcolor{blue}{\underline{$76.96_{\pm 0.01}\;(2)$}} & \textcolor{blue}{\underline{$9.09_{\pm 0.03}\;(2)$}} & \textcolor{blue}{\underline{$43.63_{\pm 0.01}\;(2)$}} & \textcolor{blue}{\underline{$9.37_{\pm 0.02}\;(2)$}} & \textcolor{blue}{\underline{$1167.7_{\pm 1.0}\;(2)$}} & \textcolor{blue}{\underline{$1160.1_{\pm 0.7}\;(2)$}} & \textcolor{blue}{\underline{$25.3_{\pm 0.6}\;(2)$}} & $20.3_{\pm 0.5}\;(4)$ & \textcolor{blue}{\underline{$2.25\;(2)$}} \\
& KNN & \textcolor{green!60!black}{$76.25_{\pm 0.01}\;(3)$} & \textcolor{green!60!black}{$9.83_{\pm 0.03}\;(3)$} & $42.41_{\pm 0.01}\;(5)$ & \textcolor{green!60!black}{$10.38_{\pm 0.02}\;(3)$} & \textcolor{green!60!black}{$1144.7_{\pm 0.9}\;(3)$} & \textcolor{green!60!black}{$1130.5_{\pm 0.7}\;(3)$} & $16.8_{\pm 0.4}\;(8)$ & $13.7_{\pm 0.6}\;(10)$ & \textcolor{green!60!black}{$4.75\;(3)$} \\
& FeatureBagging & $75.59_{\pm 0.16}\;(4)$ & $12.14_{\pm 0.20}\;(4)$ & $39.30_{\pm 0.56}\;(10)$ & $13.19_{\pm 0.21}\;(10)$ & $1086.9_{\pm 4.8}\;(4)$ & $1061.4_{\pm 5.2}\;(10)$ & $19.0_{\pm 1.1}\;(6)$ & $18.0_{\pm 1.1}\;(8)$ & $7.00\;(6)$ \\
& LOF & $74.17_{\pm 0.02}\;(6)$ & $12.83_{\pm 0.03}\;(7)$ & $39.97_{\pm 0.02}\;(9)$ & $12.97_{\pm 0.04}\;(8)$ & $1068.9_{\pm 0.9}\;(7)$ & $1066.0_{\pm 1.1}\;(8)$ & $13.1_{\pm 0.4}\;(10)$ & $14.3_{\pm 0.5}\;(9)$ & $8.00\;(9)$ \\
& DTE-C & $73.46_{\pm 0.05}\;(9)$ & $13.02_{\pm 0.11}\;(8)$ & $38.16_{\pm 0.09}\;(11)$ & $13.75_{\pm 0.15}\;(12)$ & $1063.9_{\pm 3.0}\;(8)$ & $1046.7_{\pm 3.6}\;(12)$ & $10.9_{\pm 0.6}\;(13)$ & $9.4_{\pm 0.4}\;(15)$ & $11.00\;(11)$ \\
& DTE-IG & $73.34_{\pm 0.43}\;(11)$ & $13.09_{\pm 0.30}\;(9)$ & $41.92_{\pm 0.34}\;(6)$ & $12.15_{\pm 0.23}\;(5)$ & $1061.4_{\pm 7.0}\;(9)$ & $1085.1_{\pm 5.6}\;(6)$ & $15.5_{\pm 1.3}\;(9)$ & $18.5_{\pm 0.8}\;(7)$ & $7.75\;(7)$ \\
& OCSVM & $71.88_{\pm 0.01}\;(13)$ & $13.88_{\pm 0.02}\;(12)$ & $37.18_{\pm 0.01}\;(12)$ & $14.03_{\pm 0.04}\;(13)$ & $1042.9_{\pm 0.7}\;(12)$ & $1039.9_{\pm 1.1}\;(13)$ & $8.5_{\pm 0.4}\;(17)$ & $7.2_{\pm 0.2}\;(18)$ & $13.75\;(14)$ \\
& ICL & $71.82_{\pm 0.10}\;(14)$ & $14.02_{\pm 0.08}\;(13)$ & $40.07_{\pm 0.15}\;(8)$ & $13.11_{\pm 0.15}\;(9)$ & $1040.8_{\pm 1.6}\;(14)$ & $1063.3_{\pm 3.5}\;(9)$ & $11.5_{\pm 0.3}\;(11)$ & $12.7_{\pm 0.4}\;(11)$ & $11.12\;(12)$ \\
& IForest & $71.26_{\pm 0.05}\;(15)$ & $15.07_{\pm 0.05}\;(15)$ & $31.26_{\pm 0.10}\;(20)$ & $16.53_{\pm 0.10}\;(16)$ & $1021.1_{\pm 1.4}\;(15)$ & $986.2_{\pm 2.4}\;(16)$ & $4.7_{\pm 0.4}\;(27)$ & $4.4_{\pm 0.2}\;(27)$ & $18.88\;(18)$ \\
& MCD & $69.10_{\pm 0.19}\;(17)$ & $16.55_{\pm 0.15}\;(17)$ & $31.10_{\pm 0.22}\;(21)$ & $17.88_{\pm 0.12}\;(19)$ & $982.3_{\pm 3.4}\;(17)$ & $950.2_{\pm 2.7}\;(19)$ & $11.0_{\pm 0.6}\;(12)$ & $10.4_{\pm 0.4}\;(12)$ & $16.75\;(16)$ \\
& VAE & $70.20_{\pm 0.04}\;(16)$ & $15.76_{\pm 0.06}\;(16)$ & $35.34_{\pm 0.06}\;(15)$ & $15.99_{\pm 0.10}\;(15)$ & $998.7_{\pm 1.5}\;(16)$ & $993.8_{\pm 2.2}\;(15)$ & $7.6_{\pm 0.3}\;(19)$ & $6.4_{\pm 0.6}\;(22)$ & $16.75\;(16)$ \\
& PlanarFlow & $68.11_{\pm 0.40}\;(18)$ & $17.11_{\pm 0.13}\;(18)$ & $32.95_{\pm 0.17}\;(17)$ & $17.53_{\pm 0.04}\;(17)$ & $968.3_{\pm 3.5}\;(18)$ & $958.6_{\pm 1.3}\;(17)$ & $10.8_{\pm 0.7}\;(14)$ & $10.1_{\pm 0.6}\;(13)$ & $16.50\;(15)$ \\
& SLAD & $66.87_{\pm 0.05}\;(21)$ & $17.84_{\pm 0.03}\;(19)$ & $33.49_{\pm 0.09}\;(16)$ & $17.59_{\pm 0.03}\;(18)$ & $948.9_{\pm 0.6}\;(19)$ & $955.7_{\pm 0.7}\;(18)$ & $6.6_{\pm 0.2}\;(22)$ & $7.1_{\pm 0.1}\;(20)$ & $19.12\;(19)$ \\
& COPOD & $67.20_{\pm 0.01}\;(20)$ & $18.37_{\pm 0.04}\;(20)$ & $27.99_{\pm 0.01}\;(29)$ & $19.51_{\pm 0.06}\;(28)$ & $941.2_{\pm 1.1}\;(20)$ & $913.2_{\pm 1.5}\;(27)$ & $3.9_{\pm 0.3}\;(28)$ & $4.2_{\pm 0.2}\;(28)$ & $25.00\;(25)$ \\
& PCA & $64.50_{\pm 0.03}\;(26)$ & $18.49_{\pm 0.03}\;(21)$ & $29.14_{\pm 0.02}\;(26)$ & $18.68_{\pm 0.03}\;(22)$ & $936.2_{\pm 0.7}\;(22)$ & $931.7_{\pm 0.7}\;(22)$ & $5.9_{\pm 0.1}\;(24)$ & $4.7_{\pm 0.2}\;(26)$ & $23.62\;(23)$ \\
& ECOD & $68.10_{\pm 0.01}\;(19)$ & $18.52_{\pm 0.06}\;(22)$ & $28.25_{\pm 0.02}\;(28)$ & $19.01_{\pm 0.08}\;(23)$ & $938.2_{\pm 1.4}\;(21)$ & $925.9_{\pm 1.8}\;(23)$ & $6.4_{\pm 0.3}\;(23)$ & $5.9_{\pm 0.3}\;(23)$ & $22.75\;(22)$ \\
& HBOS & $66.58_{\pm 0.02}\;(22)$ & $18.62_{\pm 0.06}\;(23)$ & $29.90_{\pm 0.03}\;(24)$ & $18.34_{\pm 0.05}\;(20)$ & $934.3_{\pm 1.5}\;(23)$ & $941.5_{\pm 1.2}\;(20)$ & $6.9_{\pm 0.3}\;(21)$ & $7.1_{\pm 0.4}\;(19)$ & $21.50\;(21)$ \\
& DIF & $65.16_{\pm 0.02}\;(24)$ & $18.87_{\pm 0.04}\;(25)$ & $28.53_{\pm 0.06}\;(27)$ & $19.31_{\pm 0.04}\;(26)$ & $929.0_{\pm 0.9}\;(25)$ & $918.1_{\pm 1.0}\;(25)$ & $3.5_{\pm 0.2}\;(29)$ & $3.7_{\pm 0.2}\;(29)$ & $26.25\;(28)$ \\
& GANomaly & $66.02_{\pm 0.58}\;(23)$ & $18.65_{\pm 0.25}\;(24)$ & $31.59_{\pm 0.38}\;(19)$ & $18.38_{\pm 0.16}\;(21)$ & $930.7_{\pm 6.0}\;(24)$ & $937.6_{\pm 3.7}\;(21)$ & $8.6_{\pm 0.8}\;(16)$ & $9.0_{\pm 0.7}\;(16)$ & $20.50\;(20)$ \\
& DROCC & $61.17_{\pm 0.63}\;(28)$ & $19.53_{\pm 0.20}\;(26)$ & $29.19_{\pm 0.33}\;(25)$ & $19.05_{\pm 0.23}\;(24)$ & $909.1_{\pm 5.3}\;(26)$ & $921.3_{\pm 6.3}\;(24)$ & $5.6_{\pm 0.4}\;(25)$ & $5.6_{\pm 0.2}\;(24)$ & $25.25\;(27)$ \\
& GOAD & $60.34_{\pm 0.46}\;(31)$ & $19.99_{\pm 0.23}\;(27)$ & $30.94_{\pm 0.31}\;(23)$ & $19.15_{\pm 0.27}\;(25)$ & $896.0_{\pm 5.9}\;(27)$ & $917.6_{\pm 6.6}\;(26)$ & $7.5_{\pm 0.5}\;(20)$ & $6.4_{\pm 0.4}\;(21)$ & $25.00\;(25)$ \\
& DDPM & $64.81_{\pm 0.12}\;(25)$ & $20.48_{\pm 0.09}\;(29)$ & $30.98_{\pm 0.11}\;(22)$ & $20.03_{\pm 0.10}\;(29)$ & $883.9_{\pm 2.3}\;(29)$ & $896.4_{\pm 2.4}\;(30)$ & $3.5_{\pm 0.5}\;(29)$ & $3.4_{\pm 0.3}\;(30)$ & $27.88\;(29)$ \\
& LODA & $61.85_{\pm 0.51}\;(27)$ & $20.42_{\pm 0.20}\;(28)$ & $27.40_{\pm 0.37}\;(30)$ & $20.14_{\pm 0.24}\;(30)$ & $890.6_{\pm 5.2}\;(28)$ & $897.9_{\pm 6.2}\;(29)$ & $5.0_{\pm 1.0}\;(26)$ & $4.9_{\pm 1.0}\;(25)$ & $27.88\;(29)$ \\
& DeepSVDD & $60.80_{\pm 0.35}\;(29)$ & $20.86_{\pm 0.22}\;(30)$ & $32.10_{\pm 0.19}\;(18)$ & $19.33_{\pm 0.28}\;(27)$ & $872.1_{\pm 5.7}\;(30)$ & $911.9_{\pm 7.0}\;(28)$ & $8.4_{\pm 0.5}\;(18)$ & $9.8_{\pm 0.4}\;(14)$ & $24.25\;(24)$ \\
& DAGMM & $60.50_{\pm 1.19}\;(30)$ & $22.54_{\pm 0.47}\;(31)$ & $25.24_{\pm 1.28}\;(31)$ & $22.75_{\pm 0.40}\;(31)$ & $831.8_{\pm 13.5}\;(31)$ & $826.1_{\pm 11.3}\;(31)$ & $3.0_{\pm 0.3}\;(31)$ & $3.0_{\pm 0.5}\;(31)$ & $30.88\;(31)$ \\
\midrule
\multirow{5}{*}{\rotatebox[origin=c]{90}{\textbf{TFM-based}}} & TACTIC-Clean & $75.24_{\pm 0.01}\;(5)$ & $12.22_{\pm 0.05}\;(5)$ & $41.70_{\pm 0.02}\;(7)$ & $11.66_{\pm 0.05}\;(4)$ & $1086.1_{\pm 1.4}\;(6)$ & $1100.2_{\pm 1.3}\;(4)$ & $17.9_{\pm 0.4}\;(7)$ & $19.0_{\pm 0.5}\;(6)$ & $5.50\;(4)$ \\
& TACTIC-Cont & $73.42_{\pm 0.01}\;(10)$ & $13.74_{\pm 0.03}\;(11)$ & $36.36_{\pm 0.02}\;(14)$ & $14.59_{\pm 0.04}\;(14)$ & $1048.0_{\pm 0.7}\;(11)$ & $1028.0_{\pm 1.0}\;(14)$ & $9.9_{\pm 0.3}\;(15)$ & $8.1_{\pm 0.2}\;(17)$ & $13.25\;(13)$ \\
& OUTFORMER & $74.16_{\pm 0.08}\;(7)$ & $13.32_{\pm 0.05}\;(10)$ & \textcolor{green!60!black}{$43.00_{\pm 0.13}\;(3)$} & $12.19_{\pm 0.02}\;(6)$ & $1059.1_{\pm 1.4}\;(10)$ & $1087.3_{\pm 0.5}\;(5)$ & $19.9_{\pm 0.4}\;(4)$ & \textcolor{green!60!black}{$21.5_{\pm 0.3}\;(3)$} & $6.00\;(5)$ \\
& FoMo-0D & $72.65_{\pm 0.04}\;(12)$ & $14.03_{\pm 0.09}\;(14)$ & $42.70_{\pm 0.13}\;(4)$ & $12.44_{\pm 0.10}\;(7)$ & $1041.1_{\pm 2.1}\;(13)$ & $1080.7_{\pm 2.4}\;(7)$ & $19.1_{\pm 0.5}\;(5)$ & \textcolor{blue}{\underline{$23.4_{\pm 0.4}\;(2)$}} & $8.00\;(9)$ \\
% \midrule
& TabPFN-Extension & $73.62_{\pm 0.11}\;(8)$ & $12.29_{\pm 0.07}\;(6)$ & $37.09_{\pm 0.07}\;(13)$ & $13.43_{\pm 0.07}\;(11)$ & $1086.4_{\pm 1.6}\;(5)$ & $1058.8_{\pm 1.3}\;(11)$ & \textcolor{green!60!black}{$24.8_{\pm 0.3}\;(3)$} & $19.1_{\pm 0.3}\;(5)$ & $7.75\;(7)$ \\
\midrule
% & \textbf{\method (ours)} & \textcolor{red}{\textbf{$80.24_{\pm 0.07}\;(1)$}} & \textcolor{red}{\textbf{$8.39_{\pm 0.10}\;(1)$}} & \textcolor{red}{\textbf{$50.43_{\pm 0.10}\;(1)$}} & \textcolor{red}{\textbf{$7.82_{\pm 0.06}\;(1)$}} & \textcolor{red}{\textbf{$1189.5_{\pm 3.2}\;(1)$}} & \textcolor{red}{\textbf{$1208.1_{\pm 2.2}\;(1)$}} & \textcolor{red}{\textbf{$44.9_{\pm 1.3}\;(1)$}} & \textcolor{red}{\textbf{$46.3_{\pm 1.6}\;(1)$}} & \textcolor{red}{\textbf{$1.00\;(1)$}} \\
& \textbf{\method}

& \textcolor{red}{$\bm{80.24_{\pm 0.07}\;(1)}$}

& \textcolor{red}{$\bm{8.39_{\pm 0.10}\;(1)}$}

& \textcolor{red}{$\bm{50.43_{\pm 0.10}\;(1)}$}

& \textcolor{red}{$\bm{7.82_{\pm 0.06}\;(1)}$}

& \textcolor{red}{$\bm{1189.5_{\pm 3.2}\;(1)}$}

& \textcolor{red}{$\bm{1208.1_{\pm 2.2}\;(1)}$}

& \textcolor{red}{$\bm{44.9_{\pm 1.3}\;(1)}$}

& \textcolor{red}{$\bm{46.3_{\pm 1.6}\;(1)}$}

& \textcolor{red}{$\bm{1.00\;(1)}$} \\
\bottomrule
\end{tabular}%%
}
% \vspace{-5pt}
% \parbox{\textwidth}{\scriptsize\textit{Notes.} Values are mean $\pm$ standard deviation across five seeds. AUC and Top-3 Ratio are reported in \%. Parentheses denote relative ranks among all 31 models. Total Rank is the mean of the eight parenthetical ranks. Best, second-best, and third-best results are shown in \textcolor{red}{\textbf{red bold}}, \textcolor{blue}{\underline{blue underline}}, and \textcolor{green!60!black}{green}, respectively.}
\end{table*}
\begin{figure*}[t]
    \centering
    \includegraphics[width=\linewidth]{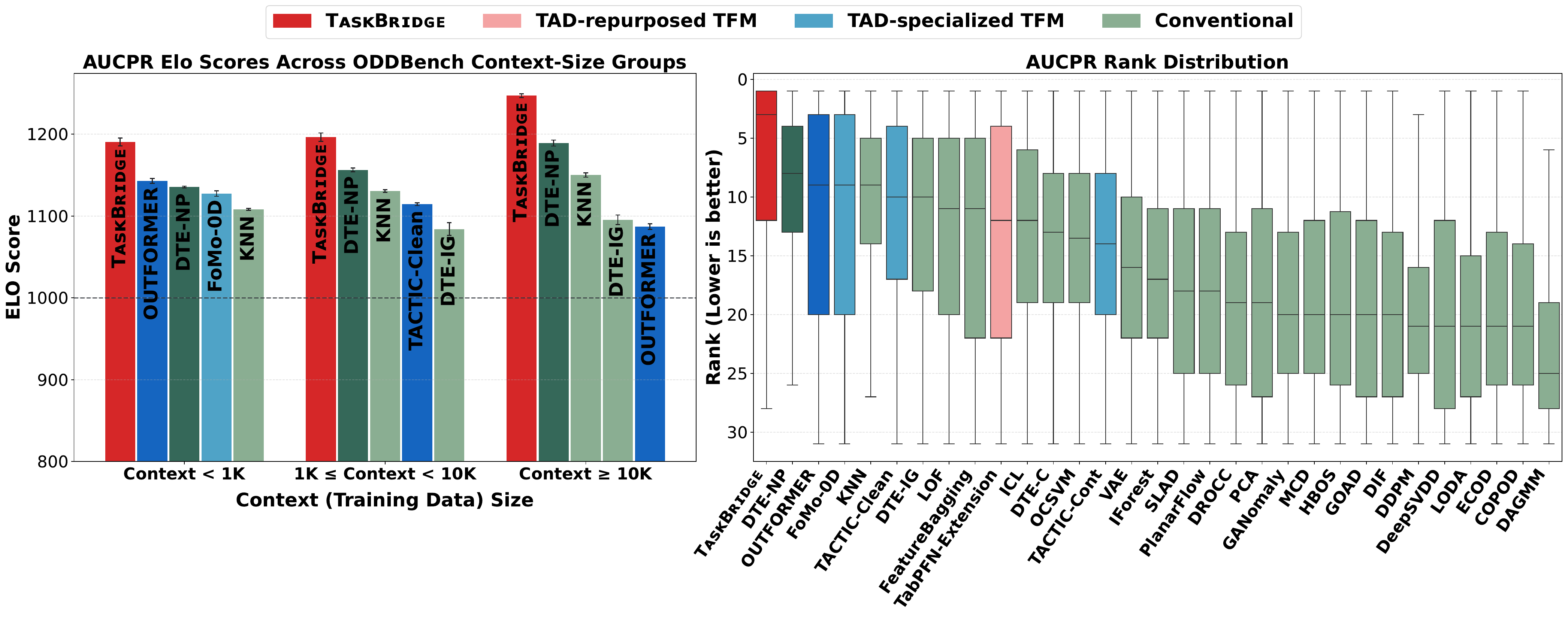}
    % \vspace{-20pt}
    % \caption{Elo score comparision between total 31 baselines. (Extended version of right-side figure of Figure~\ref{fig:overall_comparison_intro}.) \method achieves the highest Elo scores in both AUCROC and AUCPR among all baselines, demonstrating our method steadily get robust ad performance across diverse table datasets.}
    \caption{AUCPR performance comparison on ODDBench. (Extended results corresponding to Figure~\ref{fig:aucroc_elo_rank_dist}.)
(Left) Elo scores across three dataset groups
($<1$K, $1$K--$10$K, and $\geq10$K), showing the five methods with the highest average Elo scores in each group.
(Right) Per-dataset rank distributions across all baselines; colors follow
Figure~\ref{fig:overall_comparison_intro}.
\method also achieves the highest Elo score across all dataset groups with different context size and the strongest overall rank distribution for AUCPR. }
\label{fig:aucpr_elo_rank_dist}
   \vspace{-10pt}
\end{figure*}
This section provides an extended evaluation of \method against a diverse set of tabular anomaly detection baselines across 790 ODDBench datasets, complementing the results reported in the main paper (Section~\ref{ssec:overallperformance}).
% During evaluation, TACTIC, TabPFN-Extension, and VAE fail to produce results on a small subset of datasets due primarily to out-of-memory errors or dataset--model incompatibilities. To maintain a common set of datasets for aggregate comparison, these missing results are imputed using the average performance of the remaining successfully evaluated methods on the corresponding dataset. Actually, for partial ODDBench datasets including all succeed baselines (\~720 datasets), \method keep obtains best performance for all metrics. 

Table~\ref{tab:oddbench-31-baselines}, extending Table~\ref{tab:oddbench-top5-aucroc} in the main paper, compares \method with 31 baselines using eight complementary evaluation metrics: average performance, average rank, Elo score, and Top3-ratio, each computed for AUCROC and AUCPR.
Details of the Elo score and Top3-ratio are provided in Appendix~\ref{Appendix:EvaluationMetrics}. Per-dataset results for all methods are provided in the supplementary material due to the large number of evaluation datasets and baselines.

During evaluation, TACTIC, TabPFN-Extension, and VAE fail to produce results on a small subset of datasets due primarily to out-of-memory errors or dataset--model incompatibilities. To maintain a common set of datasets for aggregate comparison, these missing results are imputed using the average performance of the remaining successfully evaluated methods on the corresponding dataset.
\begin{wraptable}[11]{r}{0.42\textwidth}
    \vspace{-5pt}
    \centering
    \caption{Elo scores on the subset of ODDBench datasets for which all 31 methods successfully produce results. The five methods with the highest AUCROC Elo scores are shown.}
    \label{tab:BenchmarkCompleteDatasets}
    \vspace{-5pt}
    \small
    \begin{tabular}{l|c|c}
        \toprule
        Method & AUCROC & AUCPR \\
        \midrule
        \textbf{\method} & \red{\bm{$1188_{\pm2.6}$}} & \red{\bm{$1207_{\pm2.8}$}} \\
        DTE-NP &  \textcolor{blue}{\underline{$1167_{\pm1.1}$}} & \textcolor{blue}{\underline{$1157_{\pm1.0}$}} \\
        KNN & \textcolor{green!60!black}{$1142_{\pm0.9}$} & \textcolor{green!60!black}{$1127_{\pm0.9}$} \\
        TACTIC-Clean & $1103_{\pm1.2}$ & $1106_{\pm1.4}$ \\
        FeatureBagging & $1084_{\pm5.1}$ & $1057_{\pm5.6}$ \\
        \bottomrule
    \end{tabular}
    \vspace{-10pt}
\end{wraptable}
As a complementary check, Table~\ref{tab:BenchmarkCompleteDatasets} reports AUCROC and AUCPR Elo scores on the subset of approximately 720 ODDBench datasets for which all 31 methods successfully produce results, without requiring any imputation. \method still achieves the strongest performance on both metrics.
\begin{figure*}[t]
    \centering
    \includegraphics[width=\linewidth]{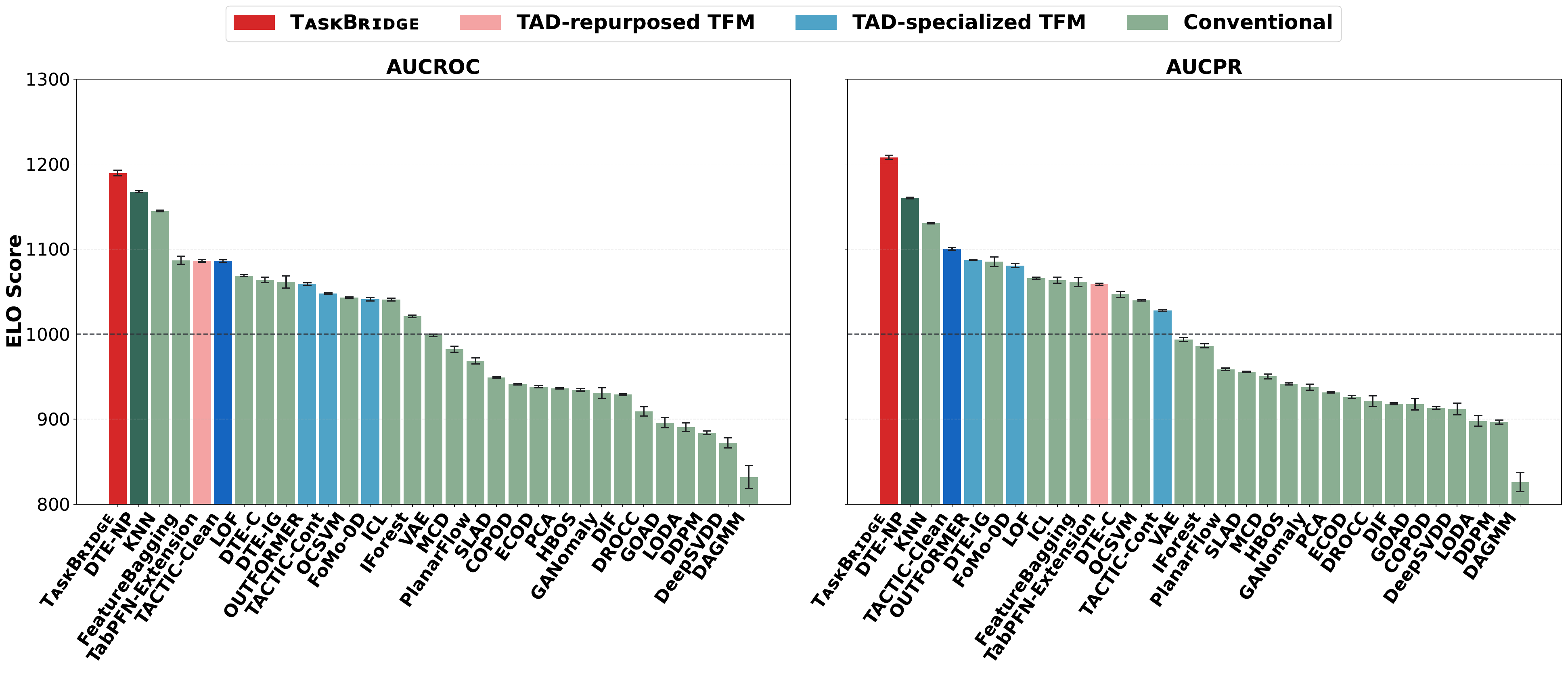}
    % \vspace{-20pt}
    % \caption{Elo score comparision between total 31 baselines. (Extended version of right-side figure of Figure~\ref{fig:overall_comparison_intro}.) \method achieves the highest Elo scores in both AUCROC and AUCPR among all baselines, demonstrating our method steadily get robust ad performance across diverse table datasets.}
    \caption{Elo score comparison across all baselines on all 790 ODDBench datasets. \method achieves the highest Elo scores for both AUCROC and AUCPR, demonstrating consistently strong performance across diverse tabular datasets.}
\label{fig:elo_aucpr_aucroc_total}
    \vspace{-10pt}
\end{figure*}

\begin{figure*}[!t]
    \centering
    \includegraphics[width=\linewidth]{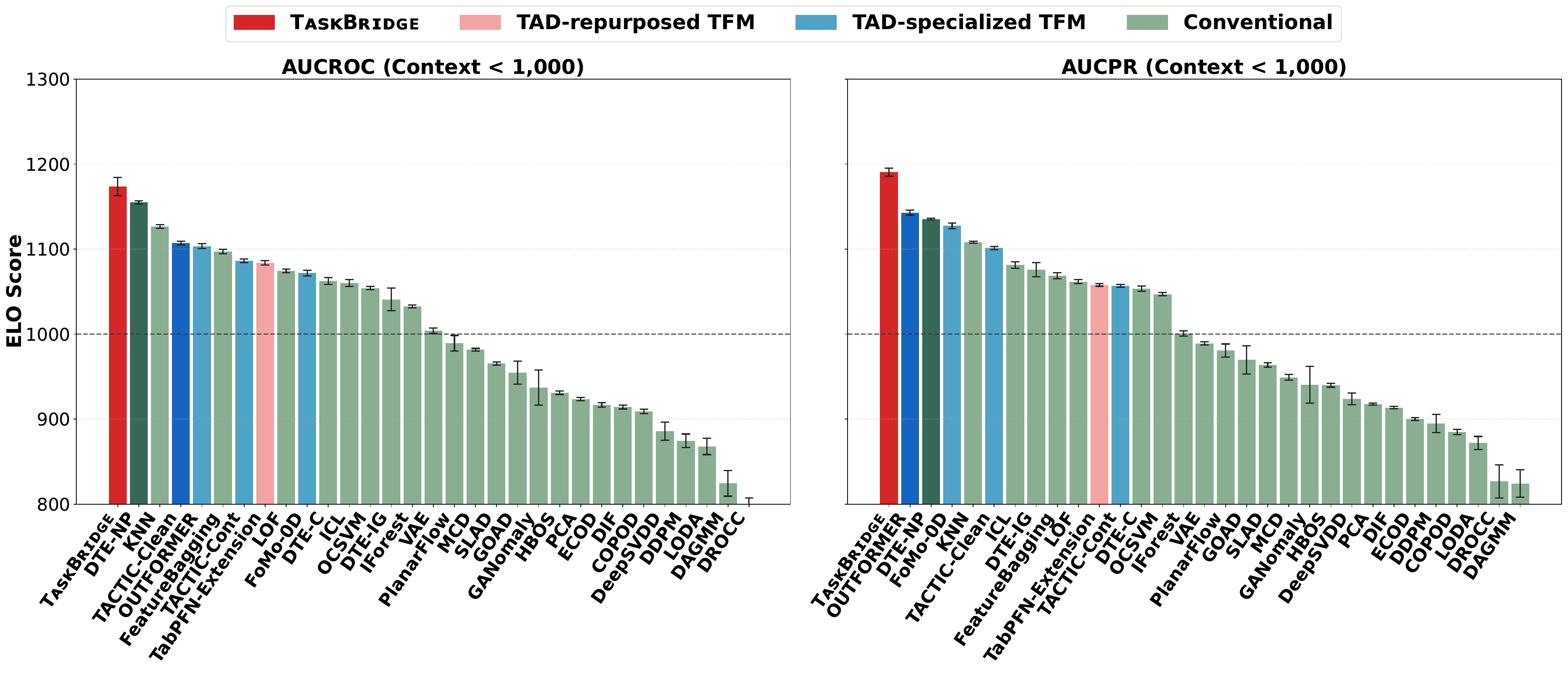}
    % \vspace{-20pt}
    % \caption{Elo score comparision between total 31 baselines. (Extended version of right-side figure of Figure~\ref{fig:overall_comparison_intro}.) \method achieves the highest Elo scores in both AUCROC and AUCPR among all baselines, demonstrating our method steadily get robust ad performance across diverse table datasets.}
    \caption{Elo score comparison across all baselines on partial ODDBench datasets with context (training data) size below $1$K. \method achieves the highest Elo scores for both AUCROC and AUCPR, demonstrating strong performance even with limited normality-anchoring evidence.}
\label{fig:elo_aucpr_aucroc_1000}
    \vspace{-10pt}
\end{figure*}

\begin{figure*}[!t]
    \centering
    \includegraphics[width=\linewidth]{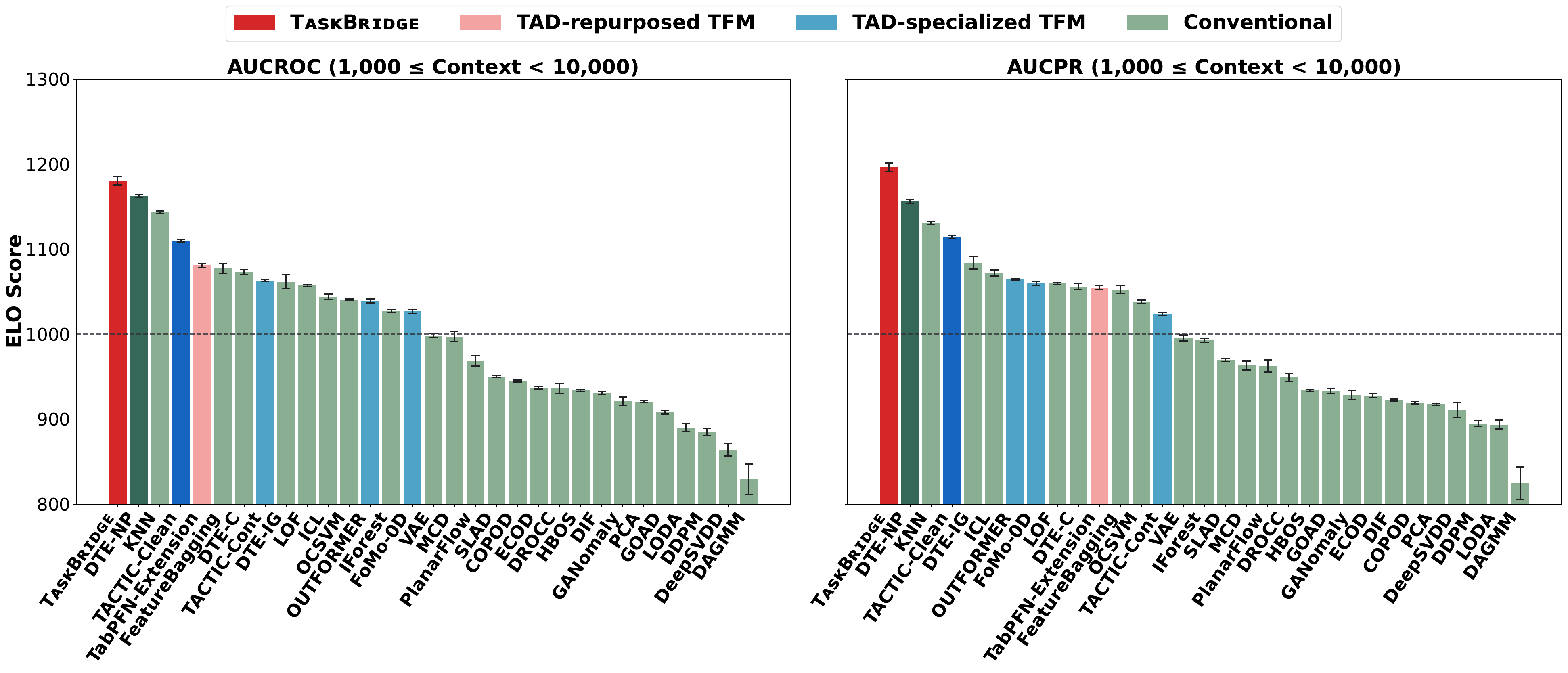}
    % \vspace{-20pt}
    % \caption{Elo score comparision between total 31 baselines. (Extended version of right-side figure of Figure~\ref{fig:overall_comparison_intro}.) \method achieves the highest Elo scores in both AUCROC and AUCPR among all baselines, demonstrating our method steadily get robust ad performance across diverse table datasets.}
    \caption{Elo score comparison across all baselines on partial ODDBench datasets with context (training data) size between $1$K and $10$K. \method achieves the highest Elo scores for both AUCROC and AUCPR.}
\label{fig:elo_aucpr_aucroc_1000_9999}
    \vspace{-10pt}
\end{figure*}

\begin{figure*}[!t]
    \centering
    \includegraphics[width=\linewidth]{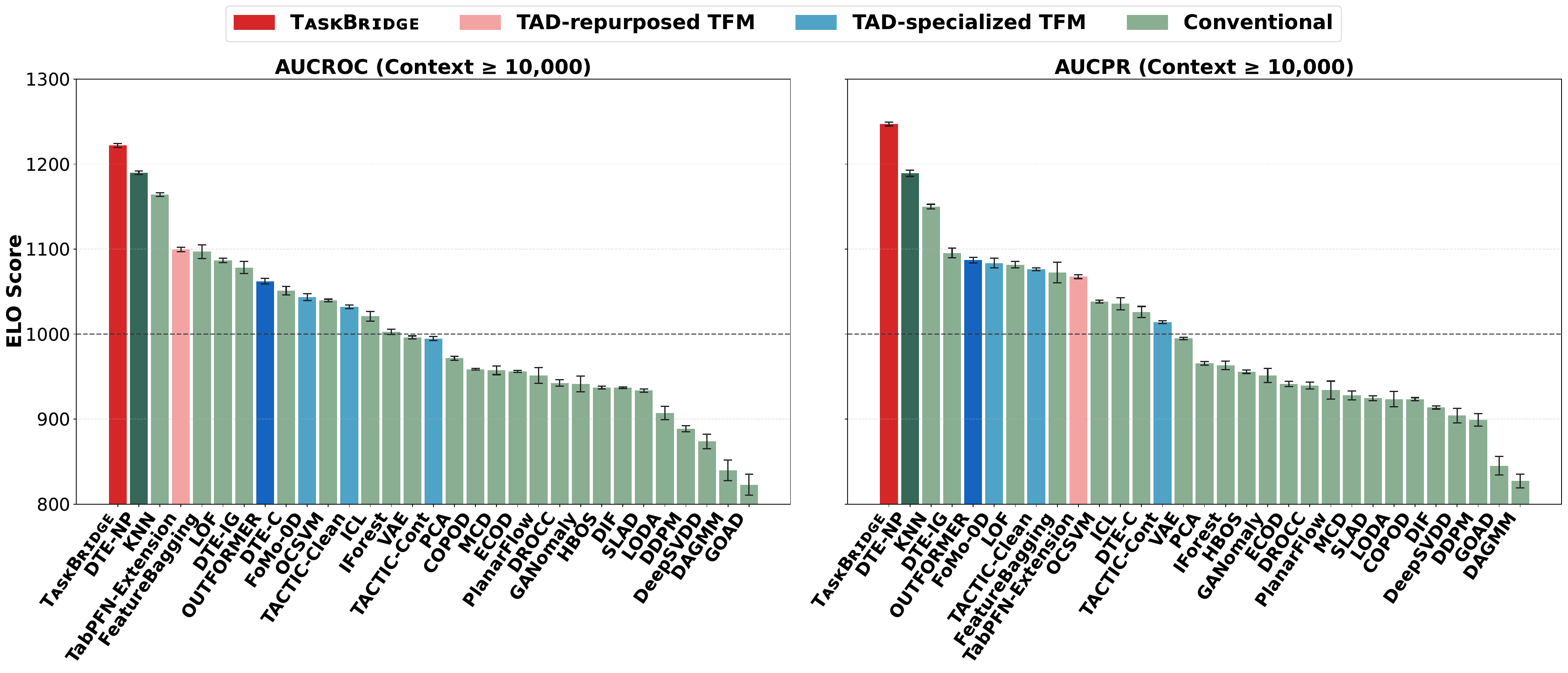}
    % \vspace{-20pt}
    % \caption{Elo score comparision between total 31 baselines. (Extended version of right-side figure of Figure~\ref{fig:overall_comparison_intro}.) \method achieves the highest Elo scores in both AUCROC and AUCPR among all baselines, demonstrating our method steadily get robust ad performance across diverse table datasets.}
    \caption{Elo score comparison across all baselines on partial ODDBench datasets with context (training data) size of at least $10$K. \method achieves the highest Elo scores for both AUCROC and AUCPR, maintaining strong performance with large amounts of normality-anchoring evidence.}
\label{fig:elo_aucpr_aucroc_10000}
\vspace{-10pt}
\end{figure*}
Figure~\ref{fig:aucpr_elo_rank_dist} provides the corresponding AUCPR analysis
to Figure~\ref{fig:aucroc_elo_rank_dist} in the main paper, reporting Elo scores
across ODDBench dataset groups stratified by context (training data) size and per-dataset rank distributions. We further
provide the full Elo score comparisons for all datasets and each context-size
group in Figures~\ref{fig:elo_aucpr_aucroc_total}--
% \ref{fig:elo_aucpr_aucroc_1000},
% \ref{fig:elo_aucpr_aucroc_1000_9999}, and
\ref{fig:elo_aucpr_aucroc_10000}.

% Required package:
% \usepackage{booktabs}

% ============================================================
% Remove-1
% ============================================================
\begin{table}[t]
\centering
\caption{Ablation results when removing a single virtual task from the set of task templates. Single denotes single-attribute, Localized/Global denote localized/global (subspace) projection, Prototype denotes prototype-based organization, and Distributional denotes distributional extremity.}
% \vspace{-5pt}
\label{tab:remove1}
% \resizebox{\linewidth}{!}{%
\small
\setlength{\tabcolsep}{3pt}
\begin{tabular}{l|ccccc}
\toprule
\textbf{Remove-1}
& \textbf{Single}
& \textbf{Localized}
& \textbf{Global}
& \textbf{Prototype}
& \textbf{Distributional} \\
\midrule
Avg. AUCROC
& 80.08 & 79.58 & 79.98 & 79.94 & 80.07 \\
Avg. AUCPR
& 49.76 & 49.26 & 50.19 & 49.77 & 49.99 \\
\bottomrule
\end{tabular}%
%\vspace{-5pt}
% }
\end{table}

% ============================================================
% Remove-2
% ============================================================
\begin{table}[!t]
\centering
\caption{Ablation results when removing two virtual tasks from the set of task templates. S denotes single-attribute, L/G denote localized/global (subspace) projection, P denotes prototype-based organization, and D denotes distributional extremity.}
% \vspace{-5pt}
\label{tab:remove2}
%\resizebox{\linewidth}{!}{%
\small
\setlength{\tabcolsep}{3pt}
\begin{tabular}{l|cccccccccc}
\toprule
\textbf{Remove-2}
& \textbf{S+L}
& \textbf{S+G}
& \textbf{S+P}
& \textbf{S+D}
& \textbf{L+G}
& \textbf{L+P}
& \textbf{L+D}
& \textbf{G+P}
& \textbf{G+D}
& \textbf{P+D} \\
\midrule
% Avg. AUC-ROC
% & -- & -- & -- & -- & --
% & -- & -- & -- & -- & -- \\
% Avg. AUC-PR
% & -- & -- & -- & -- & --
% & -- & -- & -- & -- & -- \\
Avg. AUCROC
& 79.35 & 80.11 & 79.65 & 79.72 & 79.30
& 78.95 & 79.13 & 79.62 & 79.72 & 79.60 \\
Avg. AUCPR
& 47.19 & 49.16 & 48.42 & 48.86 & 48.57
& 47.82 & 48.16 & 49.17 & 49.63 & 48.88 \\
\bottomrule
\end{tabular}%
%}
% \vspace{-5pt}
\end{table}

% ============================================================
% Remove-3
% ============================================================
\begin{table}[!t]
\centering
\caption{Ablation results when removing three virtual tasks from the set of task templates. S denotes single-attribute, L/G denote localized/global (subspace) projection, P denotes prototype-based organization, and D denotes distributional extremity.}
% \vspace{-5pt}
\label{tab:remove3}
% \resizebox{\linewidth}{!}{%
\small
\setlength{\tabcolsep}{3pt}
\begin{tabular}{l|cccccccccc}
\toprule
\textbf{Remove-3}
& \textbf{S+L+G}
& \textbf{S+L+P}
& \textbf{S+L+D}
& \textbf{S+G+P}
& \textbf{S+G+D}
& \textbf{S+P+D}
& \textbf{L+G+P}
& \textbf{L+G+D}
& \textbf{L+P+D}
& \textbf{G+P+D} \\
\midrule
% Avg. AUC-ROC
% & -- & -- & -- & -- & --
% & -- & -- & -- & -- & -- \\
% Avg. AUC-PR
% & -- & -- & -- & -- & --
% & -- & -- & -- & -- & -- \\
Avg. AUCROC
& 78.21 & 77.61 & 78.02 & 79.01 & 79.58
& 78.52 & 78.36 & 78.64 & 78.07 & 79.02 \\
Avg. AUCPR
& 45.20 & 42.60 & 43.85 & 46.30 & 47.58
& 45.62 & 45.57 & 47.15 & 45.18 & 47.96 \\
\bottomrule
\end{tabular}%
% \vspace{-5pt}
% }
\end{table}

% ============================================================
% Remove-4 (Single)
% ============================================================
\begin{table}[!t]
\centering
\caption{Ablation results when removing four virtual tasks from the set of task templates. For brevity, each variant is labeled by the remaining task. Single denotes single-attribute, Localized/Global denote localized/global (subspace) projection, Prototype denotes prototype-based organization, and Distributional denotes distributional extremity.}
% \vspace{-10pt}
\label{tab:remove4}
% \resizebox{\linewidth}{!}{%
\small
\setlength{\tabcolsep}{3pt}
\begin{tabular}{l|ccccc}
\toprule
\textbf{Remove-4}
& \textbf{Single}
& \textbf{Localized}
& \textbf{Global}
& \textbf{Prototype}
& \textbf{Distributional} \\
\midrule
% Avg. AUC-ROC
% & -- & -- & -- & -- & -- \\
% Avg. AUC-PR
% & -- & -- & -- & -- & -- \\
Avg. AUCROC
& 77.12 & 77.99 & 76.11 & 76.15 & 76.79 \\
Avg. AUCPR
& 42.33 & 42.46 & 36.51 & 39.01 & 38.76 \\
\bottomrule
\end{tabular}%
% \vspace{-10pt}
% }
\end{table}

\begin{table}[!t]
\centering
\caption{Further ablation of inter-task selection by randomly retaining increasing numbers of tasks after intra-task selection for score ensembling. Across all ODDBench datasets, performance improves as more tasks are retained, while the full \method with inter-task selection achieves the best performance.}
% \vspace{-10pt}
\label{tab:wo-inter-selection}
% \resizebox{\linewidth}{!}{%
\small
\setlength{\tabcolsep}{3pt}
\begin{tabular}{l|cccc}
\toprule
\textbf{}
& \textbf{2 Remaining Tasks}
& \textbf{3 Remaining Tasks}
& \textbf{4 Remaining Tasks}
& \textbf{Full \method}
 \\
\midrule
% Avg. AUC-ROC
% & -- & -- & -- & -- & -- \\
% Avg. AUC-PR
% & -- & -- & -- & -- & -- \\
Avg. AUCROC
& 78.11 & 79.38 & 79.89 & \textbf{80.16} \\
Avg. AUCPR
& 44.95 & 48.43 & 49.66 & \textbf{50.43} \\
\bottomrule
\end{tabular}%
% \vspace{-10pt}
% }
\end{table}
\subsection{Details of Ablation Studies}
\label{app:AblationDetails}

For all ablation studies, we vary only the component under investigation while
keeping the remaining configuration of \method fixed, using TabICLv2 as the
pretrained TFM backbone.

For the task-family ablation, shown in Figure~\ref{fig:ablation-study}, Tables~\ref{tab:remove1}, \ref{tab:remove2}, \ref{tab:remove3}, and \ref{tab:remove4} report the detailed
AUCROC and AUCPR results obtained on 790 ODDBench datasets by removing one to four task templates, respectively. Performance generally degrades as more tasks are removed,
while the remaining families still retain competitive performance. These results support the complementary contribution of the instantiated virtual-task families and the robustness of the subsequent task-selection procedure.

For the task-selection ablation, \textit{wo-Intra-Selection}
removes candidate configuration selection within each task and instead uses the same default task configuration for all tasks. For the held-out split ablation, we change only the number of held-out splits from the default $S=3$ to $S=1$, while keeping all other settings unchanged.

For \textit{wo-Inter-Selection}, we replace the systematic inter-task selection stage with random selection of a single task using its best intra-task configuration. Since the retention policy selected by the configuration search in Appendix~\ref{Appendix:Configuration Search} is \textsc{All}, which retains all tasks passing the inter-task eligibility criterion, simply retaining all eligible tasks would still preserve part of the inter-task selection mechanism.
As multiple eligible tasks are available for most datasets, the random single-task variant provides a reference for assessing the contribution of systematic inter-task retention.

To further assess this contribution, we randomly retain increasing numbers of tasks and ensemble their scores. As shown in Table~\ref{tab:wo-inter-selection}, performance generally improves as more complementary tasks are included, yet the full \method with inter-task selection still achieves the strongest performance. Together with the task-family ablation, these results suggest that task complementarity and task-score aggregation can partially compensate for the absence of inter-task selection, while retaining tasks according to anomaly-detection-oriented predictive properties provides additional gains.

Overall, the full selection procedure achieves the strongest performance, while the $S=1$ variant retains competitive detection performance with a substantially more efficient task-selection process.
\begin{figure*}[t]
    \centering

    \begin{subfigure}[c]{0.36\linewidth}
        \centering
        \includegraphics[width=\linewidth]{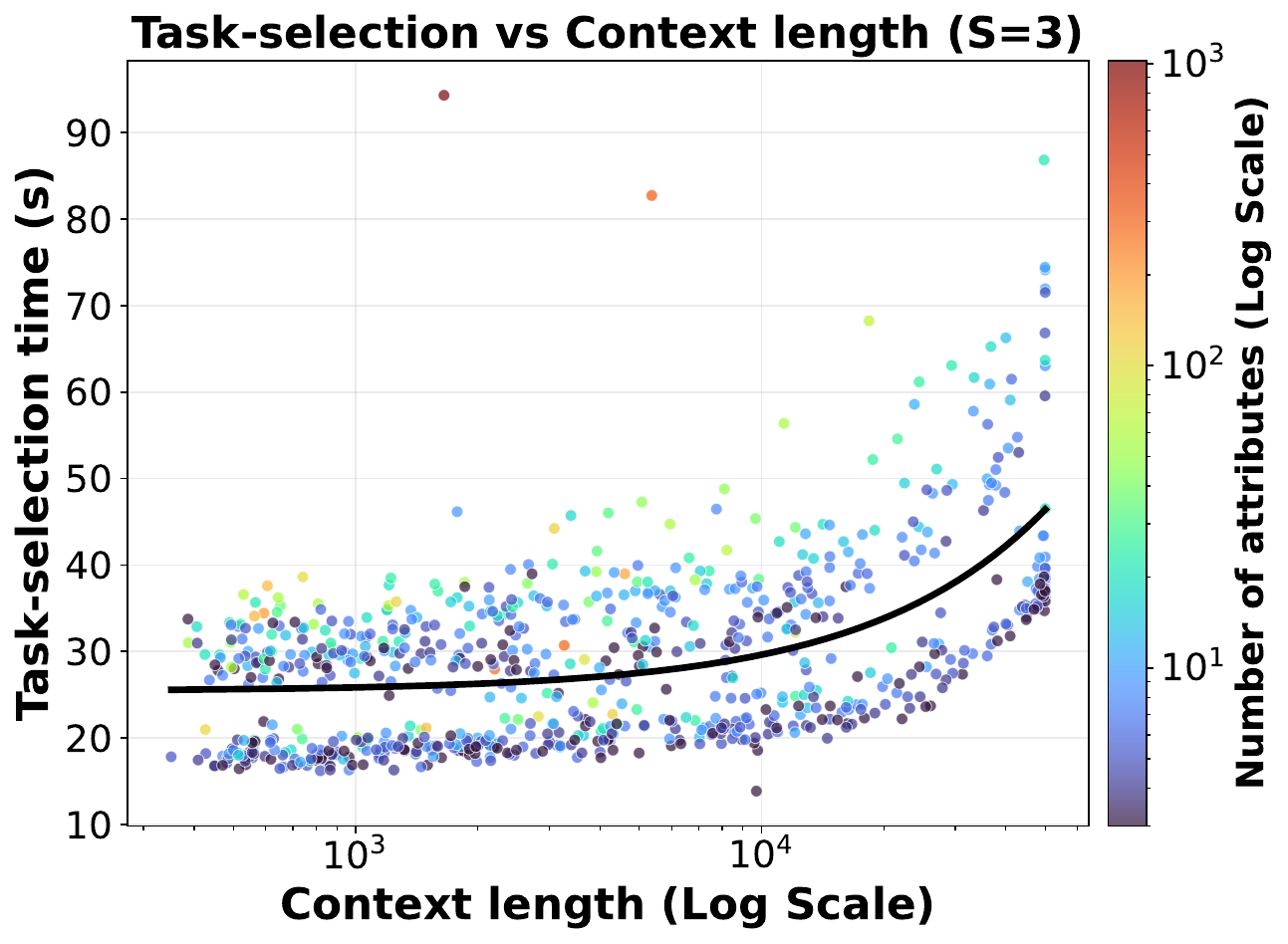}
        \label{fig:rule_selection_s3}
    \end{subfigure}
    \hfill
    \begin{subfigure}[c]{0.36\linewidth}
        \centering
        \includegraphics[
            width=\linewidth]{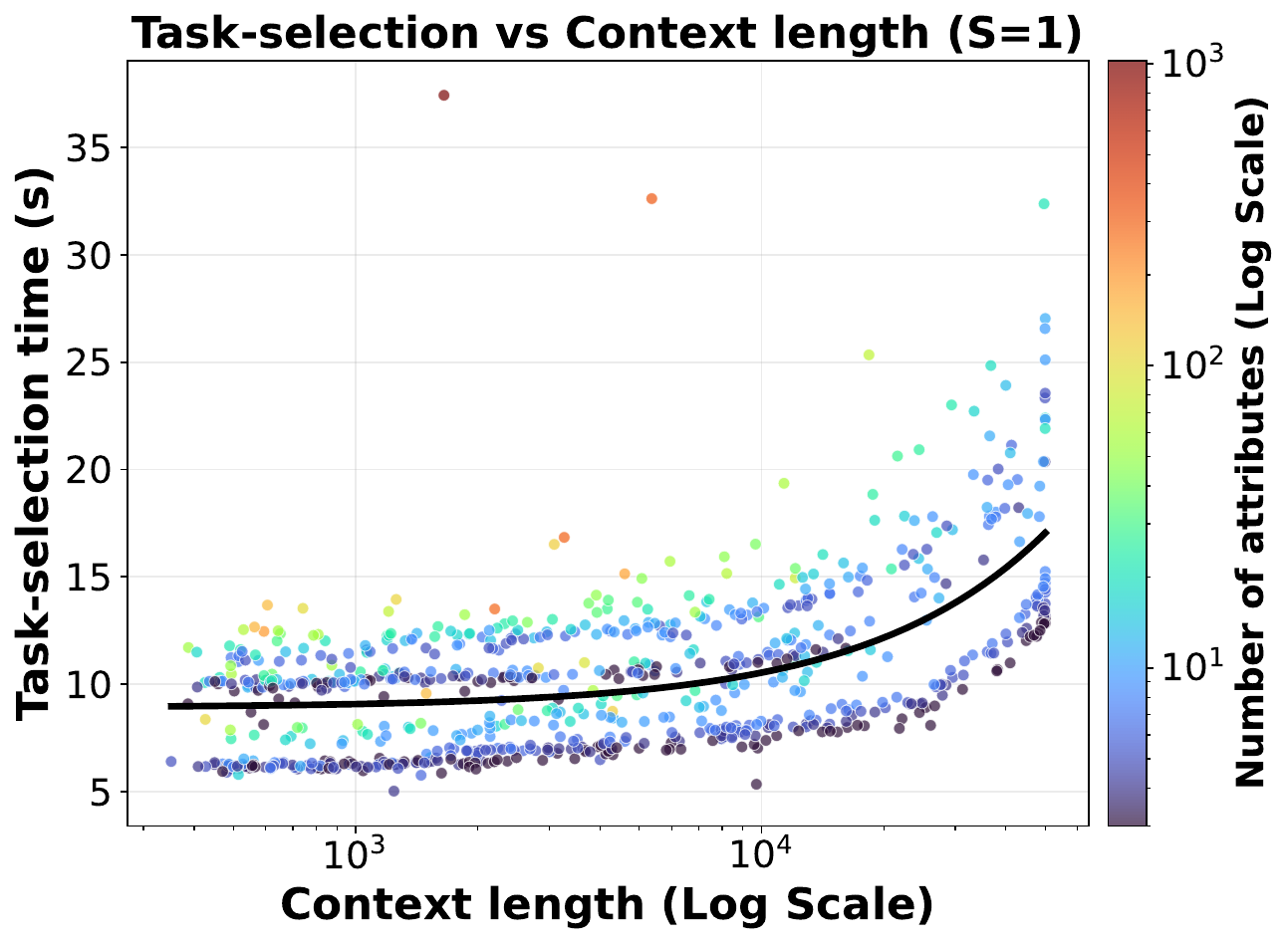}
        \label{fig:rule_selection_s1}
    \end{subfigure}
    \hfill
    \begin{subfigure}[c]{0.26\linewidth}
        \centering
        \includegraphics[
            width=\linewidth]{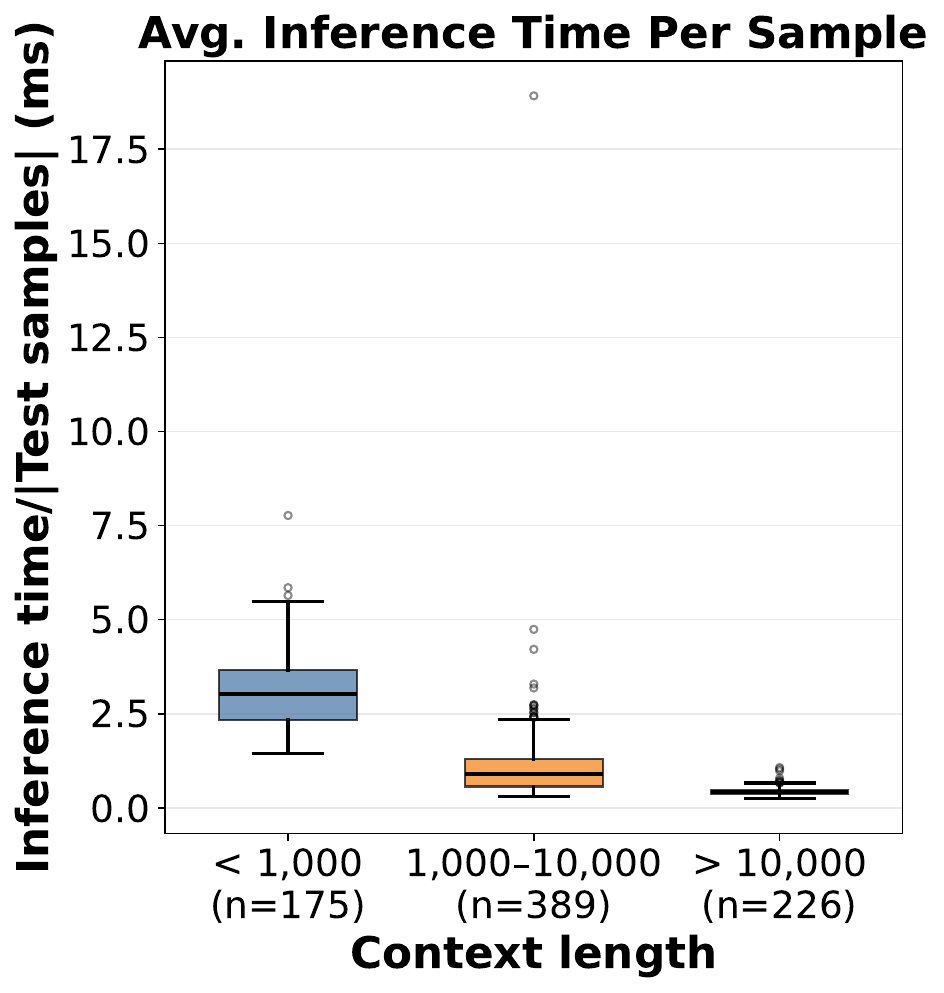}
        \label{fig:inference time per sample}
    \end{subfigure}

% \vspace{-20pt}
\caption{Inference-efficiency analysis of \method.
(Left, middle) Task-selection time across ODDBench datasets using
$3$ (default) and $1$ held-out splits, respectively, as a function of
context length and the number of attributes. 
(Right) Average inference time per test sample after task selection, grouped by
context length; \method maintains millisecond-level inference across all
groups. Here, $n$ denotes the number of datasets in each group:
$<1{,}000$, $1{,}000$--$10{,}000$, and $>10{,}000$ context
length.
Overall, \method preserves efficient inference with TAD-repurposed TFMs.}
    \vspace{-10pt}
    \label{fig:efficiency_analysis}
\end{figure*}
\subsection{Inference Efficiency}
\label{app:InferenceEfficiency}

The main inference overhead of \method arises from the hierarchical task-selection
procedure. Once the tasks are selected, inference with the pretrained
general-purpose TFM requires only milliseconds per sample, as shown in the rightmost
panel of Figure~\ref{fig:efficiency_analysis}. We evaluate the task-selection
runtime using both our default configuration with three held-out splits and an efficient variant using a single split, which also achieves the strongest overall performance against 31 baselines on ODDBench.

As shown in the leftmost panels of Figure~\ref{fig:efficiency_analysis}, the
task-selection time of the default \method remains below one minute for most
datasets, even when the context size approaches the $10^4$ scale. Reducing the
number of held-out splits to $S=1$ further lowers this cost, requiring less
than 20 seconds for most datasets across diverse context sizes and
dimensionalities. Importantly, task selection is performed only once per
dataset and requires neither dataset-specific hyperparameter search nor
parameter optimization, thereby preserving efficient TAD inference through
in-context learning.

Overall, this one-time selection cost preserves \method as an efficient TAD-repurposing framework.
As shown in Figure~\ref{fig:overall_comparison_intro}, \method remains
substantially more efficient than the existing TabPFN-Extension even when its
task-selection cost is included in the comparison.

\begin{figure*}[t]
    \centering
    \includegraphics[width=\linewidth]{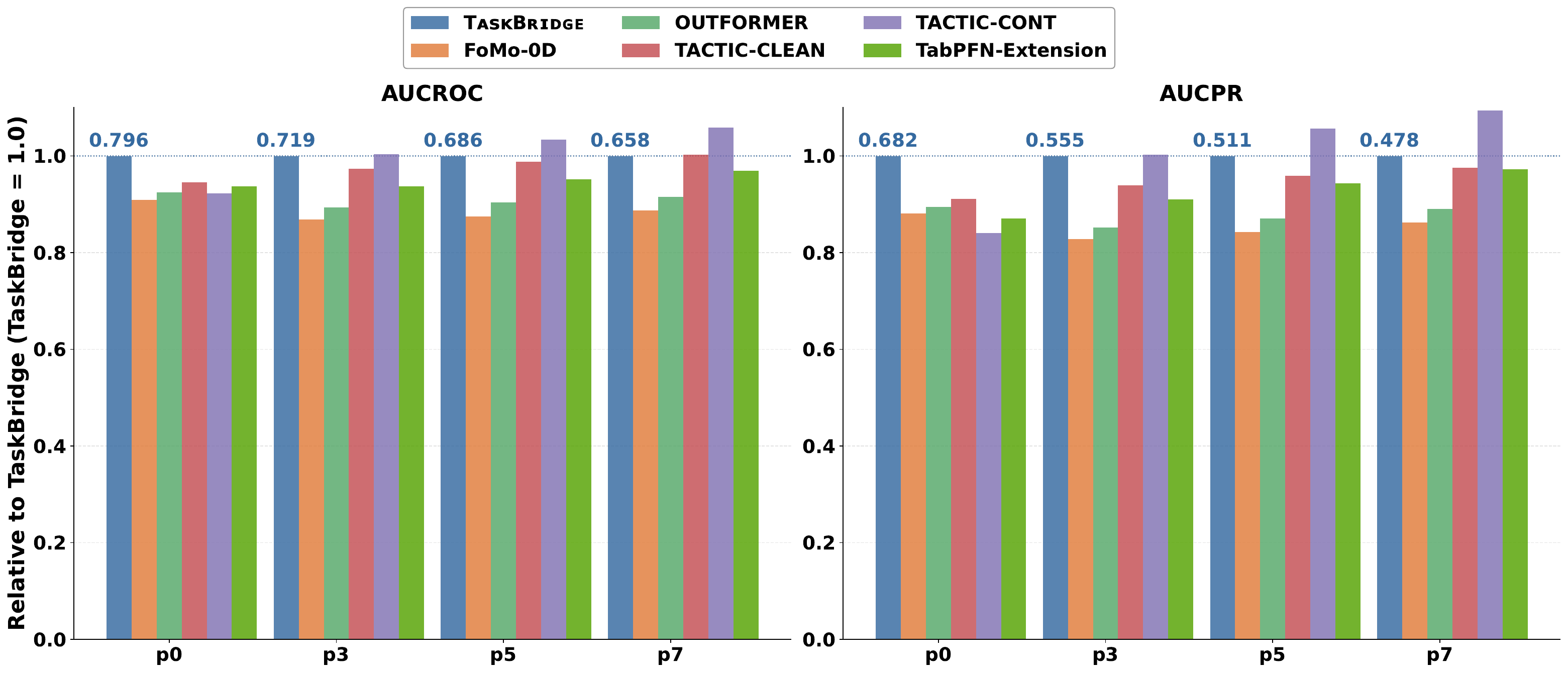}
    % \vspace{-20pt}
    % \caption{Elo score comparision between total 31 baselines. (Extended version of right-side figure of Figure~\ref{fig:overall_comparison_intro}.) \method achieves the highest Elo scores in both AUCROC and AUCPR among all baselines, demonstrating our method steadily get robust ad performance across diverse table datasets.}
    \caption{Robustness to context contamination on ODDBench. AUCROC and AUCPR performance are shown relative to \method at each contamination level (\method$=1$), with its absolute performance reported above each group. p${k}$ denotes the setting in which $k\%$ of the context samples are replaced with anomalous samples. \method remains competitive with TFM-based TAD baselines as contamination increases, while the contamination-specific TACTIC-CONT exhibits the strongest robustness under heavier contamination.}
\label{fig:ContextContamination}
    \vspace{-10pt}
\end{figure*}
\subsection{Robustness Analysis on Context Contamination}
\label{app:ContaminationRobustness}
Considering practical scenarios in which the training/context data may be contaminated, we evaluate the robustness of \method to context contamination, which can interfere with normality anchoring. We construct a fixed subset of 295 ODDBench datasets that support the most severe 10\% contamination setting without removing more than 50\% of the test anomalies. 
For each dataset, we replace a fraction of normality-aware context samples with randomly selected anomalies drawn from the test/query split. We evaluate contamination rates of 0\%, 3\%, 5\%, and 7\%.
To ensure a fair comparison across contamination levels, all settings use the same test/query set after excluding the fixed anomaly pool used for contamination. Normal query samples are then downsampled to keep the anomaly ratio in the test split consistent across contamination levels. We compare \method under each contamination level with TFM-based TAD baselines, including TACTIC-CONT, which is specifically pretrained to handle contaminated contexts.

As shown in Figure~\ref{fig:ContextContamination}, the performance of all methods generally degrades as context contamination increases. Nevertheless, \method remains competitive and generally outperforms the TFM-based baselines that are not specifically trained for contaminated contexts. TACTIC-CONT shows the strongest robustness under heavier contamination, consistent with its contamination-aware pretraining.
These results suggest that context contamination can impair the normality anchoring of virtual supervised tasks, while the data-adaptive task-selection procedure still retains relatively suitable tasks under corrupted contexts.
Improving the construction and selection of virtual tasks under contaminated contexts therefore constitutes an important direction for future work.

\begin{figure*}[t]
    \centering
    \includegraphics[width=\linewidth]{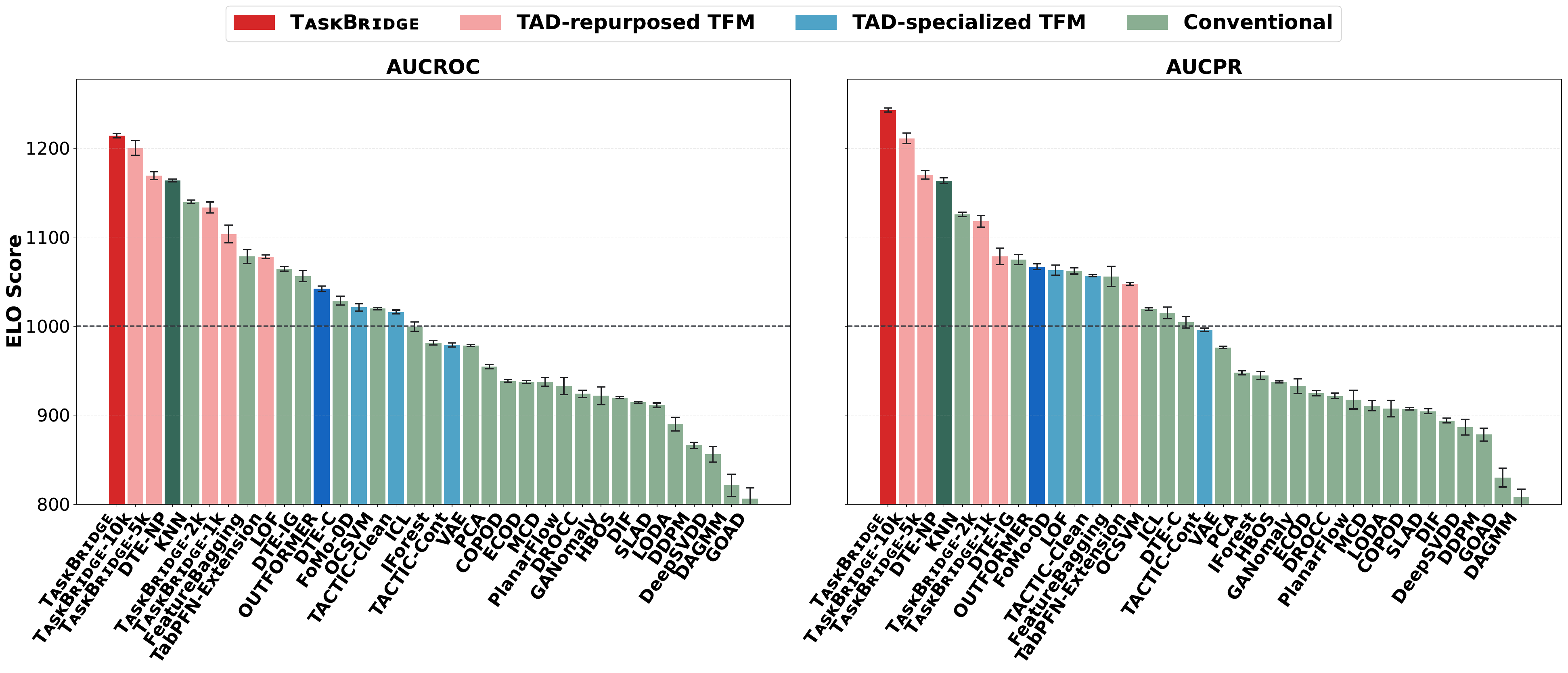}
    %\vspace{-20pt}
    % \caption{Elo score comparision between total 31 baselines. (Extended version of right-side figure of Figure~\ref{fig:overall_comparison_intro}.) \method achieves the highest Elo scores in both AUCROC and AUCPR among all baselines, demonstrating our method steadily get robust ad performance across diverse table datasets.}
    \caption{Elo score comparison across all baselines on 226 ODDBench datasets, including four \method variants with restricted context budgets. Each \method-$i$K variant limits the context size to $i$K samples. All variants, including the $1$K setting, retain strong performance relative to the other baselines.}
\label{fig:ContextSizeRobustness}
    \vspace{-15pt}
\end{figure*}
% \subsection{Robustness Analysis on Context Size}
% \label{app:ContextSizeRobustness}
% To evaluate the robustness of \method on the context size which is highly related to the normality-anchoring signals and heldout samples for selecting suitable virtual nomrality-anchored tasks,
% We design one specific experiment. First, we select datasets in ODDBench given context is larger than 10k. (Actually, there are 226 datasets). 
% For these datasets, we sample randomly 10k, 5k, 2k, 1k context samples and use them as the \method resource for normality-anchored task construction/selection and even for in-context learning for pretrained PFN-based TFMs.
% We evaluate those variants in terms of elo score on these datasets compare with 31baselines including original \method.
\subsection{Robustness Analysis on Context Size}
\label{app:ContextSizeRobustness}

To evaluate the robustness of \method to the amount of context evidence available for normality-anchored virtual task construction, data-adaptive task selection, and in-context prediction, we conduct a controlled context-size analysis. We select the 226 ODDBench datasets whose original normality-aware context contains more than $10$K samples. For each dataset, we randomly subsample the context to fixed budgets of $10$K, $5$K, $2$K, and $1$K samples. Each subsampled context is used throughout the entire \method pipeline, including normality-anchored task
construction, data-adaptive task selection, and in-context prediction with the pretrained general-purpose TFM. We compare these variants with the original full-context \method and the 30 baselines using Elo scores computed over the same 226 datasets.

As shown in Figure~\ref{fig:ContextSizeRobustness}, performance gradually decreases as the available context budget is reduced, indicating that richer context evidence benefits the entire \method pipeline. Nevertheless, \method remains highly competitive even under substantial context reduction. In particular, the $10$K and $5$K variants outperform all external baselines in both AUCROC and AUCPR Elo scores, while the $2$K and $1$K variants remain among the top-performing methods and ahead of the TFM-based TAD baselines.
% These results complement the context-size analysis in
% Section~\ref{ssec:overallperformance}, where \method also exhibits strong performance across datasets with naturally varying context sizes.
% More importantly, the controlled analysis here shows that \method remains effective even when the same datasets are provided with substantially less nominal evidence. 
These results complement the context-size analysis in Section~\ref{ssec:overallperformance}, where \method also exhibits strong performance across datasets with naturally varying context sizes. This robustness suggests that complementary normality-anchored task constructions, which probe different predictive structures in the context, together with data-adaptive task selection, can preserve informative task structures even under limited context budgets, while benefiting from the context-efficient predictive capability of pretrained general-purpose TFMs.
Together with the inference analysis in Appendix~\ref{app:InferenceEfficiency}, these results further indicate that \method can operate efficiently with small context budgets while maintaining robust anomaly detection performance.

\subsection{Generalization to Different Pretrained PFN-based TFM Backbones}
\label{app:MethodwDifferentTFM}
\begin{table*}[t]
\centering
\caption{Backbone generalization of \method with TabPFN v2.6. Total Rank is computed as in Table~\ref{tab:oddbench-31-baselines}. With this alternative pretrained general-purpose TFM backbone, \method remains highly competitive among 31 methods.}

\label{tab:oddbench-31-baselines-tabpfnv2p6}
\setlength{\tabcolsep}{3.2pt}
\scriptsize
\resizebox{\textwidth}{!}{%
\begin{tabular}{llrrrrrrrrr}
\toprule
\multicolumn{1}{c}{\textbf{}} & \textbf{Methods}
& \begin{tabular}[c]{@{}c@{}}\textbf{Avg.}\\\textbf{AUCROC $\uparrow$}\end{tabular}
& \begin{tabular}[c]{@{}c@{}}\textbf{Avg. Rank}\\\textbf{(AUCROC) $\downarrow$}\end{tabular}
& \begin{tabular}[c]{@{}c@{}}\textbf{Avg.}\\\textbf{AUCPR $\uparrow$}\end{tabular}
& \begin{tabular}[c]{@{}c@{}}\textbf{Avg. Rank}\\\textbf{(AUCPR) $\downarrow$}\end{tabular}
& \begin{tabular}[c]{@{}c@{}}\textbf{ELO}\\\textbf{(AUCROC) $\uparrow$}\end{tabular}
& \begin{tabular}[c]{@{}c@{}}\textbf{ELO}\\\textbf{(AUCPR) $\uparrow$}\end{tabular}
& \begin{tabular}[c]{@{}c@{}}\textbf{Top-3 Ratio}\\\textbf{(AUCROC) $\uparrow$}\end{tabular}
& \begin{tabular}[c]{@{}c@{}}\textbf{Top-3 Ratio}\\\textbf{(AUCPR) $\uparrow$}\end{tabular}
& \begin{tabular}[c]{@{}c@{}}\textbf{Total}\\\textbf{Rank $\downarrow$}\end{tabular} \\
\midrule
\multirow{25}{*}{\rotatebox[origin=c]{90}{\textbf{Conventional}}} & DTE-NP & \textcolor{red}{$\bm{76.96_{\pm 0.01}\;(1)}$} & \textcolor{red}{$\bm{8.90_{\pm 0.02}\;(1)}$} & \textcolor{red}{$\bm{43.63_{\pm 0.01}\;(1)}$} & \textcolor{red}{$\bm{9.18_{\pm 0.03}\;(1)}$} & \textcolor{red}{$\bm{1172.2_{\pm 0.9}\;(1)}$} & \textcolor{red}{$\bm{1164.6_{\pm 0.9}\;(1)}$} & \textcolor{red}{$\bm{27.9_{\pm 0.5}\;(1)}$} & $22.2_{\pm 0.5}\;(4)$ & \textcolor{red}{$\bm{1.38\;(1)}$} \\
 & KNN & \textcolor{blue}{\underline{$76.25_{\pm 0.01}\;(2)$}} & \textcolor{blue}{\underline{$9.65_{\pm 0.02}\;(2)$}} & $42.41_{\pm 0.01}\;(4)$ & \textcolor{blue}{\underline{$10.20_{\pm 0.03}\;(2)$}} & \textcolor{blue}{\underline{$1149.0_{\pm 0.9}\;(2)$}} & \textcolor{blue}{\underline{$1134.8_{\pm 1.0}\;(2)$}} & $18.6_{\pm 0.6}\;(8)$ & $14.9_{\pm 0.5}\;(10)$ & \textcolor{blue}{\underline{$4.00\;(2)$}} \\
 & FeatureBagging & \textcolor{green!60!black}{$75.59_{\pm 0.16}\;(3)$} & \textcolor{green!60!black}{$11.97_{\pm 0.21}\;(3)$} & $39.30_{\pm 0.56}\;(10)$ & $13.00_{\pm 0.22}\;(10)$ & \textcolor{green!60!black}{$1091.1_{\pm 5.0}\;(3)$} & $1065.9_{\pm 5.3}\;(10)$ & $19.8_{\pm 0.9}\;(6)$ & $18.9_{\pm 1.3}\;(8)$ & $6.62\;(6)$ \\
 & LOF & $74.17_{\pm 0.02}\;(6)$ & $12.64_{\pm 0.04}\;(6)$ & $39.97_{\pm 0.02}\;(9)$ & $12.77_{\pm 0.05}\;(7)$ & $1073.6_{\pm 1.0}\;(6)$ & $1070.9_{\pm 1.2}\;(8)$ & $14.4_{\pm 0.6}\;(10)$ & $16.3_{\pm 1.0}\;(9)$ & $7.62\;(9)$ \\
 & DTE-C & $73.46_{\pm 0.05}\;(9)$ & $12.85_{\pm 0.12}\;(7)$ & $38.16_{\pm 0.09}\;(11)$ & $13.58_{\pm 0.14}\;(12)$ & $1068.1_{\pm 3.2}\;(7)$ & $1051.0_{\pm 3.5}\;(12)$ & $11.3_{\pm 0.7}\;(12)$ & $9.7_{\pm 0.6}\;(15)$ & $10.62\;(11)$ \\
 & DTE-IG & $73.34_{\pm 0.43}\;(11)$ & $12.91_{\pm 0.30}\;(8)$ & $41.92_{\pm 0.34}\;(6)$ & $11.96_{\pm 0.23}\;(4)$ & $1065.8_{\pm 7.0}\;(9)$ & $1089.8_{\pm 5.5}\;(5)$ & $16.3_{\pm 1.5}\;(9)$ & $20.0_{\pm 1.1}\;(5)$ & $7.12\;(7)$ \\
 & OCSVM & $71.88_{\pm 0.01}\;(13)$ & $13.73_{\pm 0.03}\;(12)$ & $37.18_{\pm 0.01}\;(12)$ & $13.86_{\pm 0.04}\;(13)$ & $1046.6_{\pm 0.8}\;(12)$ & $1044.2_{\pm 0.9}\;(13)$ & $8.8_{\pm 0.6}\;(17)$ & $7.4_{\pm 0.3}\;(19)$ & $13.88\;(14)$ \\
 & ICL & $71.82_{\pm 0.10}\;(14)$ & $13.86_{\pm 0.07}\;(14)$ & $40.07_{\pm 0.15}\;(8)$ & $12.92_{\pm 0.15}\;(9)$ & $1044.8_{\pm 1.4}\;(14)$ & $1067.8_{\pm 3.5}\;(9)$ & $11.9_{\pm 0.2}\;(11)$ & $13.2_{\pm 0.4}\;(11)$ & $11.25\;(12)$ \\
 & IForest & $71.26_{\pm 0.05}\;(15)$ & $14.92_{\pm 0.06}\;(15)$ & $31.26_{\pm 0.10}\;(20)$ & $16.37_{\pm 0.10}\;(16)$ & $1024.9_{\pm 1.4}\;(15)$ & $990.7_{\pm 2.4}\;(16)$ & $5.1_{\pm 0.6}\;(27)$ & $4.5_{\pm 0.1}\;(27)$ & $18.88\;(18)$ \\
 & MCD & $69.10_{\pm 0.19}\;(17)$ & $16.41_{\pm 0.15}\;(17)$ & $31.10_{\pm 0.22}\;(21)$ & $17.75_{\pm 0.11}\;(19)$ & $986.2_{\pm 3.3}\;(17)$ & $954.5_{\pm 2.6}\;(19)$ & $11.3_{\pm 0.7}\;(13)$ & $10.3_{\pm 0.6}\;(12)$ & $16.88\;(17)$ \\
 & VAE & $70.20_{\pm 0.04}\;(16)$ & $15.60_{\pm 0.07}\;(16)$ & $35.34_{\pm 0.07}\;(15)$ & $15.80_{\pm 0.10}\;(15)$ & $1002.9_{\pm 1.6}\;(16)$ & $998.8_{\pm 2.2}\;(15)$ & $7.9_{\pm 0.4}\;(19)$ & $6.7_{\pm 0.3}\;(21)$ & $16.62\;(16)$ \\
 & PlanarFlow & $68.11_{\pm 0.40}\;(18)$ & $16.99_{\pm 0.13}\;(18)$ & $32.95_{\pm 0.17}\;(17)$ & $17.40_{\pm 0.05}\;(17)$ & $971.8_{\pm 3.5}\;(18)$ & $962.6_{\pm 1.5}\;(17)$ & $11.2_{\pm 0.6}\;(14)$ & $10.2_{\pm 0.8}\;(13)$ & $16.50\;(15)$ \\
 & SLAD & $66.87_{\pm 0.05}\;(21)$ & $17.71_{\pm 0.03}\;(19)$ & $33.49_{\pm 0.09}\;(16)$ & $17.44_{\pm 0.04}\;(18)$ & $953.0_{\pm 0.6}\;(19)$ & $960.0_{\pm 0.8}\;(18)$ & $6.9_{\pm 0.1}\;(23)$ & $7.3_{\pm 0.2}\;(20)$ & $19.25\;(19)$ \\
 & COPOD & $67.20_{\pm 0.01}\;(20)$ & $18.22_{\pm 0.05}\;(20)$ & $27.99_{\pm 0.01}\;(29)$ & $19.35_{\pm 0.06}\;(28)$ & $945.7_{\pm 1.3}\;(20)$ & $918.3_{\pm 1.5}\;(27)$ & $3.8_{\pm 0.3}\;(28)$ & $4.3_{\pm 0.2}\;(28)$ & $25.00\;(25)$ \\
 & PCA & $64.50_{\pm 0.03}\;(26)$ & $18.37_{\pm 0.02}\;(22)$ & $29.14_{\pm 0.02}\;(26)$ & $18.53_{\pm 0.04}\;(22)$ & $939.9_{\pm 0.6}\;(22)$ & $936.4_{\pm 0.8}\;(22)$ & $6.3_{\pm 0.3}\;(24)$ & $5.0_{\pm 0.3}\;(26)$ & $23.75\;(23)$ \\
 & ECOD & $68.10_{\pm 0.01}\;(19)$ & $18.37_{\pm 0.06}\;(21)$ & $28.25_{\pm 0.02}\;(28)$ & $18.85_{\pm 0.08}\;(23)$ & $942.6_{\pm 1.5}\;(21)$ & $930.9_{\pm 1.9}\;(23)$ & $6.9_{\pm 0.4}\;(22)$ & $6.2_{\pm 0.4}\;(23)$ & $22.50\;(22)$ \\
 & HBOS & $66.58_{\pm 0.02}\;(22)$ & $18.48_{\pm 0.06}\;(23)$ & $29.90_{\pm 0.03}\;(24)$ & $18.19_{\pm 0.05}\;(20)$ & $938.5_{\pm 1.6}\;(23)$ & $946.1_{\pm 1.1}\;(20)$ & $7.0_{\pm 0.3}\;(21)$ & $7.6_{\pm 0.5}\;(18)$ & $21.38\;(21)$ \\
 & DIF & $65.16_{\pm 0.02}\;(24)$ & $18.74_{\pm 0.04}\;(25)$ & $28.53_{\pm 0.06}\;(27)$ & $19.19_{\pm 0.04}\;(26)$ & $933.1_{\pm 1.0}\;(25)$ & $922.4_{\pm 1.0}\;(25)$ & $3.6_{\pm 0.1}\;(29)$ & $3.8_{\pm 0.2}\;(29)$ & $26.25\;(28)$ \\
 & GANomaly & $66.02_{\pm 0.58}\;(23)$ & $18.52_{\pm 0.26}\;(24)$ & $31.59_{\pm 0.38}\;(19)$ & $18.24_{\pm 0.16}\;(21)$ & $934.7_{\pm 6.0}\;(24)$ & $942.3_{\pm 3.7}\;(21)$ & $9.0_{\pm 0.9}\;(16)$ & $9.6_{\pm 0.7}\;(16)$ & $20.50\;(20)$ \\
 & DROCC & $61.17_{\pm 0.63}\;(28)$ & $19.41_{\pm 0.20}\;(26)$ & $29.19_{\pm 0.33}\;(25)$ & $18.91_{\pm 0.23}\;(24)$ & $913.0_{\pm 5.2}\;(26)$ & $925.9_{\pm 6.2}\;(24)$ & $5.9_{\pm 0.6}\;(25)$ & $6.1_{\pm 0.4}\;(24)$ & $25.25\;(27)$ \\
 & GOAD & $60.34_{\pm 0.46}\;(31)$ & $19.91_{\pm 0.23}\;(27)$ & $30.94_{\pm 0.31}\;(23)$ & $19.03_{\pm 0.27}\;(25)$ & $899.2_{\pm 5.9}\;(27)$ & $921.8_{\pm 6.6}\;(26)$ & $7.7_{\pm 0.5}\;(20)$ & $6.5_{\pm 0.4}\;(22)$ & $25.12\;(26)$ \\
 & DDPM & $64.81_{\pm 0.12}\;(25)$ & $20.37_{\pm 0.08}\;(29)$ & $30.98_{\pm 0.11}\;(22)$ & $19.90_{\pm 0.10}\;(29)$ & $887.9_{\pm 2.0}\;(29)$ & $900.9_{\pm 2.5}\;(30)$ & $3.5_{\pm 0.5}\;(30)$ & $3.4_{\pm 0.3}\;(30)$ & $28.00\;(30)$ \\
 & LODA & $61.85_{\pm 0.51}\;(27)$ & $20.30_{\pm 0.20}\;(28)$ & $27.40_{\pm 0.37}\;(30)$ & $20.00_{\pm 0.23}\;(30)$ & $894.8_{\pm 5.1}\;(28)$ & $902.9_{\pm 6.0}\;(29)$ & $5.1_{\pm 1.0}\;(26)$ & $5.1_{\pm 0.9}\;(25)$ & $27.88\;(29)$ \\
 & DeepSVDD & $60.80_{\pm 0.35}\;(29)$ & $20.78_{\pm 0.23}\;(30)$ & $32.10_{\pm 0.19}\;(18)$ & $19.22_{\pm 0.28}\;(27)$ & $875.3_{\pm 5.9}\;(30)$ & $915.8_{\pm 7.0}\;(28)$ & $8.4_{\pm 0.5}\;(18)$ & $10.0_{\pm 0.7}\;(14)$ & $24.25\;(24)$ \\
 & DAGMM & $60.50_{\pm 1.19}\;(30)$ & $22.44_{\pm 0.48}\;(31)$ & $25.24_{\pm 1.28}\;(31)$ & $22.64_{\pm 0.41}\;(31)$ & $836.1_{\pm 13.7}\;(31)$ & $831.0_{\pm 11.5}\;(31)$ & $3.1_{\pm 0.4}\;(31)$ & $3.1_{\pm 0.7}\;(31)$ & $30.88\;(31)$ \\
\midrule
\multirow{5}{*}{\rotatebox[origin=c]{90}{\textbf{TFM-based}}} & TACTIC-Clean & $75.22_{\pm 0.01}\;(5)$ & $12.03_{\pm 0.05}\;(4)$ & $41.67_{\pm 0.02}\;(7)$ & \textcolor{green!60!black}{$11.47_{\pm 0.04}\;(3)$} & $1090.4_{\pm 1.5}\;(4)$ & \textcolor{green!60!black}{$1104.8_{\pm 1.3}\;(3)$} & $18.7_{\pm 0.3}\;(7)$ & $19.8_{\pm 0.3}\;(6)$ & \textcolor{green!60!black}{$4.88\;(3)$} \\
 & TACTIC-Cont & $73.40_{\pm 0.01}\;(10)$ & $13.59_{\pm 0.02}\;(11)$ & $36.33_{\pm 0.02}\;(14)$ & $14.43_{\pm 0.04}\;(14)$ & $1051.8_{\pm 0.6}\;(11)$ & $1032.1_{\pm 1.0}\;(14)$ & $10.4_{\pm 0.2}\;(15)$ & $8.3_{\pm 0.4}\;(17)$ & $13.25\;(13)$ \\
 & OUTFORMER & $74.16_{\pm 0.08}\;(7)$ & $13.13_{\pm 0.05}\;(10)$ & \textcolor{blue}{\underline{$43.00_{\pm 0.13}\;(2)$}} & $11.99_{\pm 0.03}\;(5)$ & $1063.6_{\pm 1.3}\;(10)$ & $1091.9_{\pm 0.6}\;(4)$ & $20.8_{\pm 0.6}\;(4)$ & \textcolor{green!60!black}{$23.0_{\pm 0.4}\;(3)$} & $5.62\;(5)$ \\
 & FoMo-0D & $72.65_{\pm 0.04}\;(12)$ & $13.85_{\pm 0.09}\;(13)$ & \textcolor{green!60!black}{$42.70_{\pm 0.13}\;(3)$} & $12.24_{\pm 0.11}\;(6)$ & $1045.5_{\pm 2.1}\;(13)$ & $1085.5_{\pm 2.5}\;(6)$ & $20.6_{\pm 0.4}\;(5)$ & \textcolor{blue}{\underline{$24.9_{\pm 0.7}\;(2)$}} & $7.50\;(8)$ \\
 & TabPFN-Extension & $73.61_{\pm 0.11}\;(8)$ & $12.12_{\pm 0.07}\;(5)$ & $37.08_{\pm 0.07}\;(13)$ & $13.25_{\pm 0.07}\;(11)$ & $1090.3_{\pm 1.6}\;(5)$ & $1063.0_{\pm 1.4}\;(11)$ & \textcolor{green!60!black}{$25.5_{\pm 0.1}\;(3)$} & $19.8_{\pm 0.3}\;(6)$ & $7.75\;(10)$ \\
\midrule
 & \textbf{\method} & $75.23_{\pm 0.06}\;(4)$ & $13.03_{\pm 0.05}\;(9)$ & $42.27_{\pm 0.14}\;(5)$ & $12.84_{\pm 0.08}\;(8)$ & $1067.8_{\pm 1.1}\;(8)$ & $1072.5_{\pm 2.0}\;(7)$ & \textcolor{blue}{\underline{$26.2_{\pm 0.4}\;(2)$}} & \textcolor{red}{$\bm{27.7_{\pm 0.5}\;(1)}$} & $5.50\;(4)$ \\
\bottomrule
\end{tabular}
}
%\vspace{-5pt}
\end{table*}

\begin{table*}[t]
\centering
\caption{Backbone generalization of \method with TabPFN v3. Total Rank is
computed as in Table~\ref{tab:oddbench-31-baselines}. \method remains
consistently strong with the alternative pretrained general-purpose TFM backbone,
achieving the second-best overall Total Rank among 31 methods.}

\label{tab:oddbench-31-baselines-tabpfnv3}
\setlength{\tabcolsep}{3.2pt}
\scriptsize
\resizebox{\textwidth}{!}{%
\begin{tabular}{llrrrrrrrrr}
\toprule
\multicolumn{1}{c}{\textbf{}} & \textbf{Methods}
& \begin{tabular}[c]{@{}c@{}}\textbf{Avg.}\\\textbf{AUCROC $\uparrow$}\end{tabular}
& \begin{tabular}[c]{@{}c@{}}\textbf{Avg. Rank}\\\textbf{(AUCROC) $\downarrow$}\end{tabular}
& \begin{tabular}[c]{@{}c@{}}\textbf{Avg.}\\\textbf{AUCPR $\uparrow$}\end{tabular}
& \begin{tabular}[c]{@{}c@{}}\textbf{Avg. Rank}\\\textbf{(AUCPR) $\downarrow$}\end{tabular}
& \begin{tabular}[c]{@{}c@{}}\textbf{ELO}\\\textbf{(AUCROC) $\uparrow$}\end{tabular}
& \begin{tabular}[c]{@{}c@{}}\textbf{ELO}\\\textbf{(AUCPR) $\uparrow$}\end{tabular}
& \begin{tabular}[c]{@{}c@{}}\textbf{Top-3 Ratio}\\\textbf{(AUCROC) $\uparrow$}\end{tabular}
& \begin{tabular}[c]{@{}c@{}}\textbf{Top-3 Ratio}\\\textbf{(AUCPR) $\uparrow$}\end{tabular}
& \begin{tabular}[c]{@{}c@{}}\textbf{Total}\\\textbf{Rank $\downarrow$}\end{tabular} \\
\midrule
\multirow{25}{*}{\rotatebox[origin=c]{90}{\textbf{Conventional}}} & DTE-NP & \textcolor{red}{$\bm{76.96_{\pm 0.01}\;(1)}$} & \textcolor{red}{$\bm{8.93_{\pm 0.03}\;(1)}$} & \textcolor{blue}{\underline{$43.63_{\pm 0.01}\;(2)$}} & \textcolor{red}{$\bm{9.19_{\pm 0.03}\;(1)}$} & \textcolor{red}{$\bm{1171.5_{\pm 1.0}\;(1)}$} & \textcolor{red}{$\bm{1164.2_{\pm 0.9}\;(1)}$} & \textcolor{red}{$\bm{27.4_{\pm 0.6}\;(1)}$} & $21.5_{\pm 0.5}\;(4)$ & \textcolor{red}{$\bm{1.50\;(1)}$} \\
 & KNN & \textcolor{blue}{\underline{$76.25_{\pm 0.01}\;(2)$}} & \textcolor{blue}{\underline{$9.67_{\pm 0.03}\;(2)$}} & $42.41_{\pm 0.01}\;(5)$ & \textcolor{blue}{\underline{$10.21_{\pm 0.03}\;(2)$}} & \textcolor{blue}{\underline{$1148.4_{\pm 1.2}\;(2)$}} & \textcolor{blue}{\underline{$1134.5_{\pm 0.8}\;(2)$}} & $18.7_{\pm 0.6}\;(8)$ & $14.8_{\pm 0.7}\;(10)$ & \textcolor{green!60!black}{$4.12\;(3)$} \\
 & FeatureBagging & $75.59_{\pm 0.16}\;(4)$ & \textcolor{green!60!black}{$11.99_{\pm 0.20}\;(3)$} & $39.30_{\pm 0.56}\;(10)$ & $13.01_{\pm 0.21}\;(10)$ & $1090.7_{\pm 4.7}\;(4)$ & $1065.7_{\pm 5.1}\;(10)$ & $20.1_{\pm 1.2}\;(6)$ & $19.2_{\pm 1.3}\;(8)$ & $6.88\;(6)$ \\
 & LOF & $74.17_{\pm 0.02}\;(6)$ & $12.68_{\pm 0.04}\;(7)$ & $39.97_{\pm 0.02}\;(9)$ & $12.79_{\pm 0.05}\;(8)$ & $1072.6_{\pm 1.1}\;(7)$ & $1070.5_{\pm 1.1}\;(8)$ & $14.4_{\pm 0.7}\;(10)$ & $15.9_{\pm 0.9}\;(9)$ & $8.00\;(10)$ \\
 & DTE-C & $73.46_{\pm 0.05}\;(9)$ & $12.88_{\pm 0.12}\;(8)$ & $38.16_{\pm 0.09}\;(11)$ & $13.59_{\pm 0.14}\;(12)$ & $1067.3_{\pm 3.2}\;(8)$ & $1050.7_{\pm 3.6}\;(12)$ & $11.3_{\pm 0.7}\;(13)$ & $9.8_{\pm 0.4}\;(15)$ & $11.00\;(11)$ \\
 & DTE-IG & $73.34_{\pm 0.43}\;(11)$ & $12.93_{\pm 0.30}\;(9)$ & $41.92_{\pm 0.34}\;(6)$ & $11.97_{\pm 0.23}\;(5)$ & $1065.2_{\pm 7.0}\;(9)$ & $1089.4_{\pm 5.5}\;(6)$ & $16.5_{\pm 1.1}\;(9)$ & $20.2_{\pm 0.9}\;(6)$ & $7.62\;(7)$ \\
 & OCSVM & $71.88_{\pm 0.01}\;(13)$ & $13.76_{\pm 0.02}\;(12)$ & $37.18_{\pm 0.01}\;(12)$ & $13.90_{\pm 0.04}\;(13)$ & $1045.9_{\pm 0.7}\;(12)$ & $1043.3_{\pm 0.8}\;(13)$ & $8.5_{\pm 0.4}\;(18)$ & $7.3_{\pm 0.3}\;(18)$ & $13.88\;(14)$ \\
 & ICL & $71.82_{\pm 0.10}\;(14)$ & $13.88_{\pm 0.07}\;(14)$ & $40.07_{\pm 0.15}\;(8)$ & $12.95_{\pm 0.15}\;(9)$ & $1044.2_{\pm 1.4}\;(14)$ & $1067.3_{\pm 3.6}\;(9)$ & $11.8_{\pm 0.4}\;(11)$ & $13.3_{\pm 0.6}\;(11)$ & $11.25\;(12)$ \\
 & IForest & $71.26_{\pm 0.05}\;(15)$ & $14.95_{\pm 0.06}\;(15)$ & $31.26_{\pm 0.10}\;(20)$ & $16.40_{\pm 0.10}\;(16)$ & $1024.1_{\pm 1.4}\;(15)$ & $989.8_{\pm 2.4}\;(16)$ & $4.9_{\pm 0.5}\;(27)$ & $4.4_{\pm 0.1}\;(27)$ & $18.88\;(18)$ \\
 & MCD & $69.10_{\pm 0.19}\;(17)$ & $16.43_{\pm 0.15}\;(17)$ & $31.10_{\pm 0.22}\;(21)$ & $17.77_{\pm 0.12}\;(19)$ & $985.7_{\pm 3.3}\;(17)$ & $953.9_{\pm 2.7}\;(19)$ & $11.4_{\pm 0.8}\;(12)$ & $10.5_{\pm 0.5}\;(12)$ & $16.75\;(17)$ \\
 & VAE & $70.20_{\pm 0.04}\;(16)$ & $15.64_{\pm 0.08}\;(16)$ & $35.34_{\pm 0.07}\;(15)$ & $15.86_{\pm 0.10}\;(15)$ & $1001.9_{\pm 1.8}\;(16)$ & $997.5_{\pm 2.3}\;(15)$ & $7.8_{\pm 0.4}\;(19)$ & $6.7_{\pm 0.4}\;(21)$ & $16.62\;(16)$ \\
 & PlanarFlow & $68.11_{\pm 0.40}\;(18)$ & $17.01_{\pm 0.13}\;(18)$ & $32.95_{\pm 0.17}\;(17)$ & $17.43_{\pm 0.05}\;(17)$ & $971.3_{\pm 3.5}\;(18)$ & $962.0_{\pm 1.3}\;(17)$ & $11.3_{\pm 0.7}\;(14)$ & $10.5_{\pm 0.8}\;(12)$ & $16.38\;(15)$ \\
 & SLAD & $66.87_{\pm 0.05}\;(21)$ & $17.75_{\pm 0.03}\;(19)$ & $33.49_{\pm 0.09}\;(16)$ & $17.48_{\pm 0.03}\;(18)$ & $951.9_{\pm 0.7}\;(19)$ & $959.1_{\pm 0.7}\;(18)$ & $6.8_{\pm 0.2}\;(22)$ & $7.2_{\pm 0.2}\;(19)$ & $19.00\;(19)$ \\
 & COPOD & $67.20_{\pm 0.01}\;(20)$ & $18.26_{\pm 0.05}\;(20)$ & $27.99_{\pm 0.01}\;(29)$ & $19.40_{\pm 0.06}\;(28)$ & $944.7_{\pm 1.2}\;(20)$ & $917.0_{\pm 1.6}\;(27)$ & $4.1_{\pm 0.3}\;(28)$ & $4.3_{\pm 0.2}\;(28)$ & $25.00\;(25)$ \\
 & PCA & $64.50_{\pm 0.03}\;(26)$ & $18.40_{\pm 0.02}\;(21)$ & $29.14_{\pm 0.02}\;(26)$ & $18.57_{\pm 0.04}\;(22)$ & $939.1_{\pm 0.6}\;(22)$ & $935.3_{\pm 0.8}\;(22)$ & $6.3_{\pm 0.1}\;(24)$ & $4.9_{\pm 0.3}\;(26)$ & $23.62\;(23)$ \\
 & ECOD & $68.10_{\pm 0.01}\;(19)$ & $18.40_{\pm 0.06}\;(22)$ & $28.25_{\pm 0.02}\;(28)$ & $18.89_{\pm 0.08}\;(23)$ & $941.7_{\pm 1.4}\;(21)$ & $929.9_{\pm 2.0}\;(23)$ & $6.9_{\pm 0.4}\;(21)$ & $6.7_{\pm 0.5}\;(22)$ & $22.38\;(22)$ \\
 & HBOS & $66.58_{\pm 0.02}\;(22)$ & $18.53_{\pm 0.06}\;(23)$ & $29.90_{\pm 0.03}\;(24)$ & $18.24_{\pm 0.05}\;(20)$ & $937.2_{\pm 1.5}\;(23)$ & $944.8_{\pm 1.2}\;(20)$ & $6.8_{\pm 0.3}\;(22)$ & $7.1_{\pm 0.4}\;(20)$ & $21.75\;(21)$ \\
 & DIF & $65.16_{\pm 0.02}\;(24)$ & $18.77_{\pm 0.04}\;(25)$ & $28.53_{\pm 0.06}\;(27)$ & $19.21_{\pm 0.04}\;(26)$ & $932.2_{\pm 0.9}\;(25)$ & $921.7_{\pm 0.9}\;(25)$ & $3.5_{\pm 0.1}\;(29)$ & $3.9_{\pm 0.2}\;(29)$ & $26.25\;(28)$ \\
 & GANomaly & $66.02_{\pm 0.58}\;(23)$ & $18.56_{\pm 0.25}\;(24)$ & $31.59_{\pm 0.38}\;(19)$ & $18.27_{\pm 0.16}\;(21)$ & $933.8_{\pm 5.8}\;(24)$ & $941.3_{\pm 3.7}\;(21)$ & $9.0_{\pm 1.0}\;(16)$ & $9.6_{\pm 0.7}\;(16)$ & $20.50\;(20)$ \\
 & DROCC & $61.17_{\pm 0.63}\;(28)$ & $19.43_{\pm 0.19}\;(26)$ & $29.19_{\pm 0.33}\;(25)$ & $18.94_{\pm 0.23}\;(24)$ & $912.5_{\pm 5.2}\;(26)$ & $925.0_{\pm 6.2}\;(24)$ & $5.8_{\pm 0.5}\;(25)$ & $6.1_{\pm 0.4}\;(24)$ & $25.25\;(26)$ \\
 & GOAD & $60.34_{\pm 0.46}\;(31)$ & $19.93_{\pm 0.24}\;(27)$ & $30.94_{\pm 0.31}\;(23)$ & $19.07_{\pm 0.27}\;(25)$ & $898.4_{\pm 6.1}\;(27)$ & $920.8_{\pm 6.7}\;(26)$ & $7.7_{\pm 0.6}\;(20)$ & $6.5_{\pm 0.4}\;(23)$ & $25.25\;(26)$ \\
 & DDPM & $64.81_{\pm 0.12}\;(25)$ & $20.39_{\pm 0.09}\;(29)$ & $30.98_{\pm 0.11}\;(22)$ & $19.93_{\pm 0.09}\;(29)$ & $887.2_{\pm 2.1}\;(29)$ & $900.2_{\pm 2.2}\;(30)$ & $3.5_{\pm 0.5}\;(30)$ & $3.3_{\pm 0.3}\;(30)$ & $28.00\;(30)$ \\
 & LODA & $61.85_{\pm 0.51}\;(27)$ & $20.33_{\pm 0.19}\;(28)$ & $27.40_{\pm 0.37}\;(30)$ & $20.04_{\pm 0.24}\;(30)$ & $894.0_{\pm 5.0}\;(28)$ & $901.6_{\pm 6.1}\;(29)$ & $5.1_{\pm 0.9}\;(26)$ & $5.0_{\pm 1.0}\;(25)$ & $27.88\;(29)$ \\
 & DeepSVDD & $60.80_{\pm 0.35}\;(29)$ & $20.81_{\pm 0.22}\;(30)$ & $32.10_{\pm 0.19}\;(18)$ & $19.25_{\pm 0.28}\;(27)$ & $874.4_{\pm 5.8}\;(30)$ & $915.0_{\pm 6.9}\;(28)$ & $8.5_{\pm 0.5}\;(17)$ & $10.2_{\pm 0.4}\;(14)$ & $24.12\;(24)$ \\
 & DAGMM & $60.50_{\pm 1.19}\;(30)$ & $22.47_{\pm 0.47}\;(31)$ & $25.24_{\pm 1.28}\;(31)$ & $22.68_{\pm 0.40}\;(31)$ & $835.1_{\pm 13.4}\;(31)$ & $829.5_{\pm 11.2}\;(31)$ & $3.2_{\pm 0.3}\;(31)$ & $3.1_{\pm 0.6}\;(31)$ & $30.88\;(31)$ \\
\midrule
\multirow{5}{*}{\rotatebox[origin=c]{90}{\textbf{TFM-based}}} & TACTIC-Clean & $75.22_{\pm 0.01}\;(5)$ & $12.08_{\pm 0.05}\;(5)$ & $41.68_{\pm 0.02}\;(7)$ & \textcolor{green!60!black}{$11.50_{\pm 0.04}\;(3)$} & $1089.4_{\pm 1.4}\;(6)$ & \textcolor{green!60!black}{$1104.2_{\pm 1.3}\;(3)$} & $18.8_{\pm 0.3}\;(7)$ & $19.8_{\pm 0.3}\;(7)$ & $5.38\;(4)$ \\
 & TACTIC-Cont & $73.40_{\pm 0.01}\;(10)$ & $13.63_{\pm 0.02}\;(11)$ & $36.34_{\pm 0.02}\;(14)$ & $14.46_{\pm 0.03}\;(14)$ & $1050.8_{\pm 0.6}\;(11)$ & $1031.3_{\pm 0.9}\;(14)$ & $10.1_{\pm 0.3}\;(15)$ & $8.1_{\pm 0.4}\;(17)$ & $13.25\;(13)$ \\
 & OUTFORMER & $74.16_{\pm 0.08}\;(7)$ & $13.16_{\pm 0.06}\;(10)$ & \textcolor{green!60!black}{$43.00_{\pm 0.13}\;(3)$} & $12.01_{\pm 0.03}\;(6)$ & $1062.9_{\pm 1.4}\;(10)$ & $1091.6_{\pm 0.6}\;(5)$ & $21.2_{\pm 0.6}\;(4)$ & \textcolor{green!60!black}{$23.5_{\pm 0.6}\;(3)$} & $6.00\;(5)$ \\
 & FoMo-0D & $72.65_{\pm 0.04}\;(12)$ & $13.87_{\pm 0.08}\;(13)$ & $42.70_{\pm 0.13}\;(4)$ & $12.26_{\pm 0.10}\;(7)$ & $1045.0_{\pm 1.9}\;(13)$ & $1085.1_{\pm 2.4}\;(7)$ & $20.8_{\pm 0.8}\;(5)$ & \textcolor{blue}{\underline{$25.3_{\pm 0.6}\;(2)$}} & $7.88\;(9)$ \\
 & TabPFN-Extension & $73.61_{\pm 0.11}\;(8)$ & $12.13_{\pm 0.07}\;(6)$ & $37.08_{\pm 0.07}\;(13)$ & $13.26_{\pm 0.06}\;(11)$ & $1090.1_{\pm 1.7}\;(5)$ & $1062.9_{\pm 1.2}\;(11)$ & \textcolor{green!60!black}{$25.9_{\pm 0.3}\;(3)$} & $20.3_{\pm 0.3}\;(5)$ & $7.75\;(8)$ \\
\midrule
 & \textbf{\method-TabPFNv3} & \textcolor{green!60!black}{$75.91_{\pm 0.06}\;(3)$} & $12.05_{\pm 0.06}\;(4)$ & \textcolor{red}{$\bm{44.64_{\pm 0.17}\;(1)}$} & $11.91_{\pm 0.12}\;(4)$ & \textcolor{green!60!black}{$1091.0_{\pm 1.2}\;(3)$} & $1095.1_{\pm 2.7}\;(4)$ & \textcolor{blue}{\underline{$26.3_{\pm 0.5}\;(2)$}} & \textcolor{red}{$\bm{27.6_{\pm 0.8}\;(1)}$} & \textcolor{blue}{\underline{$2.75\;(2)$}} \\
\bottomrule
\end{tabular}%%
}\vspace{-10pt}
\end{table*}
To assess whether \method generalizes beyond the TabICLv2 backbone
used in our main experiments~\citep{qu2026tabiclv2betterfasterscalable}, we replace it with two alternative pretrained general-purpose TFMs:
TabPFN~v2.6\footnote{\url{https://huggingface.co/Prior-Labs/tabpfn_2_6}}
and TabPFN~v3~\citep{grinsztajn2026tabpfn3technicalreport}, both of which
also exhibit strong predictive performance on TabArena~\citep{NEURIPS2025_1697e3fb}. We evaluate each variant on ODDBench against the same 30 TAD baselines.
For each backbone, we perform the same one-time backbone-specific calibration
of the framework-level hyperparameters for task selection and score aggregation,
using the common configuration grid described in
Appendix~\ref{Appendix:Configuration Search}.

% As shown in Table~\ref{tab:oddbench-31-baselines-tabpfnv2p6} and ~\ref{tab:oddbench-31-baselines-tabpfnv3}, \method replacing TFM backbone with each new backone, overall maintain strong performance compared to other TAD (4th, 2nd rank in terms of total rank, respectively).
As shown in Tables~\ref{tab:oddbench-31-baselines-tabpfnv2p6} and \ref{tab:oddbench-31-baselines-tabpfnv3}, \method maintains strong overall performance when replacing the default TFM backbone with each alternative backbone, ranking fourth and second in terms of Total Rank, respectively.
Notably, even when using the same TabPFN~v3.0 backbone as TabPFN-Extension, \method achieves better performance across all evaluation metrics, indicating that our framework harnesses the pretrained TFM more effectively for TAD.
These results demonstrate that the core TAD mechanism of \method---selecting suitable virtual supervised tasks and effectively aggregating task-wise scores---generalizes across different pretrained TFM backbones with strong
capabilities for inferring task structure from supervised context and producing
reliable predictive support.

\end{document}